\documentclass[journal]{IEEEtran}

\ifCLASSINFOpdf
\else
\fi

\usepackage[dvipsnames]{xcolor}

\usepackage{overpic}
\graphicspath{{images/}}
\usepackage{color}
\usepackage{soul}
\usepackage{bm}
\usepackage{multirow}
\usepackage{multicol}
\usepackage{xspace}
\usepackage{graphicx}
\usepackage{float}
\usepackage{subfig}
\usepackage{url}
\usepackage{cite}

\usepackage{xspace}
\def\flow{RMP{flow}\xspace}
\def\algebra{RMP-algebra\xspace}
\def\tree{RMP-tree\xspace}

\def\pushforward{{\small\texttt{pushforward}}\xspace}
\def\pullback{{\small\texttt{pullback}}\xspace}
\def\resolve{{\small\texttt{resolve}}\xspace}

\usepackage{amsmath}
\usepackage{amsfonts}
\usepackage{amssymb}
\usepackage{amsthm}
\usepackage{bm}
\usepackage{bbm}
\usepackage{mathtools}
\usepackage{enumitem}
\usepackage{thmtools,thm-restate}
\usepackage{algorithm}
\usepackage{algorithmic}
\usepackage{color}
\usepackage{graphicx}
\usepackage{comment}

\theoremstyle{plain}
\newtheorem{corollary}{Corollary}

\theoremstyle{definition}

\theoremstyle{remark}

\def\AA{\mathcal{A}}\def\CC{\mathcal{C}}

\def\MM{\mathcal{M}}\def\NN{\mathcal{N}}

\def\SS{\mathcal{S}}\def\TT{\mathcal{T}}
\def\XX{\mathcal{X}}

\def\Bb{\mathbf{B}}\def\Cb{\mathbf{C}}

\def\Gb{\mathbf{G}}\def\Hb{\mathbf{H}}\def\Ib{\mathbf{I}}
\def\Jb{\mathbf{J}}\def\Lb{\mathbf{L}}
\def\Mb{\mathbf{M}}
\def\Rb{\mathbf{R}}

\def\ab{\mathbf{a}}
\def\fb{\mathbf{f}}
\def\gb{\mathbf{g}}

\def\mbb{\mathbf{m}}

\def\qb{\mathbf{q}}
\def\ub{\mathbf{u}}
\def\vb{\mathbf{v}}\def\xb{\mathbf{x}}

\def\Rbb{\mathbb{R}}

\def\R{\Rbb}

\def\diag{\mathrm{diag}}

\def\t{\top}
\def\*{\star}

\newcommand{\norm}[1]{ \| #1 \|  }

\newcommand{\lr}[2]{ \left\langle #1, #2 \right\rangle}
\DeclareMathOperator*{\argmin}{arg\,min}

\usepackage[dvipsnames]{xcolor}

\newcommand{\paral}{{\!/\mkern-5mu/\!}}

\newcommand{\ma}{\mathbf{a}}

\newcommand{\q}{\mathbf{q}}
\newcommand{\qd}{{\dot{\q}}}
\newcommand{\qdd}{{\ddot{\q}}}
\newcommand{\uu}{\mathbf{u}}

\newcommand{\vv}{\mathbf{v}}

\newcommand{\x}{\mathbf{x}}
\newcommand{\xd}{{\dot{\x}}}
\newcommand{\xdd}{{\ddot{\x}}}
\newcommand{\y}{\mathbf{y}}
\newcommand{\yd}{{\dot{\y}}}
\newcommand{\ydd}{{\ddot{\y}}}
\newcommand{\z}{\mathbf{z}}

\newcommand{\f}{\mathbf{f}}

\newcommand{\zero}{\mathbf{0}}

\newcommand{\J}{\mathbf{J}}
\newcommand{\Jd}{{\dot{\J}}}
\newcommand{\Jp}{\J_{\phi}}
\newcommand{\Jdp}{\Jd_{\phi}}

\newcommand{\B}{\mathbf{B}}
\newcommand{\C}{\mathbf{C}}
\newcommand{\D}{\mathbf{D}}

\newcommand{\G}{\mathbf{G}}

\newcommand{\I}{\mathbf{I}}

\newcommand{\M}{\mathbf{M}}

\newcommand{\mT}{\mathbf{T}}
\newcommand{\V}{\mathbf{V}}

\newcommand{\wt}[1]{{\widetilde{#1}}}

\newcommand{\T}{\top}

\usepackage{url}

\usepackage[capitalise]{cleveref}

\newcommand{\green}[1]{{\leavevmode\color{black}#1}}
\newcommand{\red}[1]{{\leavevmode\color{black}#1}}

\newcommand{\blueold}[1]{{\leavevmode\color{black}#1}}

\newif\ifLONG
\LONGfalse

\newif\ifAPP
\APPtrue

\begin{document}
\title{RMP\textit{flow}: A Geometric Framework for \\ Generation of Multi-Task Motion Policies}

\newcommand{\inst}[1]{$^{#1}$}
\author{
	Ching-An Cheng\inst{1,2},
	Mustafa Mukadam\inst{1,2},
	Jan Issac\inst{1},
	Stan Birchfield\inst{1},\\
	Dieter Fox\inst{1,3},
	Byron Boots\inst{1,2},
	Nathan Ratliff\inst{1}
\thanks{\inst{1}Seattle Robotics Lab, NVIDIA, Seattle, WA, USA,
\inst{2}Robot Learning Lab, Georgia Institute of Technology, Atlanta, GA, USA, \inst{3}Robotics and State Estimation Lab, University of Washington, Seattle, WA, USA}
\thanks{Manuscript received June 24, 2019; revised May 31, 2020.}
}


\markboth{RMPflow - journal version - ArXiv, 2020}%
{Cheng \MakeLowercase{\textit{et al.}}: RMPflow}

\IEEEspecialpapernotice{(Invited Paper)}

\maketitle

\begin{abstract}
Generating robot motion for multiple tasks in dynamic environments is
challenging, requiring an algorithm to respond reactively while
accounting for complex nonlinear relationships between tasks.
In this paper, we develop a novel policy synthesis algorithm, 
\flow, based
on geometrically consistent transformations of Riemannian Motion Policies (RMPs). RMPs are a class of reactive motion policies that parameterize non-Euclidean behaviors as dynamical systems in intrinsically nonlinear task spaces.  Given a set of RMPs designed for individual tasks, \flow can combine these
policies to generate an expressive global policy, while simultaneously exploiting
sparse structure for computational efficiency.  We
study the geometric properties of \flow  and provide sufficient
conditions for stability.
Finally, we experimentally demonstrate that accounting for \red{the natural Riemannian}
geometry of task policies can simplify
classically difficult problems, such as
planning
through clutter on
high-DOF manipulation systems.
\end{abstract}

\def\abstractname{Note to Practitioners}
\begin{abstract}

    Requirements on safety and responsiveness for collaborative robots have
    driven a need for new ideas in control design that bridge between standard
    objectives in low-level control (such as trajectory tracking)   and
    high-level behavioral objectives (such as collision avoidance)
    often relegated to planning systems.
 	Modern results from geometric control, which promise
	stable controllers that can smoothly and safely transition between many
    behavioral tasks, therefore become highly relevant. However, for years this
    field has remained inaccessible due to its mathematical
    complexity.
    This paper aims to 1) make those ideas accessible to robotics and control
    experts by recasting them in a concrete algorithmic framework amenable to
    controller design, and 2) to additionally generalize them to better satisfy
    the specific needs of robotic behavior generation.
    Our experiments demonstrate that the resulting controllers can engender
    natural behavior that adapts instantaneously to changing surroundings with
    zero planning while performing manipulation tasks.
    The framework is gaining traction within the robotics community, finding
    increasing application in areas such as autonomous navigation, tactile
    servoing, and multi-agent systems.
    Future research will address learning these controllers from data to
    simplify that process of design and tuning, which at present can require experience.

\end{abstract}

\begin{IEEEkeywords}
\red{Operational Space Control, Acceleration Control,} Motion and Path Planning, Collision Avoidance, Dynamics
\end{IEEEkeywords}

\IEEEpeerreviewmaketitle

\vspace{-2mm}
\section{Introduction}

\IEEEPARstart{I}{n} this work, we develop a new  motion generation and control
framework that enables globally stable controller design for 
non-Euclidean spaces (namely, spaces defined by non-constant Riemannian metrics with non-trivial curvature).
\red{Non-Euclidean geometries arise commonly in the natural world, in particular
in the problem of obstacle avoidance.  When obstacles are present, straight
lines are no longer a reasonable definition of geodesics (shortest length paths).  Rather, geodesics must naturally flow around these obstacles that, in effect, become holes in the space and block trajectories from passing.}  This behavior implies a form of non-Euclidean geometry because the space is naturally curved by the presence of obstacles.

The planning literature has made substantial progress in modeling non-Euclidean task-space behaviors, but at the expense of efficiency and reactivity.
Starting with early differential geometric models of obstacle avoidance~\cite{rimon-ams-1991} and building toward modern planning algorithms and optimization techniques~\cite{RIEMORatliff2015ICRA,VijayakumarTopologyMotionPlanning2013,Watterson-TrajOptManifolds-RSS-18,ToussaintTrajOptICML2009,LavallePlanningAlgorithms06,KaramanRRTStar2011,GammellBitStar2014,mukadam2017continuous},
these algorithms can calculate highly nonlinear trajectories. However, they are often computationally intensive, sensitive to noise, and unresponsive to perturbation. In addition, the internal nonlinearities of robots due to kinematic constraints are sometimes simplified in the optimization. \red{While fast approximation and replanning heuristics have been proposed, the above characteristics in their nature make them unsuitable for motion generation in dynamic situations.}

At the same time, a separate thread of literature, emphasizing fast reactive responses over computationally expensive planning, developed efficient closed-loop
control techniques such as \red{operational space control}~\cite{khatib1987unified}. But while these techniques account for internal geometries from the robot's kinematic structure, they assume simple Euclidean geometry in task spaces~\cite{Peters_AR_2008,UdwadiaGaussPrincipleControl2003}, \red{thus} failing to provide a complete treatment of external geometries.
\red{For example, the unified formulation of operational space control~\cite{Peters_AR_2008,UdwadiaGaussPrincipleControl2003} is implicitly built on a classical mechanics concept called Gauss's principle of least constraint~\cite{udwadia1996analytical} which assumes each task space is Euclidean. }
Consequently obstacle avoidance, e.g., has to rely on \emph{extrinsic} potential functions, leading to undesirable \red{deceleration} behavior when the robot is close to the obstacle. \red{If somehow} the non-Euclidean geometry can be \emph{intrinsically} considered, then fast obstacle avoidance motion would naturally arise as traveling along the induced geodesics.
The need for a holistic solution to motion generation and control has motivated a number of recent system architectures that tightly integrate planning and control~\cite{2017_rss_system,mukadam2017approximately}.

\red{We improve upon these works by developing}
a new approach to synthesizing control policies that can intrinsically
accommodate and leverage the modeling capacity of
non-Euclidean robotics tasks.
Taking inspiration from Geometric Control
Theory~\cite{bullo2004geometric},\footnote{See
	\cref{app:GDSs} for a discussion of why geometric
	mechanics and geometric control theory constitute a good starting point.} we
design a novel recursive algorithm, \flow, which represents a class of nonlinear policies in terms of a recently proposed
control-policy descriptor \red{known as the} Riemannian
Motion Policy (RMP)~\cite{ratliff2018riemannian}.
\flow
enables geometrically consistent fusion of many component policies defined in non-Euclidean task spaces that are related through a tree structure.
\red{In essence, \flow computes a robot's desired acceleration by solving a high-dimensional weighted least-squared problem in which the weight matrices are nonlinear functions of the robot's position and velocity (i.e., the system's state).
While solving a high-dimensional optimization problem seems computationally difficult at first glance, \flow avoids this pitfall by computing the policy through performing forward and backward message passing along the tree structure that relates different task spaces.
As a result, the computation paths shared across different tasks can be leveraged to achieve efficiency.
Algorithmically, we can view \flow as mimicking} the Recursive Newton-Euler algorithm~\cite{walker1982efficient} in structure, but generalizing it beyond rigid-body systems to a broader class of highly nonlinear transformations and spaces.

In contrast to existing frameworks, 
our framework\red{, through the use of nonlinear weight matrix functions,}
naturally models non-Euclidean task spaces with Riemannian metrics that are not
only configuration dependent, but also \emph{velocity} dependent.
This allows \flow to consider, e.g., the \emph{direction} a robot travels to define the importance weights in combing policies.
For example, an obstacle, despite
being close to the robot, can usually be ignored if robot is heading away from
it.
This new class of policies leads to an extension of Geometric Control Theory,
building a new class of non-physical mechanical systems
we call geometric dynamical systems (GDS).

\red{While \flow offers extra flexibility in control design, one might naturally ask if it is even stable, as the use of weight function introduces additional feedback signals that could destroy the original stability of the component policies.
The answer to this question is affirmative.
We prove that \flow is Lyapunov-stable. Moreover, we show that the construction of \flow is coordinate-free.} In particular, when using \flow,
robots can be viewed each as different parameterizations of the same task space, defining a precise notion of behavioral
consistency between robots.
Additionally, under this framework, the implicit curvature arising from
non-constant Riemannian metrics (which may be roughly viewed as configuration-velocity dependent inertia matrices in operational space control) produces nontrivial and intuitive policy contributions
that 
are critical to guaranteeing stability and generalization across embodiments.

\red{We demonstrate the properties of \flow in simulations and experiments.}
Our experimental results illustrate how these {curvature terms} can be impactful in practice, generating nonlinear geodesics that result in curving or orbiting around obstacles.
Furthermore, we demonstrate the utility of our framework with a fully reactive real-world system
implementation on multiple dual-arm manipulation problems.

\red{An earlier conference version of this paper was published
as~\cite{cheng2018rmpflow} with more details in a corresponding
technical report~\cite{cheng2018rmpflowarxiv} which includes
many specific examples of the RMPs used in the experiments (Appendix D).
\green{In addition to providng extra details to~\cite{cheng2018rmpflow}, this
extended manuscript offers a new tutorial in \cref{sec:tutorial} that discusses in depth
the design rationale behind \flow and how \flow relates to and generalizes
existing schemes.}
The design of \flow is highly inspired by the seminal work of RMPs~\cite{ratliff2018riemannian} that promotes the concept of including geometric information in policy fusion. This paper and its former version~\cite{cheng2018rmpflow} formalize the original intuition in \cite{ratliff2018riemannian} and further extend this idea to geometric mechanics and beyond.
Increasingly RMPs and \flow have been applied broadly into robotic systems, finding applications in
autonomous navigation \cite{meng2019NeuralAutoNavigation,meng2020TopologicalNavigation},
manipulation systems \cite{2017_rss_system}, \green{humanoid control \cite{wingo2020Extendings}}, reactive logical task sequencing
\cite{paxton2019RLDS},
tactile servoing \cite{sutanto2019TactileServoing},
and multi-agent systems \cite{li2019MultiAgentRMPsArXiv,li2019LyapunovRMPs}.
And related work has begun exploring the learning of RMPs
\cite{mukadam2019RMPFusion,rana2019LearningRmpsFromDemonstration}.
}

\section{Motion Generation and Control}

Motion generation and control can be formulated as the problem of transforming curves \green{between} the configuration space $\CC$ to the task space $\TT$.
Specifically, let the configuration space
$\CC$ be a $d$-dimensional smooth manifold. A robot's motion can be described as a curve $q:[0,\infty) \to \CC$ such that the robot's configuration at time $t$ is a point $q(t) \in \CC$.
Without loss of generality, suppose $\CC$ has a global coordinate $\q: \CC \to \R^d$, called the \emph{generalized coordinate}; for brevity, we identify the curve $q$ with its coordinate expression  $\q \circ q$
and write $\q(q(t))$ as $\q(t) \in \R^d$.
A typical example of the generalized coordinate is the joint angles of a $d$-DOF (degrees-of-freedom) robot: we denote $\q(t)$ as the joint angles at time $t$ and $\qd(t)$, $\qdd(t)$ as the joint velocities and accelerations, respectively.
To describe the tasks, we consider another manifold $\TT$, the task space, which is related to the configuration space $\CC$ through a smooth \emph{task map} $\psi: \CC \to \TT$. The task space $\TT$ can be the end-effector position/orientation~\cite{khatib1987unified,albu2002cartesian}, or more generally can be a space that describes whole-body robot motion, e.g., in simultaneous tracking and collision avoidance~\cite{sentis2006whole,lo2016virtual}.
\red{Under this setup, thus the goal of motion generation and control can be viewed as designing the curve $q$ (in a closed-loop manner) so that the transformed curve $\psi \circ q $ exhibits desired behaviors on the task space $\TT$.}

\red{To simplify the exposition, below we suppose that the robot's dynamics
have been feedback linearized and restrict our attention to designing
acceleration-based controllers. We remark that a torque-based setup can be
similarly derived by redefining the pseudo-inverse in \resolve in Section~\ref{sec:RMPAlgebra} in terms of the inner product space induced by the robot's physical inertia on $\CC$~\cite{Peters_AR_2008},
so long as the system is
fully actuated and the inverse dynamics can be modeled.}

\subsection{Notation}
For clarity, we use boldface to distinguish the coordinate-dependent representations from abstract objects; e.g. we write $ q(t) \in \CC$ and $\q(t) \in \R^d$.  In addition, we will often omit the time- and input-dependency of objects unless necessary; e.g., we may write $ q \in \CC$ and $(\q,\qd, \qdd)$.
For derivatives, we use both symbols $\nabla$ and $\partial$, with a transpose relationship: for $\x \in \R^m$ and a differential map $\y:\R^m \to \R^n$, we write $\nabla_\x \y(\x) = \partial_\x \y(\x)^\t \in \R^{m \times n}$.
\red{This choice of notation allows us to write $\nabla_\y f(\y) \in \R^n$ when $f$ is a scalar function and perform chain-rule $\partial_x f(\y(\x)) = \partial_\y  f(\y) \partial_\x \y(\x)$ in the usual way. }
For a matrix $\M \in \R^{m\times m}$, we denote $\mbb_i = (\Mb)_i$ as its $i$th column and $M_{ij} = (\Mb)_{ij}$ as its $(i,j)$ element. To compose a matrix from vector or scalar elements,
we use $(\cdot)_{\cdot}^\cdot$ for vertical (or matrix) concatenation  and $[\cdot]_{\cdot}^\cdot$ for horizontal concatenation. For example, we write $\M = [\mbb_i]_{i=1}^m = (M_{ij})_{i,j=1}^m$ and $\M^\t = (\mbb_i^\t)_{i=1}^m =  (M_{ji})_{i,j=1}^m$. We use $\R^{m\times m}_{+}$ and $\R^{m\times m}_{++}$ to denote the symmetric, positive semi-definite/definite matrices, respectively.

\subsection{Motion Policies and the Geometry of Motion}

We model motion as a second-order
differential equation
of $\qdd = \pi(\q, \qd)$, where we call $\pi$ a \textit{motion policy} and $(\q, \qd)$  the \emph{state}.
In contrast to an open-loop trajectory, which forms the basis
of many motion planners, a motion policy expresses the entire continuous collection of its integral trajectories{} and therefore is robust to perturbations.
Motion policies can model many adaptive behaviors, such as reactive obstacle avoidance
~\cite{DRCIntegratedSystemTodorov2013,2017_rss_system} or responses driven by planned Q-functions~\cite{OptimalControlTheoryTodorov06}, and their  second-order formulation enables rich behavior that cannot be realized by the velocity-based approach~\cite{liegeois1977automatic}.

The geometry of motion has been considered by many planning and control algorithms.
Geometrical modeling of task spaces is used in topological motion
planning~\cite{VijayakumarTopologyMotionPlanning2013}, and motion optimization
has leveraged Hessian to exploit the natural geometry of costs~\cite{RatliffCHOMP2009,ToussaintTrajOptICML2009,Mukadam-ICRA-16,Dong-RSS-16}.
Ratliff et al.~\cite{RIEMORatliff2015ICRA}, e.g., use the workspace geometry inside a Gauss-Newton optimizer and generate natural obstacle-avoiding reaching motion through traveling along geodesics of curved spaces.

Geometry-aware motion policies were also developed in parallel by the control community. \red{Operational space control} is the best example~\cite{khatib1987unified}.
Unlike the planning approaches, operational space control focuses on the internal geometry of the robot and considers only simple task-space geometry: it reshapes the workspace dynamics into a simple spring-mass-damper system with a {constant} inertia matrix, enforcing a form of Euclidean geometry in the task space.
\red{By contrast, pure potential-field approaches~\cite{KhatibPotentialFields1985,flacco2012icra,kaldestad2014icra} fail to realize this idea of task-space geometry and lead to inconsistent behaviors across robots.}
Variants of operational space control have been proposed to consider different metrics~\cite{Nakanishi_IJRR_2008,Peters_AR_2008,lo2016virtual}, task hierarchies~\cite{sentis2006whole,platt2011multiple}, and non-stationary inputs~\cite{IjspeertDMPs2013}.

While these algorithms have led to many advances, we argue that their isolated focuses on either the internal or the external geometry limit the performance. The planning approach fails to consider reactive dynamic behaviors; the control approach cannot\footnote{\red{Existing works, like variants of operational space control and designs centered around Geometric Control Theory~\cite{bullo2004geometric}, can consider at most position-dependent metrics.}} model the effects of velocity dependent metrics, which are critical to generating sensible obstacle avoidance motions, as discussed in the introduction.
While the benefits of velocity dependent metrics was recently explored using RMPs~\cite{ratliff2018riemannian}, a systematic understanding of its properties, like stability, is still an open question.

\section{From Operational Space Control to Geometric Control} \label{sec:tutorial}

We set the stage for our development of RMPflow and geometric dynamical systems (GDSs)
in Section~\ref{sec:RMP algebra} and \ref{sec:analysis} by first giving some
background on the key tools central to this work.
We will first give a tutorial on a controller design technique known as energy shaping and the geometric
formulation of classical mechanics, both of which are commonly less
familiar to robotics researchers.
Then we will show how geometric control \cite{bullo2004geometric}, which to a great extent developed
independently of operational space control within a distinct community,
nicely summarizes these two ideas and leads to constructive techniques of
leveraging energy shaping in the context of geometric mechanics.

\green{This section targets at readers more familiar with operational space control and
introduces relevant geometric ideas in a way that we hope is more accessible than
the traditional exposition of geometric control/mechanics which assumes a background in differential geometry.
The material presented in this section primarily rehashes existing techniques
from a perhaps unfamiliar community, restating them in a way designed to be more
natural to researchers familar with operational space control.
}

\green{ To set the stage, we start with a simple example of controller design
using these techniques where a robot must trade off competing tasks of reaching
to a target, obstacle avoidance, speed regulation, and joint limit avoidance.
We then incrementally present the underlying ideas:} First we review classical
operational space control wherein tasks are represented as hard constraints on the mechanical
system. Next we show how energy shaping and the geometric mechanics formalism enable us to
easily develop provably stable operational space controllers that
simultaneously trade off many tasks.
Finally, we end with a discussion of the limitations of these geometric control techniques that RMPflow
and GDSs will address in Section~\ref{sec:RMP algebra} and \ref{sec:analysis}.

\subsection{Motivating Example for Geometric Control}

\green{

The goal of this tutorial is to build an understanding of geometric mechanics
and how it is used in geometric control. As motivation, we present a basic
example of how geometric control can be used for intuitive controller design.
Geometric control is only the precursor to this paper's main topic of RMPflow;
its limitations are discussed in Section~\ref{sec:LimitationsOfGeometricControl}.

In our example, we consider the problem of getting a manipulator to 1) reach toward
a target while 2) avoiding an obstacle, 3) satisfying joint limits, and 4)
regulating the speed of its body.
With geometric control, this problem can be framed as building controllers in task spaces defined by
differentiable maps from the robot's configuration space $\CC$.
Specifically, the configuration space $\CC$
is $d$ dimensional (each joint represents a dimension) and our task spaces of interest are:
1) the end-effector location $\x_e = \psi_e(\q)$ (end-effector forward
   kinematics);
2) a collection of $n$ control points on the robot $\x_i = \psi_i(\q)$ for
   $i=1,\ldots,n$ (body-point forward kinematics);
3) the scalar distance $z_i = d_\mathrm{obs}(\x_i)$ between each of these
   control points and the obstacle;
4) a joint limit control space $\uu = \psi_\mathrm{jl}(\q)$. This final map is a
   map whose inverse $\q = \psi_\mathrm{jl}^{-1}(\uu)$ takes the entire
   unconstrained $\R^d$
   and squishes each dimension so that $\psi_\mathrm{jl}^{-1}(\R^d)$
   fits within the joint limits. Sigmoids, for
   instance, work well for this inverse function, and $\psi_\mathrm{jl}$ is then
   just the inverse of that.

These maps create a tree structure (known as a transform tree in general,
or what we will call the \tree below), with $\q$ at the root, $\uu$, $\x_e$ and
$\x_i$ at depth 1, and $z_i$ at depth 2 as a child of $\x_i$.
At each node of this transform tree, we place controllers with associated
nonlinear priorities:
1) An end-effector attractor on $\x_e$ (with damping) pulling toward the goal
   $\x_g$, which increases in priority as the system approaches the target.
2) Obstacle avoidance controllers on $z_i$ pushing away from obstacles with
   priority increasing with obstacle proximity.
3) Damping controllers on each body point $\x_i$ with a constant priority
   to regulate speed.
4) A joint limit controller operating in $\uu$ pulling toward a nominal
   configuration $\uu_0 = \psi_{jl}(\q_0)$, for some nominal joint values
   $\q_0$, with increasing priority away from $\uu_0$.

Note that for the joint limit space $\uu$, since all of $\R^d$ is squished down
to fit within the joint limits via $\q = \psi_{jl}^{-1}(\uu)$, even a controller
that increases its priority and desired accelerations linearly in $\uu$ will
have a dramatic nonlinear increase near a joint limit due to the nonlinearities
of $\psi_{jl}$.  Furthermore, by adding obstacle controllers to just the
single-dimensional $z_i$ distance spaces, their priorities act only along that one
dimension toward the obstacle, enabling the body-point space $\x_i$ from which
they stem to freely accommodate competing controllers acting orthogonally to those
directions.

There are many good choices for these controllers and associated priority
functions. Geometric control defines the rules for what is allowed and how to
automatically combine the controllers at the configuration space $\CC$, so that the resulting controller
trades off the individual task priorities effectively and stabily.  The rest of
this tutorial is dedicated to the construction of these rules and the principles
behind them, as well as a discussion of their limitations.

}

\subsection{Energy Shaping and Classical Operational Space Control}
\label{sec:EnergyShapingClassicalOpSpaceControl}

Energy shaping is a controller design technique: the designer first configures a virtual
mechanical system by shaping its kinetic and potential energies to exhibit a
certain behavior, and then drive the robot's dynamics to mimic that virtual system.
This scheme overall generates a control law with a well-defined Lyapunov
function, given as the virtual system's total energy, and therefore has provable stability.

For instance, the earliest form of operational space control \cite{khatib1987unified}
formulates a virtual system that places all mass at the end-effector.
Behavior is then shaped by applying potential energy functions (regulated by a
damper) to that virtual mass (e.g. by connecting the end-effector to a target
using a virtual damped spring). Controlling the system to behave like that
virtual system then generates a control law whose stability is governed by the
total energy of that virtual point-mass system.
In this context,
the choice of virtual mechanical system (the point end-effector mass) represents a
form of {\it kinetic energy shaping}, and the subsequent choice of potential
energy applied to that point end-effector mass is known as {\it potential
energy shaping}.
This particular pattern of task-centric kinetic and potential energy shaping,
is common throughout the operational space control literature.

A similar theme can be found in \cite{Peters_AR_2008}.
Here the virtual mechanical systems are designed by constraining an existing
mechanical system (e.g. the robot's original dynamics) to satisfy task
constraints.  This is achieved by designing controllers around a generalized
form of Gauss's principle of least constraint~\cite{udwadia1996analytical}, so
that virtual mechanical systems would behave in a sense as similarly as
possible to the true robotic mechanical system while realizing the required
task accelerations. In other words, the energies of the original mechanical
system are reshaped to that given by the task constraints.

In essence, the early examples above follow the principle that faithful
executions of the task enable a simplified stability analysis as long as the
task space behavior is itself well-understood and stable.
This style of analysis and controller design has been successful
in practice. Nonetheless, it faces a limitation that the controllers cannot
have more tasks than the number of DOF in the system. This restriction becomes particularly
problematic when one wishes to introduce more complex auxiliary behaviors, such
as collision avoidance where the number of tasks might scale with the number of
obstacles and the number of control points on the robot's body.

The rest of this section is dedicated to unify and then generalize the above ideas
through the lens of geometric mechanics.
The results developed therein will extend operational space control to handle more complex settings of many competing tasks
through using weighted priorities that can change as a function of the robot's configuration.
However, we will eventually see in \cref{sec:LimitationsOfGeometricControl} that even this extension is still not quite sufficient for representing many common behaviors.
The insights into sources of these limitations are the motivation of the development of \flow and geometric dynamical systems (GDSs).

\subsection{A Simple First Step toward Weighted Priorities}

This section leverages Gauss's principle of least constraint (different from
the techniques mentioned briefly in
Section~\ref{sec:EnergyShapingClassicalOpSpaceControl} \cite{Peters_AR_2008})
to illustrate the concept of energy shaping, which will be used more abstractly
below to derive a simple technique for combining multiple task-space policies.

\vspace{1mm}
\subsubsection{Gauss's Principle}

Gauss's principle of least constraint states that a nonlinearly
constrained collection of particles evolves in a way that is most similar to
its unconstrained evolution, as long as this notion of similarity is measured
using the inertia-weighted squared error \cite{UdwadiaGaussPrincipleControl2003}.
For example, let us consider $N$ particles: $\x_i\in\R^3$
with respective (positive) inertia $m_i\in\R_+$, for $i=1,\dots,N$.
Then the acceleration $\xdd_i$ of the $i^{\mathrm{th}}$
particle
under Gauss's principle can be written as
\begin{align} \label{eqn:RawGaussPrinciple}
  \textstyle \xdd =  \argmin_{\xdd' \in\mathcal{\AA}}\frac{1}{2}\|\xdd^d - \xdd'\|_\M^2
\end{align}
where $\mathcal{A}$ denotes the set of admissible constrained accelerations.
To simplify the notation, we stacked\footnote{We use the notation $\vv = (\vv_1;\vv_2;,\dots,;\vv_N)$
to denote stacking of vectors $\vv_i\in\R^3$ into a single vector $\vv\in\R^{3N}$.}
the particle accelerations into a vector $\xdd = (\xdd_1;\dots; \xdd_N)$ and construct a diagonal matrix $\M =
\mathrm{diag}(m_1\Ib,\ldots,m_N\Ib)$, where $\Ib \in \R^{3\times3}$ is the
identity matrix.

\vspace{1mm}
\subsubsection{Kinematic Control-Point Design}

Let us use the above idea to design a robot controller.
If we define many kinematic control points $\x_i\in\R^3, i = 1,\ldots,N$
distributed across the robot's body and calculate a desired acceleration
at those points $\xdd_i^d$, a sensible way to trade off these different accelerations is through the following quadratic program (QP):
\begin{align} \label{eqn:ConstantWeightControlPointPrioritiesQP}
    \textstyle \min_{\xdd_i} \sum_{i=1}^N \frac{m_i}{2} \|\xdd_i^d - \xdd_i\|^2 \quad \mathrm{s.t.}\ \ \xdd_i = \J_i\qdd + \Jd_i\qd,
\end{align}
where each $m_i>0$ is the importance weight in the QP, $\x_i = \psi_i(\q)$ is the forward kinematics map to the
$i^{\mathrm{th}}$ control point and $\J_i =
\partial_{\q_i}\psi_i$ is its Jacobian.
\green{This QP states that the system (subject to the kinematic constraints on
how each control point can accelerate) should follow the desired accelerations
if possible, while trading off different tasks using the priorities given by
$m_i$ in the event they cannot be achieved exactly.}
(As mentioned we assumed the system has been feedback linearized so we focus on acceleration only.)

Comparing this QP to that given by Gauss's principle in
\eqref{eqn:RawGaussPrinciple}, \green{we see the importance weight $m_i$ in
\eqref{eqn:ConstantWeightControlPointPrioritiesQP} plays the same role as the
inertia $m_i$ in \eqref{eqn:RawGaussPrinciple} (motivating the use of the same
symbol in both cases).
Therefore,} one can immediately see that its
solution gives the constrained dynamics of a mechanical system defined by $N$
point particles of inertia $m_i$ with unconstrained accelerations $\xdd_i^d$ and
acceleration constraints $\xdd_i = \J_i\qdd + \Jd_i\qd$.  In particular, if
$\xdd_i^d = -m_i^{-1}\nabla\phi_i - \beta_i\xd_i$ for some non-negative
potential function $\phi_i$ and constant $\beta_i$, we arrive at a mechanical system with total energy
$
	\textstyle
     \sum_{i=1}^N \left(\frac{m_i}{2}\|\xd_i\|^2 + \phi_i(\x_i)\right)
$.
Controlling the robot system according to desired accelerations $\qdd^*$
given by solving \eqref{eqn:ConstantWeightControlPointPrioritiesQP}
ensures that this total energy dissipates at a rate defined by the collective
non-negative dissipation terms $\sum_i
m_i\beta_i\|\xd_i\|^2$.  This total energy, therefore, acts as a Lyapunov function.

This kinematic control-point design technique utilizes now more explicitly the
methodology of energy shaping. In this case, we use Gauss's principle to design
a virtual mechanical system that strategically distributes point masses
throughout the robot's body at key control points (kinetic energy shaping). We
then apply damped virtual potential functions to those masses to generate
behavior (potential energy shaping). In combination, we see that the resulting
system can be viewed as a QP which tries to achieve all tasks simultaneously the best it can.
When an exact replication of all tasks is impossible, the QP uses
the inertia values as importance weights to define how the system should trade
off task errors.

\subsection{Abstract Task Spaces: Simplified Geometric Mechanics}
\label{sec:AbstractTaskSpaces}

The controller we just described demonstrates the core concept around energy
shaping, but is limited by requiring that tasks be designed specifically on
kinematic control-points distributed \emph{physically} across the robot's body.
Usually task spaces are often more abstract than that, and generally we want to
consider any task space that can be described as a nonlinear map from the
configuration space.

Using abstract task spaces is common in trajectory optimization.
For instance, \cite{ToussaintTrajOptICML2009} describes some abstract topological spaces for
behavior creation which enable behaviors such as wrapping an arm around a pole
and unwrapping it, and
abstract models of
workspace geometry are represented in
\cite{RIEMORatliff2015ICRA,MainpriceWarpingIROS2016}
by designing high-dimensional task spaces consisting of
stacked (proximity weighted) local coordinate
representations of surrounding obstacles
conveying how obstacles shape the space around them.
Likewise, similar abstract spaces are highly relevant for
describing common objectives in operational space control problems.
For instance, spaces of interest
include one-dimensional spaces encoding distances to barrier constraints such
as joint limits and obstacles, distances to targets, spaces of quaternions,
and the joint space itself; all of these are more abstract than specific
kinematic control-points. In order to generalize the ideas in the previous section to abstract task
spaces we need better tools.
Below we show the geometric mechanics and
geometric control theory \cite{bullo2004geometric}
provide the generalization that we need.

\vspace{1mm}
\subsubsection{Quick Review of Lagrangian Mechanics}
\label{sec:LagrangianMechanics}

Lagrangian mechanics is a reformulation of classical mechanics that derives
the equations of motion by applying the Euler-Lagrange equation on the Lagrangian
of the mechanical system \cite{ClassicalMechanicsTaylor05}.
Specifically, given a generalized inertia matrix $\M(\q)$ and a potential function
$\Phi(\q)$, the Lagrangian is the difference between  kinetic and potential energies:
\begin{align} \label{eqn:Lagrangian}
	\textstyle
    \mathcal{L}(\q,\qd) = \frac{1}{2}\qd^\t\M(\q)\qd - \Phi(\q).
\end{align}
The Euler-Lagrange equation is given by
\begin{align}  \label{eqn:EulerLagrangEquation}
	\textstyle
    \frac{d}{dt}\partial_{\qd}\mathcal{L} - \partial_{\qb}\mathcal{L} = \tau_{\mathrm{ext}}
\end{align}
where $\tau_{\mathrm{ext}}$ is the external force applied on the system.  Applying \eqref{eqn:EulerLagrangEquation} to the Lagrangian \eqref{eqn:Lagrangian}
gives the equations of motion:
\begin{align}\label{eqn:LagrangianMechanics}
    \M(\qb)\qdd + \C(\q,\qd)\qd + \nabla \Phi(\q) = \tau_{\mathrm{ext}},
\end{align}
where $\C(\q, \qd)\qd = \dot{\M}(\qb,\qd)\qd - \frac{d}{dt}\left(\frac{1}{2}\qd^\t\M(\q)\qd\right)$.
For convenience, we will define this term as
\begin{align} \label{eqn:DirectCurvatureCalc}
\textstyle\bm\xi_{\Mb}(\q,\qd) = \dot{\M}(\qb,\qd)\qd - \frac{d}{dt}\left(\frac{1}{2}\qd^\t\M(\q)\qd\right)
\end{align}
which will play an important role when we discuss about the geometry of implicit task spaces. (This definition is consistent with the curvature term in GDSs that we later generalize.)

\vspace{1mm}
\subsubsection{Ambient Geometric Mechanics}
\label{sec:ambient geometry}

Geometric mechanics \cite{bullo2004geometric} is a reformulation of classical mechanics that builds on
the observation that classical mechanical systems evolve as geodesics across a
Riemannian manifold whose geometry is defined by the system's inertia matrix. We
can see what this means by exploring a simple example, which
we will use to derive an operational space QP similar in form to
\eqref{eqn:ConstantWeightControlPointPrioritiesQP}.

\green{
Specifically let $\x = \psi(\q)$ be an arbitrary differentiable task map
$\psi:\R^d\rightarrow\R^n$ where $n \geq d$.
The task map $\psi$ defines a $d$-dimensional sub-manifold $\XX=\{ \x: \x = \psi(\q), \q \in \CC \}$ of the $n$-dimensional ambient Euclidean task space $R^n$.
Without loss of generality, we suppose $\psi$ is full rank. (Reduced
rank $\psi$ results in a similar geometry, with sub-manifold dimensionality
matching the rank, but we would need to slightly modify the linear algebra used in the following discussion.)
Then we can define a positive definite matrix $\M(\q) = \J(\qb)^\t\J(\qb)$ that changes as a function of
configuration, where $\J(\q) = \partial_\qb \phi(\qb)$ is the Jacobian of the task map.

Geometric mechanics states that
we can think of $\M(\q)$ as both the generalized inertia matrix of a mechanical
system defining a dynamic
behavior
$\M(\q)\qdd + \C(\q,\qd)\qd = \tau_{\mathrm{ext}}$ (see also Equation~\eqref{eqn:LagrangianMechanics}),
and equivalently as a Riemannian metric defining an inner product
$\langle\qd_1,\qd_2\rangle_\M = \qd_1^\t\M(\qb)\qd_2$ on the
tangent space (for our purposes, the space of velocities $\qd$ at a given $\q$) of the configuration space $\CC$ (the manifold where $\qb$ lives).
In other words, the kinetic energy of the mechanical system is given by the
norm of $\qd$ with respect to the inner product defined by the metric $\M(\q)$:
\begin{align}
	\textstyle
    \mathcal{K}(\q,\qd) = \frac{1}{2}\qd^\t\M(\q)\qd
    = \|\qd\|_\M^2 = \langle\qd, \qd\rangle_\M.
\end{align}
In this particular, case with $\M = \J^\t\J$, this kinetic energy is also equal to the Euclidean velocity
in the task space
\begin{align} \label{eqn:KineticEnergyTaskSpaceVelocity}
	\textstyle
    \mathcal{K}(\q,\qd) = \frac{1}{2}\qd^\t\M(\q)\qd
    = \frac{1}{2}\qd^\t\big(\J^\t\J\big)\qd = \frac{1}{2}\|\xd\|^2.
\end{align}

Note that these Euclidean velocities are vectors living in the tangent space of
the ambient embedded manifold spanned by the columns of $\J(\qb)$, which can change with different $\qb$. (This tanget space is a first-order Taylor approximation to the surface at a point $\x_0 = \phi(\q_0) \in \XX$ for some $\qb_0$ in the sense $\x \approx \x_0 + \J(\q - \q_0)$.)
Importantly, the Euclidean inner product between velocities $\xd_1,\xd_2$ in
the task space $\x$
induces a generalized inner product between corresponding velocities
$\qd_1,\qd_2$ in the configuration space $\q$ in the sense
$\xd_1^\t\xd_2 = \qd_1^\t\M\qd_2$, for $\xd_1 = \Jb(\q) \qd_1$ and $\xd_2 = \Jb(\q) \qd_2$.
This connection between inner products offers a concrete connection between mechanics and geometry which we can exploit to link the system's equations of motion to geodesics across $\mathcal{X}$.
}

\subsubsection{Force-Free Mechanical Systems}

For systems without potential functions and external forces, we can get
insights into the connection between dynamics
and geodesics from the view point of Lagrangian mechanics as well.
The Lagrangian \eqref{eqn:Lagrangian} in this case
simplifies to $\mathcal{L} = \frac{1}{2}\qd^\t\M(\qb)\qd - \Phi(\q) = \frac{1}{2}\qd^\t\M(\qb)\qd$.
The Euler-Lagrange equation in \eqref{eqn:EulerLagrangEquation} is the first-order
optimality condition of an {action functional} which measures the time integral
of the Lagrangian across a trajectory, which takes a nice minimization form in this case
\begin{align}
	\textstyle
    \min_\xi \int_a^b \frac{1}{2} \qd^\t\M(\qb)\qd dt
    \ \ \Leftrightarrow\ \
    \min_\xi \int_a^b \frac{1}{2} \|\xd\|^2 dt,
\end{align}
where $\xi$ is a trajectory through the configuration space $\CC$.
\green{
One can show that an optimal trajectory to this length minimizion problem has a
constant speed through the ambient space $\|\xd\|$.
In other words, following the dynamics given by the Euler-Lagrangian equation }will curve across
the sub-manifold $\XX$ along a trajectory that is as straight as possible
without speeding up or slowing down.

Another way to characterize this statement
is to say the system never accelerates tangentially to the sub-manifold $\XX$,
i.e. it has no component of acceleration parallel to the tangent space. The curve certainly must accelerate to avoid diverging from the sub-manifold $\XX$, but that acceleration is always purely orthogonal to its tangent space.
Since we know that $\J$ spans the
tangent space, we can capture that sentiment fully
in the following simple equation:
\begin{align} \label{eqn:AmbientGeodesicEquation}
    \J^\t\xdd = \zero.
\end{align}
Plugging in $\xdd = \J\qdd + \Jd\qd$ we get
\begin{align} \label{eqn:TaskSpaceEqnsOfMotion}
    &\J^\t\xdd = \J^\t\left(\J\qdd + \Jd\qd\right) = \zero \nonumber \\
    &\ \ \Leftrightarrow \left(\J^\t\J\right)\qdd + \J^\t\Jd\qd = \zero.
\end{align}
Comparing \eqref{eqn:TaskSpaceEqnsOfMotion} to \eqref{eqn:LagrangianMechanics} (with zero potential and
external forces),
since we already know $\J^\t\J = \M(\q)$, we can formally
prove the connection between geodesics and classical mechanical dynamics if we can show that $\J^\t\Jd\qd = \C(\q,\qd)\qd$.
The required calculation is fairly involved, so we omit it here but note
for those inclined that it is easiest to perform using tensor notation
and the Einstein summation convention as is common in differential geometry.
This equivalence also appears as by-product fo our \flow and GDS analysis, as we will revisit in \cref{app:GDSs} as \cref{lm:Coriolos identity}.

\vspace{1mm}
\subsubsection{Forced Mechanical Systems and Geometric Control}

So far we have derived only the unforced behavior of this system as natural geodesic flow across the sub-manifold $\XX$.
To understand how desired accelerations contribute to
the least squares properties of the system, we express
\eqref{eqn:TaskSpaceEqnsOfMotion}
in $\x$ by pushing them through the identity $\xdd = \J\qdd + \Jd\qd$ and examine
how arbitrary motion across the sub-manifold $\XX$ decomposes.
Combining the equations of motion in $\xdd$ with \eqref{eqn:TaskSpaceEqnsOfMotion}, we have
\begin{align} \label{eqn:GeodesicInTaskSpace}
    \xdd
    &= \J\qdd + \Jd\qd = -\J\big(\J^\t\J\big)^{-1}\J^\t\Jd\qd + \Jd\qd \nonumber \\
    &= \left(\I - \J\big(\J^\t\J\big)^{-1}\J^\t\right)\Jd\qd\nonumber  \\
    &= \mathbf{P}_\perp \Jd\qd,
\end{align}
where in the final expression, the matrix
$\mathbf{P}_\perp = \I - \mathbf{P}_\paral$ with
$\mathbf{P}_\paral = \J\big(\J^\t\J\big)^{-1}\J^\t$ is the nullspace
projection operator projecting onto the space orthogonal to the tangent space
(spanned by the Jacobian $\J$). Note again these projections $\mathbf{P}_\perp$  and $\mathbf{P}_\paral$ are functions of the configuration $\qb$.
Therefore these geodesics accelerate only orthogonally to the tangent space. This implies that {any} trajectory,
traveling on the sub-manifold $\mathcal{X}$ but deviating from geodesics, would necessarily maintain an acceleration
component parallel to the tangent space. Let us denote this extra acceleration as $\xdd_\paral$. Importantly,
any such trajectory must still accelerate exactly as
\eqref{eqn:GeodesicInTaskSpace} in the orthogonal direction in order
to stay moving along the sub-manifold $\XX$. Therefore, we see that the overall acceleration of a trajectory on $\XX$ can
decomposed nicely into the geodesic acceleration and the tangential acceleration:
\begin{align} \label{eq:decomposition of acc}
    \xdd
    &= \mathbf{P}_\perp \Jd\qd + \mathbf{P}_{\paral} \xdd^d,
\end{align}
where $\xdd^d$ is any vector of ``desired'' acceleration whose tangential component matches the tangential
acceleration of the given trajectory, i.e.  $\mathbf{P}_\paral\xdd^d = \xdd_\paral$.

With this insight, now we show how the decomposition~\eqref{eq:decomposition of acc} is related to and generalizes~\eqref{eqn:ConstantWeightControlPointPrioritiesQP}.
This is based on the observation that \eqref{eq:decomposition of acc} is the same as the solution to the least-squared problem below:
\begin{align} \label{eqn:AbstractTaskSpacesQP}
    \textstyle \min_{\qdd} \frac{1}{2}\|\xdd^d - \xdd\|^2 \quad  \mathrm{s.t.} \ \ \xdd = \J\qdd + \Jd\qd.
\end{align}
This equivalence can be easily seen by resolving the constraint, setting the gradient
of the resulting quadratic to zero, i.e.,
\begin{align} \label{eqn:GeometricEqnsOfMotionInCoordinates}
    \J^\t\J\qdd + \J^\t\Jd\qd = \J^\t\xdd^d
\end{align}
and re-expressing the optimal solution in $\xdd$.
This relationship demonstrates that a QP very similar
in structure to \eqref{eqn:ConstantWeightControlPointPrioritiesQP}.
Both express dynamics of the forced system as $\big(\J^\t\J\big)\qdd + \J^\t\Jd\qd = \J^\t\f$
with $\f = \xdd^d$ (task space coordinates were chosen specifically
so that weights (equivalently inertia) would be $1$, so forces and accelerations
have the same units---different task coordinates would result in constant weights,
corresponding to a notion of constant mass in the ambient space).

\vspace{1mm}
\subsubsection{The Curvature Terms}

\newcommand{\Proj}{{\mathbf{P}}}

As a side note,
the above discussion
offers insight into the term $\C(\q,\qd)\qd = \J^\t\Jd\qd$.
The term $\Jd\qd$ captures
components describing both curvature of the manifold through the
ambient task space
and components describing how the specific coordinate $\q$ (tangentially) curves across
the sub-manifold $\XX$.
Explicitly, $\xdd^c \coloneqq \Jd\qd$ has units of acceleration in the ambient space
and captures how the tangent space (given by the columns of $\J$)
changes in the direction of motion.
The acceleration $\xdd^c$ decomposes as
$\xdd^c = \xdd^c_\perp + \xdd^c_\paral$
into two orthogonal
components consisting of a component perpendicular to the tangent space
$\xdd^c_\perp = \Proj_\perp\xdd^c$
and a component parallel to the tangent space
$\xdd^c_\paral = \Proj_\paral\xdd^c$.
The term $\Proj_\perp \Jd\qd = \Proj_\perp \xdd^c$ given in
\eqref{eqn:GeodesicInTaskSpace} extracts
specifically the perpendicular component $\xdd^c_\perp$. The other component
$\xdd^c_\paral$ is, therefore, in a sense irrelevant to fundamental geometric
behavior
of the underlying system, and is only required when expressing the behavior
in the specific coordinates $\q$. Indeed, when expressing the equations of
motion in $\q$, the related term manifests as
$\C(\q,\qd)\qd = \J^\t\Jd\qd = \J^\t\xdd^c_\paral$ since $\J^\t\xdd^c_\perp = \zero$,
and depends only on the parallel component $\xdd^c_\paral$.
This observation emphasizes why we designate the term
{\it fictitious forces}. Here and below
we will consider these terms to be {\it curvature} terms
as they compensate for curvature in the system coordinates.

\subsection{Non-constant Weights and Implicit Task Spaces}
\label{sec:NonconstantWeightsAndImplicitTaskSpaces}

Section~\ref{sec:AbstractTaskSpaces} above derived the geometric perspective of equations of motion, but only for mechanical systems whose inertia matrix
(equivalently Riemannian metric) can be expressed as $\M(\q)= \J(\qb)^\t\J(\qb)$ globally for
some map $\x = \psi(\q)$ with $\J(\qb) = \partial_\qb\phi(\qb)$.
Fortunately, due to a deep and fundamental theorem proved by John Nash
in 1956, called the Nash embedding theorem \cite{NashEmbeddingTheorem1956},
{\it all} Riemannian manifolds, and hence {\it all} mechanical systems can
be expressed this way, so the arguments of Section~\ref{sec:AbstractTaskSpaces}
hold without loss of generality for all mechanical systems. It is called
an embedding theorem because the map $\x = \psi(\q)$ acts to embed the manifold $\CC$
into a higher-dimensional ambient Euclidean space where we can replace implicit
geometry represented by the metric $\M(\q)$ with an explicit sub-manifold in the ambient space.

{
    \newcommand{\Jz}{\J_\zeta}
    \newcommand{\Jdz}{\Jd_\zeta}

This ambient representation
is
convenient for understanding and visualizing the
nonlinear geometry of a mechanical system. However, it is unfortunately often difficult, or
even impossible, to find a closed form expression for a task map $\x =\psi(\q)$ from a given metric $\M(\q)$. We therefore cannot
rely on our ability to operate directly in the ambient space using the QP given in~\eqref{eqn:AbstractTaskSpacesQP}.

This subsection addresses that problem
by deriving a QP expression analogous to
\eqref{eqn:AbstractTaskSpacesQP}, but using task space weights that are
general non-constant positive definite matrices (which we will see are the same as Riemannian
metrics).
We will derive this expression by considering again the ambient setting, but assuming that the unknown embedding can be decomposed in the composition of a known task space  $\x = \psi(\q)$ and then an known map from the task space to the ambient Euclidean space $\z =
\zeta(\x)$ described in the Nash's theorem.
We suppose the priority weight is given as the induced Riemannian metric $\G(\x) = \Jz(\x)^\t\Jz(\x)$ on $\x$ defined by the second map $\zeta$. We note that the final result will be expressed entirely in terms of $\G(\x)$, so
it can be used without explicit knowledge of $\zeta$.

Suppose we have a Riemannian metric (equivalently, an  inertia matrix)
$\M$ which decomposes as $\M = \J^\t\J$, where  $\J = \Jz\Jp$ is the Jacobian of the composite map $\z = \zeta\circ\phi(\q)$ which itself consists of two parts
$\x = \phi(\q)$ and $\z = \zeta(\x)$.
Let $\f_\z$ denote the task space force in the ambient Euclidean space, and define $\f_\x = \J_\zeta^\t\f_\z$
Because the intermediate task space metric is $\G = \Jz^\t\Jz$, we can derive by~\eqref{eqn:GeometricEqnsOfMotionInCoordinates}
\begin{align*}
    &\J^\t\J\qdd + \J^\t\Jd\qd = \J^\t\f_\z \\
    &\Rightarrow \ \
    \left(\Jp^\t\big(\Jz^\t\Jz\big)\Jp\right) \qdd
    + \Jp^\t\Jz^\t\frac{d}{dt}\left(\Jz\Jp\right)\qd = \J_\phi^\t\J_\zeta^\t\f_\z \\
    &\Rightarrow\ \
    \Jp^\t\G\Jp \qdd
        + \Jp^\t\Jz^\t\left(\Jdz\Jp + \Jz\Jdp\right) \qd
        = \Jp^\t\f_\x \\
    &\Rightarrow\ \
    \Jp^\t\G\Jp \qdd
        + \Jp^\t\big(\Jz^\t\Jdz\xd\big) + \Jp^\t\big(\Jz^\t\Jz\big)\Jdp \qd
        = \Jp^\t\f_\x \\
    &\Rightarrow\ \
    \Jp^\t\G\Jp \qdd
        + \Jp^\t\G\Jdp\qd = \Jp^\t\left(\f_x - \bm\xi_\G\right).
\end{align*}
where we recall $\bm\xi_\G$ is given in \eqref{eqn:DirectCurvatureCalc}.
}%
Rearranging that final expression and denoting
$\underline{\xdd}^d \coloneqq \xdd^d - \G^{-1}\bm\xi_\G$ with
$\xdd^d = \G^{-1}\f_\x$,
we can write the above equation as
\begin{align}
    \Jp^\t\G\left(\underline{\xdd}^d - \big(\Jp\qdd + \Jdp\qd\big)\right) = \zero,
\end{align}
which is the first-order optimality condition of the QP
\begin{align} \label{eqn:GeometricControlQP}
    \textstyle \min_\qdd \frac{1}{2}\|\underline{\xdd}^d - \xdd\|_\G^2, \qquad  \mathrm{s.t.}\ \ \xdd = \Jp\qdd + \Jdp\qd.
\end{align}
This QP is expressed in terms of
the task map $\x = \psi(\q)$, the task space metric $\G(\xb)$, the
task space desired accelerations $\xdd^d$, and the curvature
term $\bm\xi_\G$ derived from the task space metric $\G(\xb)$.
The QP follows a very similar pattern to the QPs described
above, but this time the priority weight matrix $\G$ is a non-constant
function of $\x$. The one modification required to reach this matching
form is to augment the desired acceleration $\xdd^d$ with the curvature term $\bm\xi_\Gb$
calculated from $\G$ using \eqref{eqn:DirectCurvatureCalc} to get the target
$\underline{\xdd}^d = \xdd^d - \G^{-1}\bm\xi_\G$.
Importantly, while we start by assuming the map $\z = \zeta(\x)$, at the end we show that we actually only need to know $\G$.

\subsection{Limitations of Geometric Control}
\label{sec:LimitationsOfGeometricControl}

Even with the tools of geometric mechanics, the final QP given in
\eqref{eqn:GeometricControlQP} can still only express task priority weights as
positive definite matrices that vary as a function of configuration (i.e.
position). Frequently, more nuanced control over those priorities is crucial.
For instance, collision avoidance tasks should activate when the control-point
is close to an obstacle and heading toward it, but they should deactivate either
when the control-point is far from the obstacle \emph{or} when it is moving away from
the obstacle, regardless of its proximity.  Importantly, reducing the desired
acceleration to zero in these cases is not enough---when these tasks deactivate,
they should drop entirely from the equation rather than voting with high weight
for zero acceleration. Enabling priorities to vary as a function of the full
robot state (configuration {\it and} velocity) is therefore paramount.

The theory of \flow and GDSs developed below generalizes
geometric mechanics to enable expressing these more nuanced priority matrices
while maintaining stability.
Additionally, since geometric control theory itself
is quite abstract, we build on results reducing the calculations to
recursive least squares similar to that given in
Section~\ref{sec:NonconstantWeightsAndImplicitTaskSpaces}
to derive a concrete tree data structure to aid in the design
of controllers within this energy shaping framework.

\section{\flow} \label{sec:RMP algebra}

\red{\flow is an efficient manifold-oriented computational graph for automatic generation of motion policies that can tackle multiple task specifications. Let $\TT_{l_i}$ denote the $i$th subtask space.
\flow is aimed for problems with a task space $\TT = \{\TT_{l_i}\}$ that is related to the configuration space $\CC$ through a tree-structured task map $\psi$, in which $\CC$ is the root node and the subtask spaces $\{\TT_{l_i}\}$ are the leaf nodes.}
Given user-specified motion policies $\{\pi_{l_i}\}$ on the subtask spaces
$\{\TT_{l_i}\}$ as RMPs, \flow is designed to \emph{consistently} combine these subtask policies into a global policy $\pi$ on $\CC$.

To realize this idea, \flow introduces 1) a data structure, called the \emph{\tree}, to describe the tree-structured task map $\psi$ and the policies, and 2) a set of operators, called the \emph{\algebra}, to propagate information across the \tree. 
At time $t$, \flow operates in two steps to compute $\pi(\q(t),\qd(t))$: it first performs a \emph{forward pass} to propagate the state from the root node (i.e., $\CC$) to the leaf nodes (i.e., $\{\TT_{l_i}\}$); then it performs a \emph{backward pass} to propagate the RMPs from the leaf nodes to the root node \red{while tracking their geometric information to achieve consistency}. 
These two steps are realized by recursive use of \algebra, exploiting shared computation paths arising from the tree structure to maximize efficiency.

\red{
In the following, we describe the details of \flow and give some useful examples of subtask motion policies.
}

\subsection{Structured Task Maps}

\begin{figure}[t]\vspace{-4mm}
	\centering
	\centering
	\subfloat[\label{fig:taskmap1}]{
		\includegraphics[trim={0 20 0 20},clip, width=0.3\columnwidth,keepaspectratio]{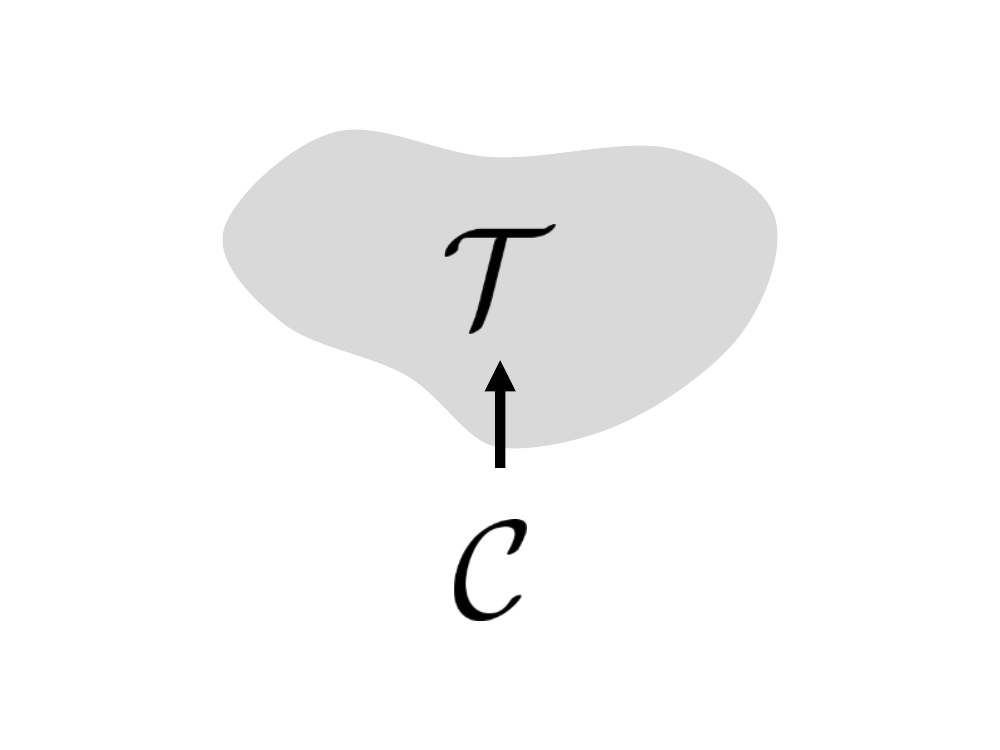}
	}
	\subfloat[\label{fig:taskmap2}]{
		\includegraphics[trim={0 20 0 20},clip, width=0.3\columnwidth,keepaspectratio]{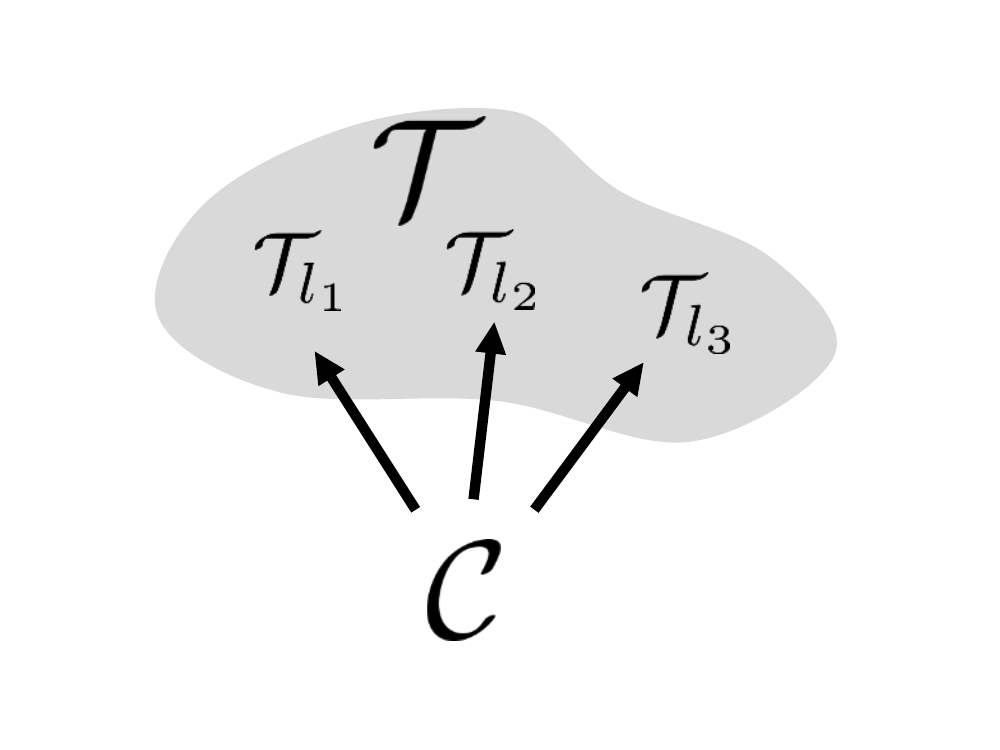}
	}
	\subfloat[\label{fig:taskmap3}]{
		\includegraphics[trim={0 20 0 20},clip, width=0.3\columnwidth,keepaspectratio]{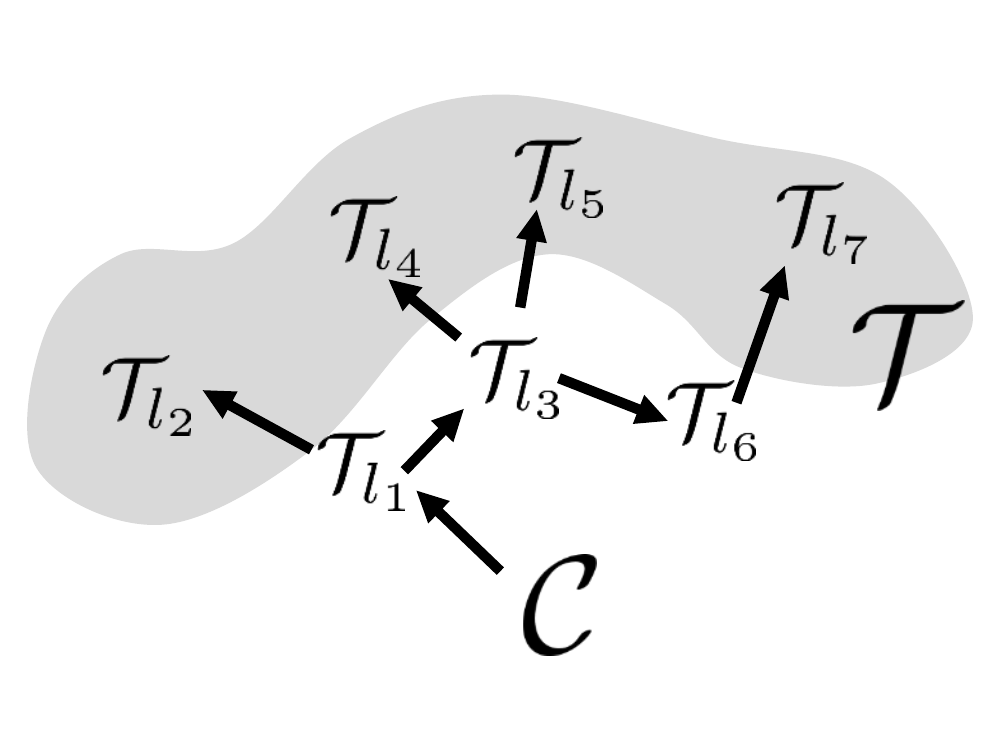}
	}
	\caption{Tree-structured task maps}
	\label{fig:taskmaps}
\end{figure}
In many applications, the task-space manifold $\TT$ is structured. In this paper, we consider the case where the task map $\psi$ can be expressed through a tree-structured composition of transformations $\{\psi_{e_i}\}$, where $\psi_{e_i}$ is the $i$th transformation. Fig.~\ref{fig:taskmaps} illustrates some common examples, where each node denotes a manifold and each edge denotes a transformation. This family trivially includes the unstructured task space $\TT$ (Fig.~\ref{fig:taskmaps}a) and the product manifold $\TT = \TT_{l_1} \times \dots \times \TT_{l_K}$ (Fig.~\ref{fig:taskmaps}b), where $K$ is the number of subtasks.
A more interesting example is the kinematic tree (Fig.~\ref{fig:taskmaps}c), where
the task map considers the relationship between the configuration space $\CC$ (the root node) and a collection of subtask spaces (the leaf nodes) that describe, e.g., the tracking and obstacle avoidance tasks along a multi-DOF robot. 

The main motivation of explicitly handling the structure in the task map $\psi$ is two-fold. First, it allows \flow to exploit computation shared across different subtask maps. Second, it allows the user to focus on designing motion policies for each subtask individually, which is easier than directly designing a global policy for the entire task space $\TT$. For example, $\TT$ may describe the problem of humanoid walking, which includes staying balanced, scheduling contacts, and avoiding collisions. Directly parameterizing a policy to satisfy all these objectives can be daunting, whereas designing a policy for each subtask is more feasible.

\subsection{Riemannian Motion Policies (RMPs)}

Knowing the structure of the task map is not sufficient for consistently combining subtask policies: we require some information about the motion policies' behaviors~\cite{ratliff2018riemannian}.
Toward this end, we adopt a more informative description of motion policies, called RMPs~\cite{ratliff2018riemannian},  for the nodes of the \tree.

An RMP describes a second-order differential equation along with its geometric information on a smooth manifold.
Specifically, let $\MM$ be an $m$-dimensional manifold with coordinate $\x \in \R^m$.
The \emph{canonical form} of an RMP on $\MM$ is a pair $(\ab, \M)^\MM$, where
$\ab : \R^m \times \R^m \to \R^m$ is a continuous motion policy and $\M: \R^m \times \R^m \to \R_+^{m\times m}$ is a differentiable map.
Borrowing terminology from mechanics (cf. \cref{sec:AbstractTaskSpaces}), we call $\ab(\x,\xd)$ the \emph{desired acceleration} and $\M(\x,\xd)$ the \emph{inertia matrix} at $(\x,\xd)$, respectively.\footnote{Here we adopt a slightly different terminology from~\cite{ratliff2018riemannian}. We note that $\M$ and $\f$ do not necessarily correspond to the inertia and force of a physical mechanical system. }
For now, we can intuitively think that $\M$ defines the directional importance of $\ab$ when it is combined with other motion policies. Later in Section~\ref{sec:analysis}, we will show that $\M$ is closely related to the concept of Riemannian metric, which describes how the space is stretched along the curve generated by $\ab$; when $\M$ depends on the state, the space becomes \emph{non-Euclidean}.

In this paper,
we additionally introduce a new RMP form, called the \emph{natural form}. Given an RMP in its canonical form $(\ab, \M)^\MM$, we define another pair as its natural form $[\f, \M]^\MM$, where $\f = \M \ab$ is the \emph{desired force} map. While the transformation between these two forms may look trivial, their distinction will be useful later when we introduce the \algebra.

\subsection{\tree}

The \tree is the core data structure used by \flow.
An \tree is a directed tree, in which each node represents an RMP and its state, and each edge corresponds to a transformation between manifolds.
The root node of the \tree describes the global policy $\pi$ on $\CC$, and the leaf nodes describe the local policies $\{\pi_{l_i}\}$ on $\{\TT_{l_i}\}$.

To illustrate, let us consider a node $u$ and its $K$ child nodes $\{v_i \}_{i=1}^K$. Suppose $u$ describes an RMP $[\f, \M ]^\MM$ and $v_i$ describes an RMP $ [\f_i, \M_i ]^{\NN_i}$, where $\NN_i = \psi_{e_i}(\MM)$ for some $\psi_{e_i}$. Then we write $u = ((\x,\xd), [\f, \M ]^\MM)$ and $v_i = ((\y_i,\yd_i), [\f_{i}, \M_i ]^{\NN_i})$, \red{where $(\x,\xd)$ is the state of $u$, and $(\y_i,\yd_i)$ is the state of $v_i$}; the edge connecting $u$ and $v_i$ points from  $u$ to $v_i$ along $\psi_{e_i}$. We will continue to use this example to illustrate how \algebra propagates the information across the \tree.

\subsection{\algebra} \label{sec:RMPAlgebra}

The \algebra of \flow consists of three operators (\pushforward, \pullback, and \resolve) to propagate information across the \tree.\footnote{Precisely they propagate the numerical values of RMPs and states at a particular time.} They form the basis of the forward and backward passes for automatic policy generation, described in the next section.
\begin{enumerate}

	\item  \pushforward is the operator to forward propagate the \emph{state} from a parent node to its child nodes. Using the previous example, given $(\x,\xd)$ from $u$, it computes $(\y_i, \yd_i) = (\psi_{e_i}(\x) , \J_i (\x) \xd )$ for each child node $v_i$, where $\J_i = \partial_\x \psi_{e_i}$ is a Jacobian matrix. The name ``pushforward" comes from the linear transformation of tangent vector $\xd$ to the image tangent vector $\yd_i$.

	\vspace{1mm}
	\item \pullback is the operator to backward propagate the natural-formed RMPs from the child nodes to the parent node. It is done by setting $[\f, \M ]^\MM$ with
	\begin{align} \label{eq:natural pullback}
		\textstyle
		\f = \sum_{i=1}^{K} \J_i^\t (\f_{i} - \M_i \dot{\J}_i \xd)
		\text{,\hspace{1mm}}
		\M = \sum_{i=1}^{K} \J_i^\t \M_i \J_i
	\end{align}
	The name ``pullback" comes from the linear transformations of the cotangent vector (1-form) $\f_{i} - \M_i \dot{\J_i} \xd$  and the inertia matrix (2-form) $\M_i$. 
	In summary, velocities can be pushfowarded along the direction of $\psi_i$, and forces and inertial matrices can be pullbacked in the opposite direction.

	To gain more intuition of \pullback, we write \pullback in the canonical form of RMPs. It can be shown that the canonical form $(\ab, \M)^{\MM}$ of the natural form $[\f,\M]^\MM$ above is the solution to a least-squares problem (cf. \cref{sec:NonconstantWeightsAndImplicitTaskSpaces}):
	\begin{align} \label{eq:least-square problem of pullback}
		\ab
		&=\textstyle
		\argmin_{\ab'} \frac{1}{2} \sum_{i=1}^{K} \norm{\J_i \ab' + \dot\J_i \xd -  \ab_{i} }_{\M_i}^2
		   \nonumber  \\
		&=\textstyle
			(\sum_{i=1}^{K} \J_i^\t \M_i \J_i )^{\dagger} (\sum_{i=1}^{K} \J_i^\t \M_i (  \ab_{i} - \dot\J_i \xd   )  )
	\end{align}
	where $ \ab_i = \M_i^{\dagger}\f_i$ and $\norm{\cdot}_{\M_i}^2 = \lr{\cdot}{\M_i \cdot} \red{=\cdot^\T \M_i \cdot}$. Because $\ydd_i =  \J_i \xdd + \dot\J_i \xd$, \pullback attempts to find an $\ab$ that can realize the desired accelerations $\{\ab_{i}\}$ while trading off approximation errors with an importance weight defined by the inertia matrix $\M_i(\y_i,\yd_i)$.
	The use of state dependent importance weights is a distinctive feature of \flow. It allows \flow to activate different RMPs according to \emph{both} configuration and velocity (see Section~\ref{sec:example RMPs} for examples).
	Finally, we note that the \pullback operator defined in this paper is slightly different from the original definition given in~\cite{ratliff2018riemannian}, which ignores the term $\dot\J_i \xd$ in~\eqref{eq:least-square problem of pullback}. While ignoring $\dot\J_i \xd$ does not necessarily destabilize the system~\cite{lo2016virtual}, its inclusion is critical to implement consistent policy behaviors. {}

	\vspace{1mm}
	\item \resolve is the last operator of \algebra. It maps an RMP from its natural form to its canonical form. Given $[\f, \M]^{\MM}$, it outputs $(\ab, \M)^{\MM}$ with $\ab = \M^{\dagger} \f$, where $\dagger$ denotes Moore-Penrose inverse. The use of pseudo-inverse is because in general the inertia matrix is only positive semi-definite. Therefore, we also call the natural form of $[\f, \M]^{\MM}$ the \emph{unresolved form}, as potentially it can be realized by multiple RMPs in the canonical form.
\end{enumerate}

\subsection{Algorithm: Motion Policy Generation}

Now we show how \flow uses the \tree and \algebra to generate a global policy $\pi$ on $\CC$. Suppose each subtask policy is provided as an RMP \red{in the natural form}. 
First, we construct an \tree with the same structure as $\psi$, where we assign subtask RMPs as the leaf nodes and the global RMP $[\f_r, \M_r]^\CC$ as the root node.
With the \tree specified, \flow can perform automatic policy generation. At every time instance, it first performs a forward pass: it recursively calls \pushforward from the root node to the leaf nodes to update the state information in each node in the \tree.
Second, it performs a backward pass: it recursively calls \pullback from the leaf nodes to the root node to back propagate the values of the RMPs in the natural form. Finally, it calls \resolve at the root node to transform the global RMP $[\f_r, \M_r]^\CC$ into its canonical form $(\ab_r, \M_r)^\CC$ for policy execution (i.e. setting $\pi(\q,\qd) = \ab_r$).

The process of policy generation of \flow uses the tree structure for
computational efficiency. For $K$ subtasks, it has time complexity $O(K)$ in the
worst case (the case with a binary tree) as opposed to $O(K\log K)$ of a naive
implementation which does not exploit the tree structure,  \green{when running
on a serial computer.
In general, any RMP-tree representing the same leaf task spaces is equivalent,
although its computational properties may differ. For instance,
any leaf has a unique path from the root defining a series of
transforms the configuration space must go through to admit the leaf's space.
That leaf can, therefore, be separated from the tree
and linked to the root with a single edge
along which a single map defined as the composition of maps along that unique
path acts. Restructuring an RMP-tree into a star-shaped topology (i.e. the flat
structure depicted in Figure~\cref{fig:taskmap2}) with all
leaves linked directly to the root using this separation process
may be more amenable to parallel computation than the original tree.
Usually it would be more effective to separate out entire subtrees by linking
them to the root with a single edge containing the composition of transforms
along the unique path to that subtree's root. Separate computational nodes can
process these entire subtrees in parallel using the \flow algorithm, then pullback
the resulting RMPs to the root using message passing.
In general, there are many ways a tree can
be equivalently restructured to admit the same task spaces, and choosing the
right tree for a given computational architecture is a design choice.
}

Furthermore, all computations of \flow are carried out using matrix-multiplications, except for the final \resolve call, because the RMPs are expressed in the natural form in \pullback instead of the canonical form suggested originally in~\cite{ratliff2018riemannian}.
This design makes \flow numerically stable, as only one matrix inversion $\M_r^\dagger \f_r$ is performed at the root node with both $\f_r$ and $\M_r$  in the span of the same Jacobian matrix due to \pullback.

\vspace{-2mm}
\subsection{Example RMPs} \label{sec:example RMPs}

We give a quick overview of some RMPs useful in practice (see Appendix D of the technical report~\cite{cheng2018rmpflowarxiv} for further  discussion of these RMPs)
We recall from~\eqref{eq:least-square problem of pullback} that $\M$ dictates the directional importance of an RMP.

\vspace{1mm}
\subsubsection{Collision/joint limit avoidance}
Barrier-type RMPs are examples that use velocity dependent inertia matrices, which can express importance as a function of robot heading (a property that traditional mechanical principles fail to capture).
Here we demonstrate a collision avoidance policy in the 1D distance space $x = d(\q)$ to an obstacle. Let $g(x,\dot{x}) = w(x)u(\dot{x}) > 0$ for some functions $w$ and $u$. We consider a motion policy such that $ m(x,\dot{x})\ddot{x}+ \frac{1}{2}\dot{x}^2\partial_x g(x,\dot{x})
= -\partial_x \Phi(x) - b\dot{x}$ and define its inertia matrix $m(x,\dot{x}) = g(x,\dot{x}) + \frac{1}{2}\dot{x}\partial_{\dot{x}}g(x,\dot{x})$, where $\Phi$ is a potential and $b>0$ is a damper.
We choose $w(x)$ to increase as $x$ decreases
(close to the obstacle),  $u(\dot{x})$ to increase 
when $\dot{x}<0$ (moving toward the obstacle), and $u(\dot{x})$ to be constant when $\dot{x}\geq0$.
With this choice, the RMP can be turned off in \pullback when the robot heads away from the obstacle.
This motion policy is a GDS and $g$ is its metric (cf. Section~\ref{sec:GDS}); the terms $\frac{1}{2}\dot{x}\partial_{\dot{x}}g(x, \dot{x})$ and $\frac{1}{2}\dot{x}^2\partial_x g(x, \dot{x})$ are due to non-Euclidean geometry and produce natural repulsive behaviors \red{as the robot moves toward the obstacle, and little or no force when it starts to move away.}

\vspace{1mm}
\subsubsection{Target attractors}
Designing an attractor policy is relatively straightforward. For a task space with coordinate $\x$, we can consider an inertia matrix $\M(\x) \succ 0$ and a motion policy such that
$
\xdd = -\nabla\wt{\Phi} - \beta(\x)\xd - \M^{-1}\bm\xi_\M
$, where $\wt{\Phi}(\x) \approx \|\x\|$ is a smooth attractor potential, $\beta(\x)\geq0$ is a damper, and $\bm\xi_{\M}$ is a curvature term due to $\M$.
It can be shown that this differential equation is also a GDS~\cite[Appendix D]{cheng2018rmpflowarxiv}.

\vspace{1mm}
\subsubsection{Orientations}
As \flow directly works with manifold objects, orientation controllers become straightforward to design, independent of the choice of coordinate (cf. Section~\ref{sec:geometric properties}). 
For example, we can define RMPs on a robotic link's surface in any preferred coordinate (e.g. in one or two axes attached to an arbitrary point) with the above described attractor to control the orientation.

\section{Theoretical Analysis of \flow} \label{sec:analysis}

\newcommand{\sdot}[2]{\overset{\lower0.1em\hbox{$\scriptscriptstyle #2$}}{#1}}

We investigate the properties of \flow when the child-node motion policies belong to a class of differential equations, which we call \emph{structured geometric dynamical systems} (structured GDSs).
We present the following results.
\begin{enumerate}
	\item {\bf Closure}: We show that the \pullback operator retains a closure of structured GDSs. When the child-node motion policies are structured GDSs, the parent-node dynamics also belong to the same class.

	\vspace{1mm}
	\item {\bf Stability}: Using the closure property, we provide sufficient conditions for the feedback policy of \flow to be stable. In particular, we cover a class of dynamics with \emph{velocity-dependent} metrics that are new to the literature.

	\vspace{1mm}
	\item {\bf Invariance}: As its name suggests, \flow is closely related to differential geometry. We show that
	\flow is intrinsically coordinate-free. This means that a set of subtask RMPs designed for one robot can be transferred to another robot while maintaining the same task-space behaviors.
\end{enumerate}

\noindent \textit{Setup  }
\red{Below we consider the manifolds in the nodes of the \tree to be finite-dimensional and smooth. Without loss of generality, for now we assume that each manifold can be described in a single chart (i.e. using a global coordinate), so that we can write down the equations concretely using finite-dimensional variables. This restriction will be removed when we presents the coordinate-free form in Section~\ref{sec:geometric properties}. We also assume that all the maps are sufficiently smooth so the required derivatives are well defined.
\blueold{The proofs of this section can found in the Appendix B of the technical report~\cite{cheng2018rmpflowarxiv}.}}

\subsection{Geometric Dynamical Systems (GDSs)} \label{sec:GDS}

We first define a new family of dynamics, called GDSs, useful to specify RMPs on manifolds. \red{(Structured GDSs will be introduced shortly in the next section.) At a high-level, a GDS can be thought as a virtual mechanical system defined on a manifold with an inertia that depends on \emph{both} configuration and velocity.
Formally, let us consider an $m$-dimensional manifold $\MM$ with chart $(\MM, \x)$ (i.e. a coordinate system on $\MM$). Let $\Gb: \R^m \times \R^m \to \R^{m\times m}_{+}$, $\B: \R^m \times \R^m \to \R^{m\times m}_{+}$, and $\Phi: \R^m \to \R$ be sufficiently smooth functions.
We say a dynamical system on $\MM$ is a \emph{GDS} $(\MM, \Gb, \B, \Phi)$, if it satisfies the differential equation}
\begin{align} \label{eq:GDS}
	&\left(\Gb(\x,\xd) + \bm\Xi_{\G}(\x,\xd)\right) \xdd
	+ \bm\xi_{\G}(\x,\xd) \nonumber \\
	& = - \nabla_\x \Phi(\x) - \Bb(\x,\xd)\xd,
\end{align}
where we define
\begin{align*}
	\bm\Xi_{\G}(\x,\xd) &\coloneqq \textstyle \frac{1}{2} \sum_{i=1}^m  \dot{x}_i \partial_{\xd} \gb_{i}(\x,\xd)\\
	\bm\xi_{\G}(\x,\xd) &\coloneqq  \sdot{\Gb}{\x}(\x,\xd) \xd - \frac{1}{2} \nabla_\x (\xd^\t \Gb(\x,\xd) \xd)\\
	\sdot{\Gb}{\xb}(\x,\xd) &\coloneqq  [\partial_{\x}  \gb_{i} (\x,\xd) \xd]_{i=1}^m
\end{align*}
and $\gb_{i}(\x,\xd)$ is the $i$th column of $\Gb(\x,\xd)$.
We refer to $\Gb(\x,\xd)$ as the \emph{metric} matrix,
$\B(\x,\xd)$ as the \emph{damping} matrix, and $\Phi(\x)$ as the \emph{potential} function which is lower-bounded.
\red{In addition, we call the term in front of $\xdd$ in~\eqref{eq:GDS},
\begin{align} \label{eq:inertia matrix}
\Mb(\x, \xd) \coloneqq \Gb(\x,\xd) + \bm\Xi_{\G}(\x,\xd),
\end{align}
the \emph{inertia} matrix of GDS $(\MM, \Gb, \B, \Phi)$, which can be asymmetric. When $\M(\x,\xd) $ is nonsingular, we say the GDS is \emph{non-degenerate}. We will assume~\eqref{eq:GDS} is non-degenerate for now so that it uniquely defines a differential equation. The discussion on the general case is postponed to Appendix~\ref{app:degnerate GDSs}.}

In GDSs, $\Gb(\x,\xd)$ induces a metric of $\xd$, 
measuring its length as $\frac{1}{2} \xd^\t \Gb(\x,\xd) \xd$.
When $\Gb(\x,\xd)$ depends on $\x$ and $\xd$, it also induces non-trivial \emph{curvature} terms $\bm\Xi(\x,\xd)$ and $\bm\xi(\x,\xd)$.
\red{
In a particular case when $\G(\x,\xd) = \G(\x)$ (i.e. it depends on configuration only), the GDSs reduce to the widely studied \emph{simple mechanical systems} (SMSs) in geometric mechanics~\cite{bullo2004geometric} (see also \eqref{eqn:LagrangianMechanics})
\begin{align} \label{eq:SMS}
\Mb(\x) \xdd
+  \C(\x,\xd)\xd + \nabla_\x \Phi(\x) = -\Bb(\x,\xd)\xd
\end{align}
where the Coriolis force $\Cb(\x,\xd) \xd$ can be shown equal to
$\bm\xi_{\Gb}(\x,\xd)$. In this special case, we have $\Mb(\x) = \Gb(\x)$, i.e.,
the inertia matrix is the same as the metric matrix (this is exactly the finding
in geometric mechanics discussed in \blueold{\cref{sec:ambient geometry}}).
We will revisit the connection between GDSs and SMSs again in \cref{app:GDSs} (and show why $\Cb(\x,\xd) \xd=  \bm\xi_{\Gb}(\x,\xd)$)
 after the analysis of geometric properties of GDSs. For now, we can think of GDSs as generalization of SMSs to have inertia $\G(\x,\xd)$ and metric $\Mb(\x,\xd)$ that also change with velocity!
This velocity-dependent extension is important and non-trivial.}
As discussed in earlier Section~\ref{sec:example RMPs}, it generalizes the dynamics of classical rigid-body systems, allowing the space to morph according to the velocity direction.

Finally, as its name hints, GDSs possess geometric properties. Particularly, when $\Gb(\x,\xd)$ is invertible, the left-hand side of~\eqref{eq:GDS} is related to a quantity $\ab_{\G} = \xdd + \Gb(\x,\xd)^{-1} ( \bm\Xi_{\G}(\x,\xd) \xdd
+ \bm\xi_{\G}(\x,\xd) ) $, known as the \emph{geometric acceleration} (cf. Section~\ref{sec:geometric properties}).
\red{(Therefore these terms must not be separated; e.g. $\Gb(\x,\xd)  \xdd $ alone may not possess particular meaning.)}
In other words, we can think of~\eqref{eq:GDS} as setting $\ab_{\G}$ along the negative natural gradient $-\G(\x,\xd)^{-1}\nabla_\x \Phi(\x)$ while imposing damping $-\G(\x,\xd)^{-1}\Bb(\x,\xd)\xd$.

\vspace{-2mm}
\subsection{Closure} \label{sec:consistency}

\red{Earlier, we argued vaguely that by tracking the geometry in \pullback in~\eqref{eq:natural pullback} through propagating RMPs instead of just motion policies}, the task properties can be preserved. Here, we formalize this consistency concept of \flow as a closure of differential equations, named structured GDSs. Structured GDSs augment GDSs with information on how the metric matrix $\Gb$ factorizes. \red{We call such information a \emph{structure}.}
Specifically, suppose $\Gb$ has a structure $\SS$ that factorizes $\G(\x,\xd) = \J(\x)^\t \Hb(\y,\yd) \J(\x)$, where 
$\y: \x \mapsto \y(\x) \in \R^n$  and $\Hb: \R^n \times \R^n \to \R^{n\times n}_+$, and $\J(\x) = \partial_\x \y$ is the Jacobian.
We say a dynamical system on $\MM$ is a \emph{structured GDS}  $(\MM, \G, \B, \Phi)_{\SS}$ if it satisfies the differential equation
\begin{align} \label{eq:structured GDS}
	&\left(\Gb(\x,\xd) + \bm\Xi_{\G}(\x,\xd)\right) \xdd
	+ \bm\eta_{\G;\SS}(\x,\xd) \nonumber \\
	& = - \nabla_\x \Phi(\x) - \Bb(\x,\xd)\xd
\end{align}
where
$
\bm\eta_{\G;\SS}(\x,\xd)
\coloneqq  \J(\x)^\t ( \bm\xi_{\Hb}(\y,\yd) +
(\Hb(\y,\dot\y) + \bm\Xi_{\Hb}(\y,\yd) )
\dot\J(\x,\xd) \xd  )
$.
\red{
If we compare GDSs in~\eqref{eq:GDS} and structured GDSs in~\eqref{eq:structured GDS}, the difference is that $\bm\xi_{\G}(\x,\xd)$ is now replaced by a different curvature term $\bm\eta_{\G;\SS}(\x,\xd)$ that is defined by \emph{both} the metric and factorization.
In fact, GDSs are structured GDSs with a \emph{trivial} structure (i.e. $\y =\x$).} Also, one can easily show that structured GDSs reduce to GDSs (i.e. the structure offers no extra information) if $\G(\x,\xd)=\G(\x)$, or if $n,m=1$.
Given two structures, we say $\SS_a$ \emph{preserves} $\SS_b$ if $\SS_a$ has the factorization (of $\Hb$) made by $\SS_b$.
In Section~\ref{sec:geometric properties}, we will show that structured GDSs are related to a geometric object, pullback connection, which turns out to be the coordinate-free version of \pullback.

\red{
Below we show the closure property: when the children of a parent node are structured GDSs, the parent node defined by \pullback is also a structured GDS with respect to the pullbacked structured metric matrix, damping matrix, and potentials.}
Without loss of generality,
we consider again a parent node on $\MM$ with $K$ child nodes on $\{\NN_i\}_{i=1}^K$. \ifLONG\else{ We }\fi note that $\G_i$ and $\B_i$ can be functions of both $\y_i$ and $\yd_i$.
\begin{restatable}{theorem}{theoremConsistency} \label{th:consistency}
	Let the $i$th child node follow $(\NN_i, \G_i, \B_i, \Phi_i)_{\SS_i}$ and have coordinate $\y_i$.
	Let $\fb_i = -\bm\eta_{\G_i;\SS_i} - \nabla_{\y_i} \Phi_i - \B_{i}\yd_i $ and $\M_i =\G_i + \bm\Xi_{\G_i}$.
	If $[\fb,\Mb]^\MM$ of the parent node is given by \emph{\pullback} with $\{[\fb_i, \M_i]^{\NN_i} \}_{i=1}^K$ and $\Mb$ is non-singular, the parent node
	follows the \emph{pullback structured GDS}
	$(\MM, \G, \B, \Phi)_\SS$,
	where $\Gb = \sum_{i=1}^{K}\J_i^\t\G_i\J_i$, $\B = \sum_{i=1}^{K}\J_i^\t \B_i \J_i$, $\Phi =  \sum_{i=1}^{K}\Phi_i \circ \y_i $, $\SS$ preserves $\SS_i$, and $\J_i = \partial_\x \y_i$.
	\red{In other words, the parent node is the RMP $(\ab, \M)^{\MM}$ where $\M = \sum_{i=1}^{K}\J_i^\t (\G_i+\bm\Xi_{\G_i} )\J_i$ and
	\begin{align*}\textstyle
	\ab 
	= \left(\Gb + \bm\Xi_{\G}\right)^\dagger \left(
	- \bm\eta_{\G;\SS}  - \nabla_\x \Phi - \Bb\xd  \right)
	\end{align*}}%
Particularly, if every $\G_i$ is velocity-free and the child nodes are GDSs, the parent node follows $(\MM, \G, \B, \Phi)$.
\end{restatable}
\noindent Theorem~\ref{th:consistency} shows structured GDSs are closed under \pullback.
It means that the differential equation of a structured GDS with a tree-structured task map can be computed by recursively applying \pullback from the leaves to the root, \red{because in each recursive step, the form of structured GDS is preserved by \pullback.
Particularly, when $\G$ is velocity-free, one can show that \pullback also preserves GDSs. We summarize these properties below.}
\begin{corollary}\label{cr:consistency}
	If all leaf nodes follow GDSs and $\Mb_r$ at the root node is nonsingular, then the root node follows $(\CC, \G, \B, \Phi)_{\SS}$ as recursively defined by Theorem~\ref{th:consistency}.
\end{corollary}

\subsection{Stability} \label{sec:stability}

By the closure property above, we analyze the stability of \flow  when the leaf nodes are (structured) GDSs. For compactness, we will abuse the notation to write $\Mb = \Mb_r$. Suppose $\Mb$ is nonsingular and let $(\CC, \G, \B, \Phi)_{\SS}$ be the resultant structured GDS at the root node.
We consider a Lyapunov function candidate
\begin{align} \label{eq:Lypunov candidate}
\textstyle
V(\q, \qd) = \frac{1}{2} \qd^\t \G(\q,\qd) \qd + \Phi(\q)
\end{align}
and derive its rate using properties of structured GDSs.
\begin{restatable}{proposition}{propositionLyapunovTimeDerivative}
	\label{pr:Lyapunov time derivative}
	For $(\CC, \G, \B, \Phi)_{\SS}$, it holds that $\dot\V(\q,\qd) = - \qd^\t \B(\q,\qd) \qd$.
\end{restatable}
\noindent Proposition~\ref{pr:Lyapunov time derivative} directly implies the stability of structured GDSs by invoking LaSalle's invariance principle~\cite{khalil1996noninear}. Here we summarize the result without proof.
\begin{corollary} \label{cr:stability}
	For $(\CC, \G, \B, \Phi)_{\SS}$, if $\G(\q,\qd), \B(\q,\qd) \succ 0 $,  the system converges to a forward invariant set $\CC_\infty \coloneqq \{(\q,\qd) : \nabla_\q \Phi(\q) = 0, \qd = 0 \}$.
\end{corollary}
To show the stability of \flow, we need to further check when the assumptions in Corollary~\ref{cr:stability} hold.
The condition  $\B(\q,\qd) \succ 0 $ is easy to satisfy: by Theorem~\ref{th:consistency},
\red{$\B(\q,\qd)$ has the form $\sum_{i=1}^{K} \J_i(\q)^\t \B_i(\x_i,\xd_i) \J_i(\q)$.
Therefore, it automatically satisfies }
$\B(\q,\qd) \succeq 0$; to strictly ensure definiteness, we can copy $\CC$ into an additional child node with a (small) positive-definite damping matrix. The condition on $\Gb(\q,\qd) \succ 0 $ can be satisfied 
based on a similar argument about $\B(\q,\qd)$.
In addition, we need to verify the assumption that $\M$ is nonsingular. Here we provide a sufficient condition. When satisfied, it implies the global stability of \flow in the sense of Corollary~\ref{cr:stability}.

\begin{restatable}{theorem}{theoremVelocityMetric}
	\label{th:condition on velocity metric}
	Suppose every leaf node is a GDS with a metric matrix in the form
	$\Rb(\x) +  \Lb(\x)^\t \D(\x, \xd) \Lb(\x)$ for differentiable functions $\Rb$, $\Lb$, and $\D$ satisfying
	\ifLONG
\else
	$\Rb(\x) \succeq 0$, $\D(\x,\xd) = \diag
	((d_{i}(\x,\dot{y}_i))_{i=1}^n) \succeq 0$, and $ \dot{y}_i \partial_{\dot{y}_i} d_{i}(\x,\dot{y}_i) \geq 0$,
	\fi
	where $\x$ is the coordinate of the leaf-node manifold and $\yd = \Lb \xd \in \R^n$.
	It holds $\bm\Xi_\G(\q,\qd) \succeq 0$. If further $\G(\q,\qd), \B(\q,\qd) \succ 0$, then $\M\in\R^{d\times d}_{++}$, and the global RMP generated by \flow converges to the forward invariant set $\CC_\infty$ in Corollary~\ref{cr:stability}.
\end{restatable}
\noindent A particular condition in Theorem~\ref{th:condition on velocity metric} is when all the leaf nodes with velocity dependent metric are 1D. Suppose $x\in\R$ is its coordinate and $g(x,\dot{x})$ is its metric matrix. The sufficient condition essentially boils down to $g(x,\dot{x})\geq0$ and $\dot{x} \partial_{\dot{x}} g(x,\dot{x})\geq 0 $. This means that, given any  $x \in \R$, $g(x,0) = 0$, $g(x,\dot{x})$ is non-decreasing when $\dot{x}>0$, and non-increasing when $\dot{x}<0$.
This condition is satisfied by the collision avoidance policy in Section~\ref{sec:example RMPs}.

\subsection{Invariance }  \label{sec:geometric properties}

\newcommand{\lsup}[2]{^{\scriptstyle #2}{#1}}
\newcommand{\pr}{\mathrm{pr}}
\newcommand{\ppartial}[1]{\frac{\partial}{\partial #1}}
\def\conn{\lsup{\nabla}{G}}
\newcommand{\Conn}[1]{\lsup{\nabla}{#1}}
\def\d{\mathrm{d}}

We now discuss the coordinate-free geometric properties of $(\CC, \G, \B, \Phi)_\SS$ generated by  \flow. Due to space constraint, we only summarize the results.
Here we assume that $\G$ is positive-definite. 

\red{
We first introduce some additional notations for the coordinate-free analysis and give definitions of common differential geometric objects (please see, e.g.,~\cite{lee2009manifolds} for an excellent tutorial). For a manifold $\CC$, we use $T\CC$ to denote its tangent bundle (i.e. a manifold that describes the tangent spaces on the base manifold $\CC$) and write $p_{T\CC}: T\CC \to \CC$ to denote the bundle projection, which recovers the corresponding point on $\CC$ (i.e. configuration) from a point on $T\CC$ (a pair of position and the attached tangent vector).
Specifically, suppose $(U, (\q, \vb))$ is a (local) chart on $T\CC$ on a neighborhood $U$. Let $\{\ppartial{q_i}, \ppartial{v_i}\}_{i=1}^d$ and $\{\d{q^i}, \d{v^i}\}_{i=1}^d$ denote the induced frame field and coframe field on $T\CC$ (i.e. the basis vector fields that characterize the tangent spaces and their dual spaces).
For $ s \in U$, we write $s$ in  coordinate  as $(\q(q), \vb(s))$, if $\sum_{i=1}^{d} v_i(s) \frac{\partial}{\partial q_i} |_q \in T_{q} \CC$, where  $q = p_{T\CC}(s)\in \CC$.
With abuse of notation, we also write $s = (\q, \vb)$ for short unless clarity is lost.
Similarly, a chart $(\tilde{U}, (\q, \vb, \ub, \ab))$ can naturally be constructed on the double tangent bundle $TT\CC$, where $\tilde{U} = p_{TT\CC}^{-1}(U)$ and $p_{TT\CC}: TT\CC \to T\CC$ is the bundle projection: we write $h = (\q, \vb, \ub, \ab) \in TT\CC$ if  $\sum_{i=1}^{d} u_i(h) \frac{\partial}{\partial q_i} |_s +  a_i(h) \frac{\partial}{\partial v_i} |_s \in T_{s} T \CC$, where $s = p_{TT\CC}(h) $.
Under these notations, for a curve $q(t)$ on $\CC$, we can write $\ddot{q}(t) \in TT\CC$ in coordinate as $(\q(t), \qd(t), \qd(t),  \qdd(t))$.
Finally, we define a geometric object called \emph{affine connection}, which defines how tangent spaces at different points on a manifold are related. Given Christoffel symbols $\Gamma_{i,j}^k$, an affine connection $\nabla$ on $TT\CC$ is defined via
$
\textstyle\nabla_{ \ppartial{s_i} } \ppartial{s_j}
= \sum_{k=1}^{2d} \Gamma_{i,j}^k  \ppartial{s_k}
$, where $\ppartial{s_i} \coloneqq \ppartial{q_i}$ and  $ \ppartial{s_{i+d}} \coloneqq \ppartial{v_i}$ for $i=1,\dots,d$.

Using this new notation, we show that GDSs can be written in a coordinate-free manner in terms of affine connection.}
Let $T\CC$ denote the tangent bundle of $\CC$, which is a natural manifold to describe the state space. 
Precisely, we prove that a GDS on $\CC$ can be expressed in terms of a unique, asymmetric affine connection $\conn$ that is compatible with a Riemannian metric $G$ (defined by $\G$) on $T\CC$. It is important to note that $G$ is defined on $T\CC$ \emph{not} the original manifold $\CC$. As the metric matrix in a GDS can be velocity dependent, we need a larger manifold.

\begin{restatable}{theorem}{theoremGeometricAcceleration} \label{th:geometric acceleration}
	Let $G$ be a Riemannian metric on $T\CC$ such that, for $s = (q,v) \in T\CC$,  $G(s) = \sum_{i,j} G^v_{ij}(s) dq^i \otimes  d q^j +  G^a_{ij} dv^i \otimes  dv^j$, where $G^v_{ij}(s)$ and $G^a_{ij}$ are symmetric and positive-definite, and $G^v_{ij}(\cdot)$ is differentiable.
	Then there is a unique affine connection $\conn$ that is compatible with $G$ and satisfies,
	$\Gamma_{i,j}^k = \Gamma_{ji}^k$,
	$\Gamma_{i,j+d}^k = 0$,
	and $\Gamma_{i+d,j+d}^k = \Gamma_{j+d,i+d}^k$, for $i,j = 1,\dots,d$ and $k = 1,\dots, 2d$.
	In coordinates, if $G_{ij}^v(\dot{q})$ is identified as $\G(\q,\qd)$, then
	$\pr_{3} (\conn_{\ddot{q}} \ddot{q})$ can be written as $\ab_\G \coloneqq \qdd +  \Gb(\q,\qd)^{-1} (\bm\xi_{\G}(\q,\qd) + \bm\Xi_{\G}(\q,\qd) \qdd )$, where $\pr_{3}: (\q,\vb,\ub,\ab) \mapsto \ub$ is a projection.
\end{restatable}
\noindent We call $\pr_{3} (\lsup{\nabla}{G}_{\dot{q}} \dot{q})$ the \emph{geometric acceleration} of $q(t)$ with respect to $\conn$. It is a coordinate-free object, because $\pr_{3}$ is defined independent of the choice of chart on $\CC$. By Theorem~\ref{th:geometric acceleration}, it is clear that a GDS can be written abstractly as
\begin{align} \label{eq:GDS abstract}
 \pr_3(\lsup{\nabla}{G}_{\ddot{q}} \ddot{q}) = (\pr_3 \circ G^{\sharp} \circ F) (s)
\end{align}
where $F: s \mapsto -d\Phi(s) - B(s) $ defines the covectors due to the potential function and damping, and $G^{\sharp} : T^*T\CC \to TT\CC$  denotes the inverse of $G$. In coordinates, 
it reads as
$
\qdd +  \Gb(\q,\qd)^{-1} (\bm\xi_{\G}(\q,\qd) + \bm\Xi_{\G}(\q,\qd) \qdd )  =  -\Gb(\q,\qd)^{-1} (\nabla_\q \Phi(\q) + \Bb(\q,\qd)\qd )
$, which is exactly~\eqref{eq:GDS}.

Extending this result, we present a coordinate-free representation of \flow when the leaf-nodes are GDSs.

\begin{restatable}{theorem}{theoremAbstractConsistency}\label{th:consistency abstract}
	Suppose $\CC$ is related to $K$ leaf-node task spaces by maps $\{\psi_i: \CC \to \TT_i\}_{i=1}^K$ and  the $i$th task space $\TT_i$ has an affine connection $\Conn{G_i}$ on $T \TT_i$, as defined in Theorem~\ref{th:geometric acceleration}, and a covector function $F_i$ defined by some potential and damping as described above. 
	Let $\lsup{\bar{\nabla}}{G} = \sum_{i=1}^{K} T\psi_i^* {\Conn{G_i}}$ be the pullback connection, $G = \sum_{i=1}^K T\psi_i^* G_i$ be the pullback metric, and $F = \sum_{i=1}^{K} T\psi_i^* F_i$ be the pullback covector, where $T\psi_i^*: T^*T\TT_i \to T^*T\CC$.
	Then $\lsup{\bar{\nabla}}{G}$ is compatible with $G$, and
	$\pr_{3} (\lsup{\bar{\nabla}}{G} _{\ddot{q}} \ddot{q})=  (\pr_3 \circ G^{\sharp} \circ F) (s) $ can be written as $ \qdd +  \Gb(\q,\qd)^{-1} (\bm\eta_{\G;\SS}(\q,\qd) + \bm\Xi_{\G}(\q,\qd) \qdd ) =  -\Gb(\q,\qd)^{-1} (\nabla_\q \Phi(\q) + \Bb(\q,\qd)\qd )$.
	In particular, if $G$ is velocity-independent, then $\lsup{\bar{\nabla}}{G} = \conn$.
\end{restatable}
\noindent Theorem~\ref{th:consistency abstract} says that the structured GDS $(\CC, \G, \B, \Phi)_\SS$ can be written abstractly, without coordinates, using the pullback of task-space covectors, metrics, and asymmetric affine connections (that are defined in Theorem~\ref{th:geometric acceleration}).
In other words, the recursive calls of \pullback in the backward pass of \flow is indeed performing ``pullback'' of geometric objects.
\red{We can think that the leaf nodes define the asymmetric affine connections, and \flow performs \pullback to pullback those connections onto $\CC$ to define $\lsup{\bar{\nabla}}{G}$.}
Theorem~\ref{th:consistency abstract} also shows, when $G$ is velocity-independent, the pullback of connection and the pullback of metric commutes.
In this case, $\lsup{\bar{\nabla}}{G} = \conn$, which is equivalent to the classic Levi-Civita connection of $G$. The loss of commutativity in general is due to the asymmetric definition of the connection in Theorem~\ref{th:geometric acceleration}, which however is necessary to derive a control law of acceleration, without further referring to  higher-order time derivatives.

\section{Operational Space Control and Geometric Mechanics in View of \flow}

With the algorithm details and theoretical properties of \flow introduced, we now discuss more precisely how \flow is connected to and generalizes existing work in geometric mechanics and operational space control.

\subsection{From Operational Space Control to \flow with GDSs} \label{app:GDSs}

Our study of GDSs (introduced in Section~\ref{sec:GDS})
is motivated by SMSs in geometric mechanics which describe the dynamics used in existing operational space control schemes (cf. \cref{sec:tutorial}).
Many formulations of mechanics exist, including Lagrangian mechanics~\cite{ClassicalMechanicsTaylor05} and the aforementioned Gauss's principle of least constraint~\cite{udwadia1996analytical}, and they are all equivalent, implicitly sharing the same mathematical structure.
But among them, we find that geometric mechanics, which models
physical systems as geodesic flow on Riemannian manifolds, is the most explicit one: it summarizes the system properties arising from the underlying manifold structure compactly, as Riemannian metrics, and connects to the broad mathematical tool set from Riemannian geometry.

These geometry-based insights provide us a way to generalize beyond the previous SMSs studied in~\cite{bullo2004geometric} and then design GDSs, a family non-classical dynamical systems that, through the use of configuration-and-velocity dependent metrics, more naturally describe behaviors of robots desired for tasks in non-Euclidean spaces.

The proposed generalization preserves several nice features from SMSs to GDSs. As in SMSs, the properties of GDSs are captured by the metric matrix.  For example, a GDS like a SMS possesses the natural conservation property of kinematic energy, i.e. it travels along a geodesic defined by $\G(\x,\xd)$ when there is no external perturbations due to $\Phi$ and $\Bb$. Note that $\G(\x,\xd)$ by definition may only be positive-semidefinite even when the system is non-degenerate; here we allow the geodesic to be defined for a degenerate metric, meaning a curve whose instant length measured by the (degenerate) metric is constant.
This geometric feature is an important tool to establish the stability of GDSs in our analysis; 
We highlight this nice property below, which is a corollary of Proposition~\ref{pr:Lyapunov time derivative}. Note that this property also hold for degenerate GDSs provided that differential equations satisfying~\eqref{eq:GDS general} exist.
\begin{corollary}
	All 
	 GDSs in the form $(\MM, \Gb, 0, 0)$ travel on geodesics defined by $\Gb$. That is, $\dot{K}(\x,\xd) = 0$, where $K(\x,\xd) = \frac{1}{2} \xd^\t \G(\x,\xd) \xd$.
\end{corollary}

As we discussed earlier, these generalized metrics induce curvature terms $\bm\Xi_\Gb$ and  $\bm{\xi}_\Gb$ that can be useful to design sensible motions for tasks in non-Euclidean spaces (cf. \cref{sec:example RMPs}).
As we showed GDSs are coordinate-free, these terms and behaviors arise naturally when traveling on geodesics that is defined by configuration-and-velocity dependent metrics.
To gain more intuition about these curvature terms, we recall that the curvature term $\bm{\xi}_\Gb$ in GDSs is related to the Coriolis force in the SMSs. This is not surprising, as from the analysis in Section~\ref{sec:geometric properties} we know that $\bm{\xi}_\G$ comes from the Christoffel symbols of the asymmetric connection in \cref{th:geometric acceleration}, just as the Coriolis force comes from the Christoffel symbols of Levi-Civita connection.  Recall it is defined as
\begin{align*}
	\bm\xi_{\G}(\x,\xd) &\coloneqq \textstyle \sdot{\Gb}{\x}(\x,\xd) \xd - \frac{1}{2} \nabla_\x (\xd^\t \Gb(\x,\xd) \xd)
\end{align*}
Now we show their relationship explicitly below.
\begin{restatable}{lemma}{coriolosIdentity}\label{lm:Coriolos identity}
	\label{lm:compact writing of Coriolis force}
	Let $\Gamma_{ijk} = \frac{1}{2}( \partial_{x_k}  G_{ij}  + \partial_{x_j}  G_{ik} - \partial_{x_j}  G_{jk})$ be the Christoffel symbol of the first kind with respect to $\Gb(\x,\xd)$, where the subscript $_{ij}$ denotes the $(i,j)$ element. Let
	$	C_{ij}
	= 
	\sum_{k=1}^{d} \dot{x}_k \Gamma_{ijk}
	$ and define $\Cb(\x,\xd) = (C_{ij})_{i,j=1}^m$. Then $\bm\xi_{\Gb}(\x,\xd)  = \Cb(\x,\xd) \xd$.
\end{restatable}
\begin{proof}[Proof of Lemma~\ref{lm:Coriolos identity}]
	Suppose $\bm\xi_\Gb = (\xi_i)_{i=1}^m$. We can compare the two definitions and verify they are indeed equivalent:
	\begin{align*}
		\xi_i
		&= \textstyle
		\sum_{j,k=1}^{d} \dot{x}_j \dot{x}_k \partial_{x_j}  G_{ik} - \frac{1}{2}\sum_{j,k=1}^{d} \dot{x}_j\dot{x}_k \partial_{x_i}  G_{jk}\\
		&= \textstyle
		\frac{1}{2} \sum_{j,k=1}^{d} \dot{x}_j \dot{x}_k \partial_{x_k}  G_{ij}
		+\frac{1}{2}\sum_{j,k=1}^{d} \dot{x}_j \dot{x}_k \partial_{x_j}  G_{ik} \\
		&\quad \textstyle
		- \frac{1}{2}\sum_{j,k=1}^{d} \dot{x}_j\dot{x}_k \partial_{x_i}  G_{jk}
		= \left( \Cb(\x,\xd) \xd \right)_i \qedhere
	\end{align*}
\end{proof}
\noindent Thus, we can think intuitively that GDSs modify the inertia and the Coriolis forces in SMSs so that the dynamical system can preserve a generalized notion of kinematic energy that is no-longer necessarily quadratic in velocity.

Finally, we note that the benefits of using configuration-and-velocity dependent metrics can also be understood from their connection to the weight matrices in least-squared problems. Recall from \cref{sec:tutorial} that 
for SMSs, the inertia matrix (which is the same as the metric matrix according to geometric mechanics) forms the importance weight in the least-squared problem. In other words, we can view common operational space control schemes as implicitly combining policies with constant or configuration dependent importance weight matrix in the least-squared sense, which implies certain restriction on the richness of behaviors that it can generate.
By contrast, \flow allows generally importance weight matrices to depend also on velocity in the least-square problems prescribed by \eqref{eq:natural pullback}, which combines RMPs from the child nodes as a RMP at the parent node in every level of the \tree (cf. \cref{sec:RMP algebra}). When the polices come from (structured) GDSs, these weight matrices now again are inertia matrices and the geometric properties of GDSs lead to similar stability and convergence properties as their SMS predecessors.
Thus, in a sense, we can view \flow as generalizing operational space control to consider also configuration-and-velocity dependent weights in policy generation, allowing more flexible trade-offs between different policies.

\subsection{Relationship between \flow and Recursive Newton-Euler Algorithms}\label{app:relationship with dynamics}

For readers familiar with robot dynamics, we remark that the forward-backward policy generation procedure of \flow is closely related to the algorithms~\cite{walker1982efficient} for computing forward dynamics (i.e. computing accelerations given forces) based on recursive Newton-Euler algorithm.
Here we discuss their relationship.

In a summary, these classic algorithms compute the forward dynamics using following steps:
	\begin{enumerate}
		\item It propagates positions and velocities from the base to the end-effector.
		\item It computes the Coriollis force by backward propagating the inverse dynamics of each link under the condition that the acceleration is zero.
		\item It computes the (full/upper-triangular/lower-triangular) joint inertia matrix.
		\item It solves a linear system of equations to obtain the joint acceleration.
	\end{enumerate}
	In~\cite{walker1982efficient}, they assume a recursive Newton-Euler algorithm (RNE) for inverse dynamics is given, and realize Step 1 and Step 2 above by calling the RNE subroutine. The computation of Step 3 depends on which part of the inertia matrix is computed. In particular, their Method 3 (also called the Composite-Rigid-Body Algorithm in~\cite[Chapter 6]{Featherstone08}) computes the upper triangle part of the inertia matrix by a backward propagation from the end-effector to the base.

	\flow can also be used to compute forward dynamics, when we set the leaf-node policy as the constant inertia system on the body frame of each link and we set the transformation in the \tree as the change of coordinates across of robot links.
	This works because we showed that when leaf-node policies are GDSs (which cover SMSs of rigid-body dynamics as a special case), the effective dynamics at the root node is the pullback GDS, which in this case is the effective robot dynamics defined by the inertia matrix of each link.

	We can use this special case to compare \flow with the above procedure.
	We see that the forward pass of \flow is equivalent to Step 1, and the backward pass of \flow is equivalent of Step 2 and Step 3, and the final \resolve operation is equivalent to Step 4.

	Despite similarity, the main difference is that \flow computes the force and the inertia matrix in a \emph{single} backward pass to exploit shared computations. This change is important, especially, the number of subtasks are large, e.g., in avoiding multiples obstacles.
	In addition, the design of \flow generalizes these classical computational procedures (e.g. designed only for rigid bodies, rotational/prismatic joints) to handle abstract and even non-Euclidean task spaces that have velocity-dependent metrics/inertias.
	This extension provides a unified framework of different algorithms and results in an expressive class of motion policies.

Finally, we note that the above idea can be slightly modified so that we can also use \flow to compute the inverse dynamics. This can be done similarly to the above construction using physical inertia to initialize leaf-node RMPs; but at the end, after the backward pass, we solve for instead the torque as $\bm\tau = \Mb_r \qdd_d - \Mb_r \fb_r$ where $\qdd_d$ is the desired joint-space acceleration.

\subsection{Related Approaches to Motion Policy Generation} \label{sec:related work}

While here we focus on the special case of \flow with GDSs, this family already covers a wide range of reactive policies commonly used in practice.
For example, when the task metric is Euclidean (i.e. constant), \flow recovers operational space control (and its variants)~\cite{khatib1987unified,sentis2006whole,Peters_AR_2008,UdwadiaGaussPrincipleControl2003,lo2016virtual}. When the task metric is only configuration dependent, \flow can be viewed as performing energy shaping to combine multiple SMSs in geometric control~\cite{bullo2004geometric}.
Further, \flow allows using velocity dependent metrics, generating behaviors all those previous rigid mechanics-based approaches fail to model.
We also note that \flow can be easily modified to incorporate exogenous time-varying inputs (e.g. forces to realize impedance control~\cite{albu2002cartesian} or learned perturbations as in DMPs~\cite{IjspeertDMPs2013}).
In computation, the structure of \flow in natural-formed RMPs resembles the classical Recursive Newton-Euler algorithm~\cite{walker1982efficient,Featherstone08} (as we just discussed above).
Alternatively, the canonical form of \flow in~\eqref{eq:least-square problem of pullback} resembles Gauss's Principle~\cite{Peters_AR_2008,UdwadiaGaussPrincipleControl2003}, but with a curvature correction $\bm\Xi_\G$ on the inertia matrix (suggested by Theorem~\ref{th:consistency}) to account for velocity dependent metrics.
Thus, we can view \flow as a natural generalization of these approaches to a broader class of non-Euclidean behaviors.

\section{Experiments} \label{sec:experiments}

We first perform controlled experiments to study the curvature effects of nonlinear
metrics, which is important for stability and collision avoidance.  Then we
conduct several full-body experiments
(video:
{\small \url{https://youtu.be/Fl4WvsXQDzo}})
to demonstrate the capabilities of \flow
on high-DOF manipulation problems in clutter, and implement an integrated
vision-and-motion system on two physical robots.
\blueold{Extra details of the RMPs used in this section can found in the Appendix
D of the technical report~\cite{cheng2018rmpflowarxiv}.}

\begin{figure*}[t]
	\centering
	\subfloat[\label{fig:1d_z}]{
		\includegraphics[trim={5 5 0 25},clip, width=0.5\columnwidth,keepaspectratio]{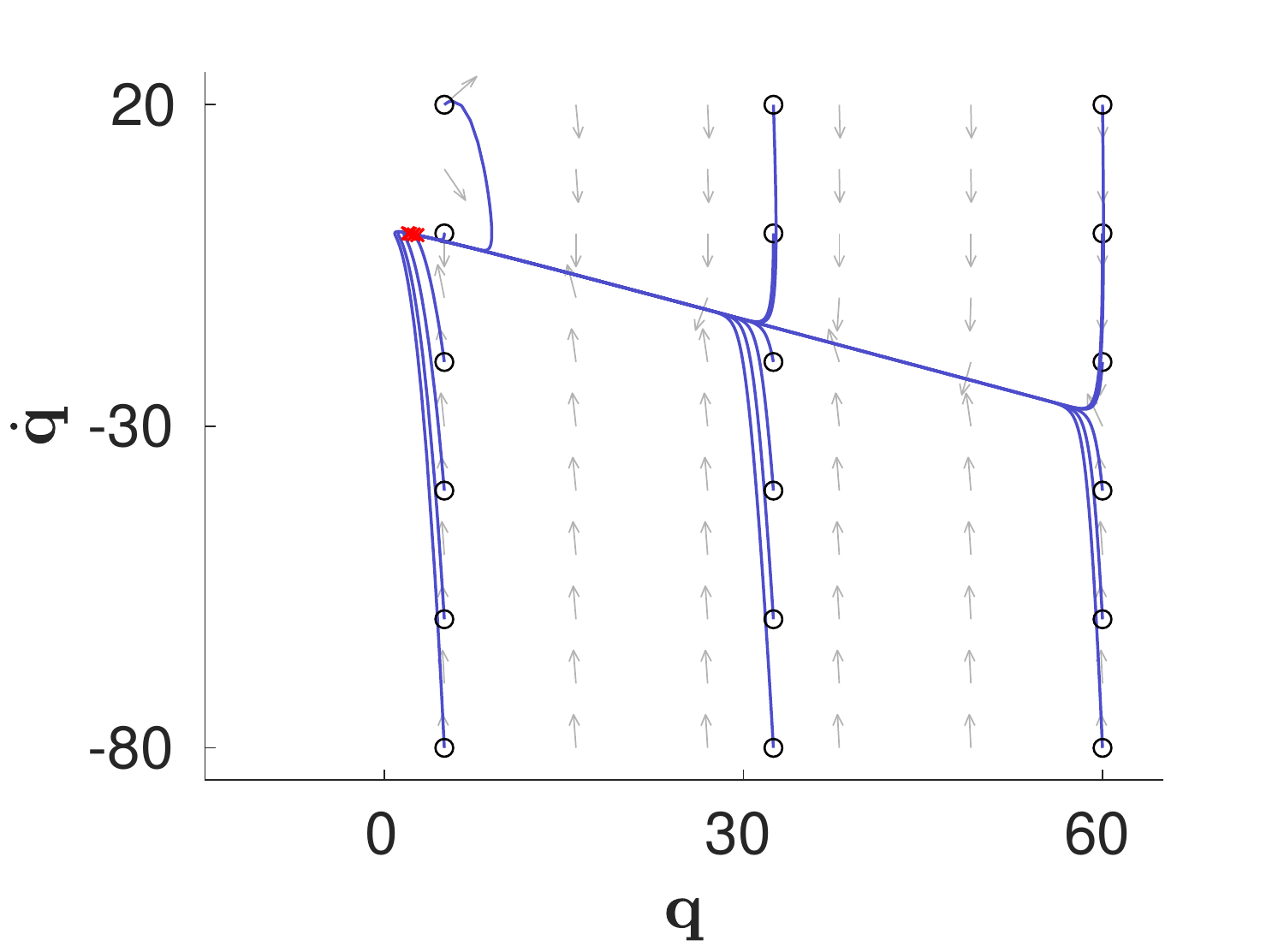}
	}\hspace{-4mm}
	\subfloat[\label{fig:1d_x}]{
		\includegraphics[trim={5 5 0 25},clip, width=0.5\columnwidth,keepaspectratio]{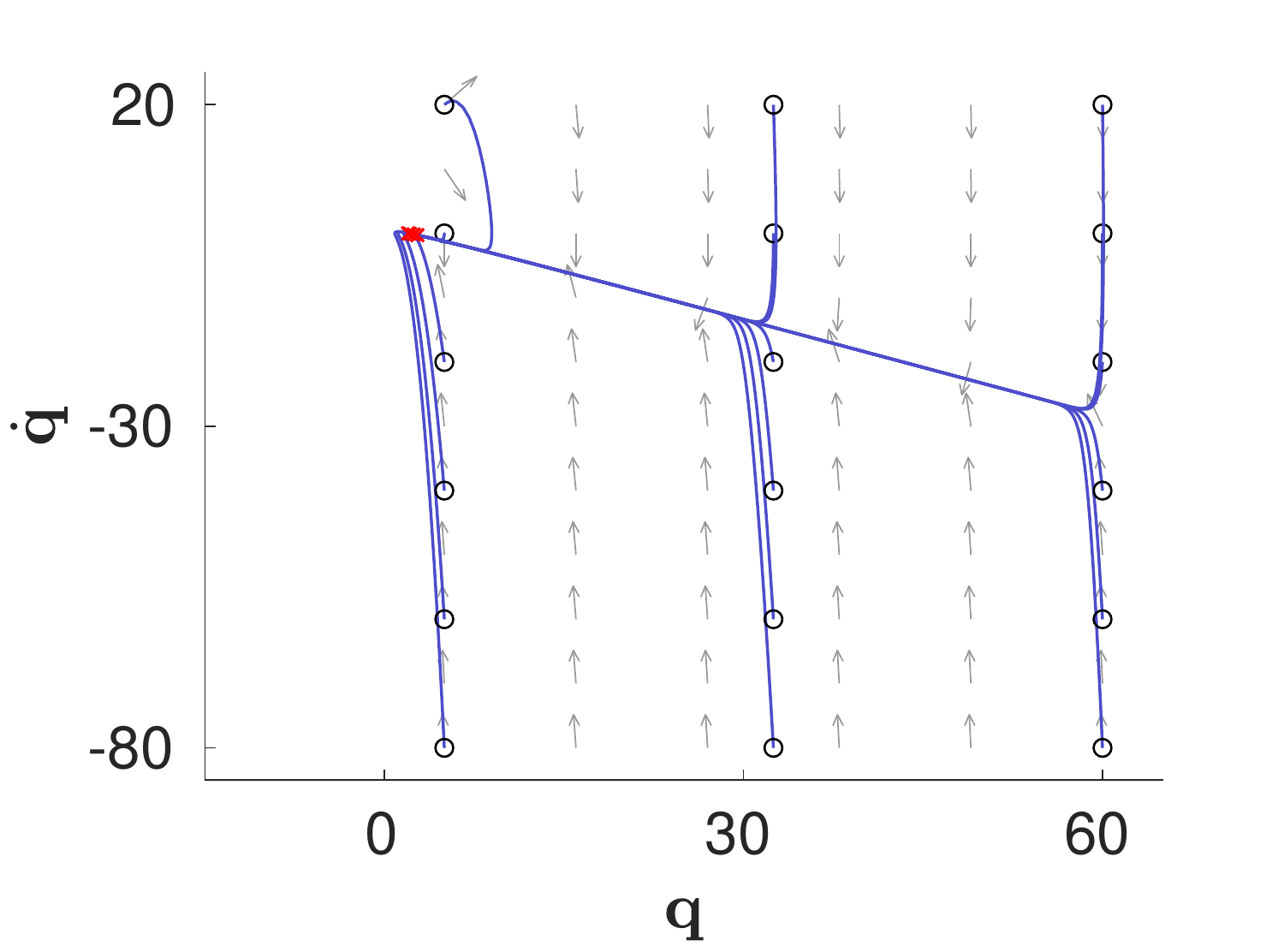}
	}\hspace{-4mm}
	\subfloat[\label{fig:1d_alpha1}]{
		\includegraphics[trim={5 5 0 25},clip, width=0.5\columnwidth,keepaspectratio]{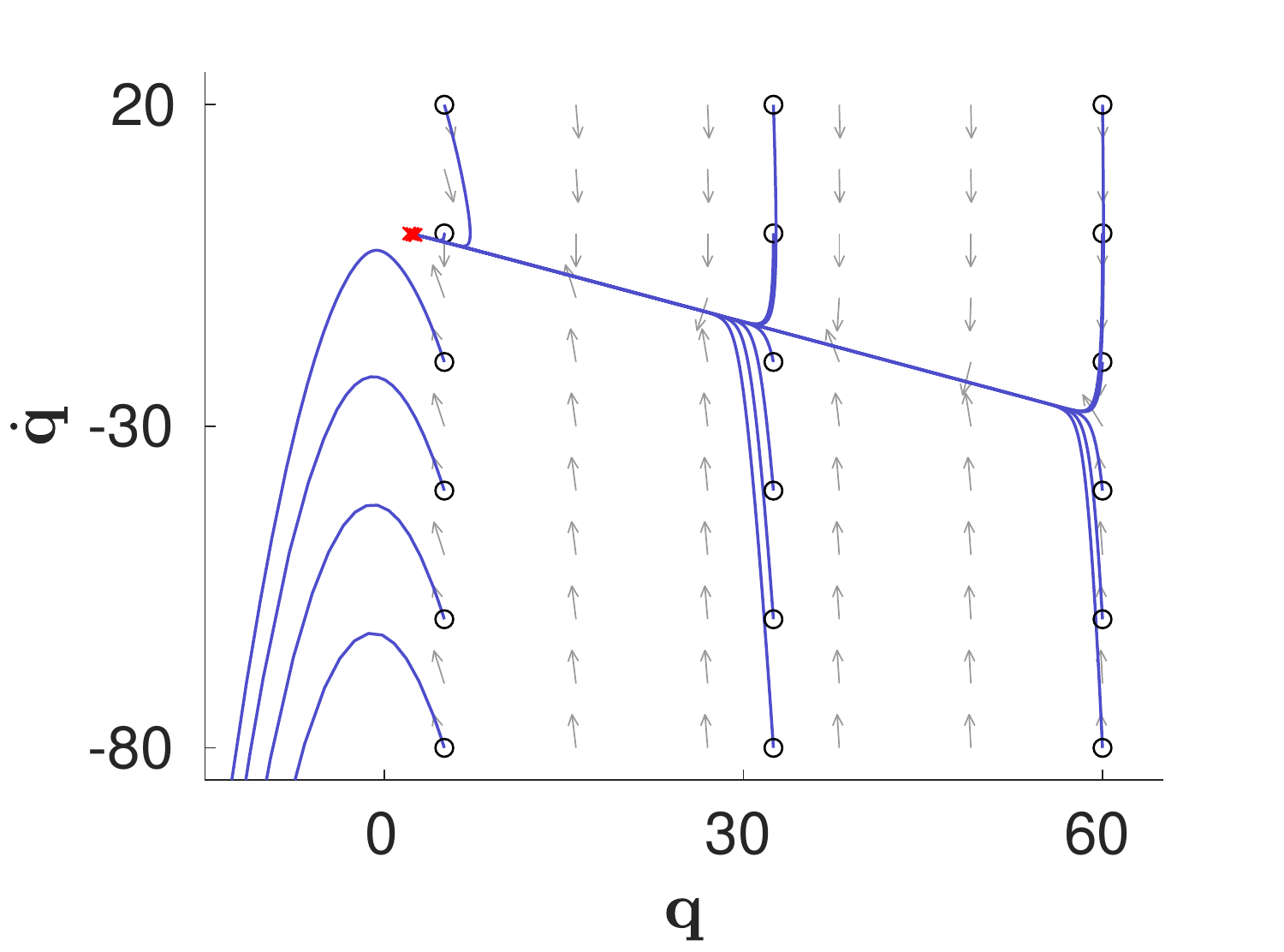}
	}\hspace{-4mm}
	\subfloat[\label{fig:1d_alpha1_damp}]{
		\includegraphics[trim={5 5 0 25},clip, width=0.5\columnwidth,keepaspectratio]{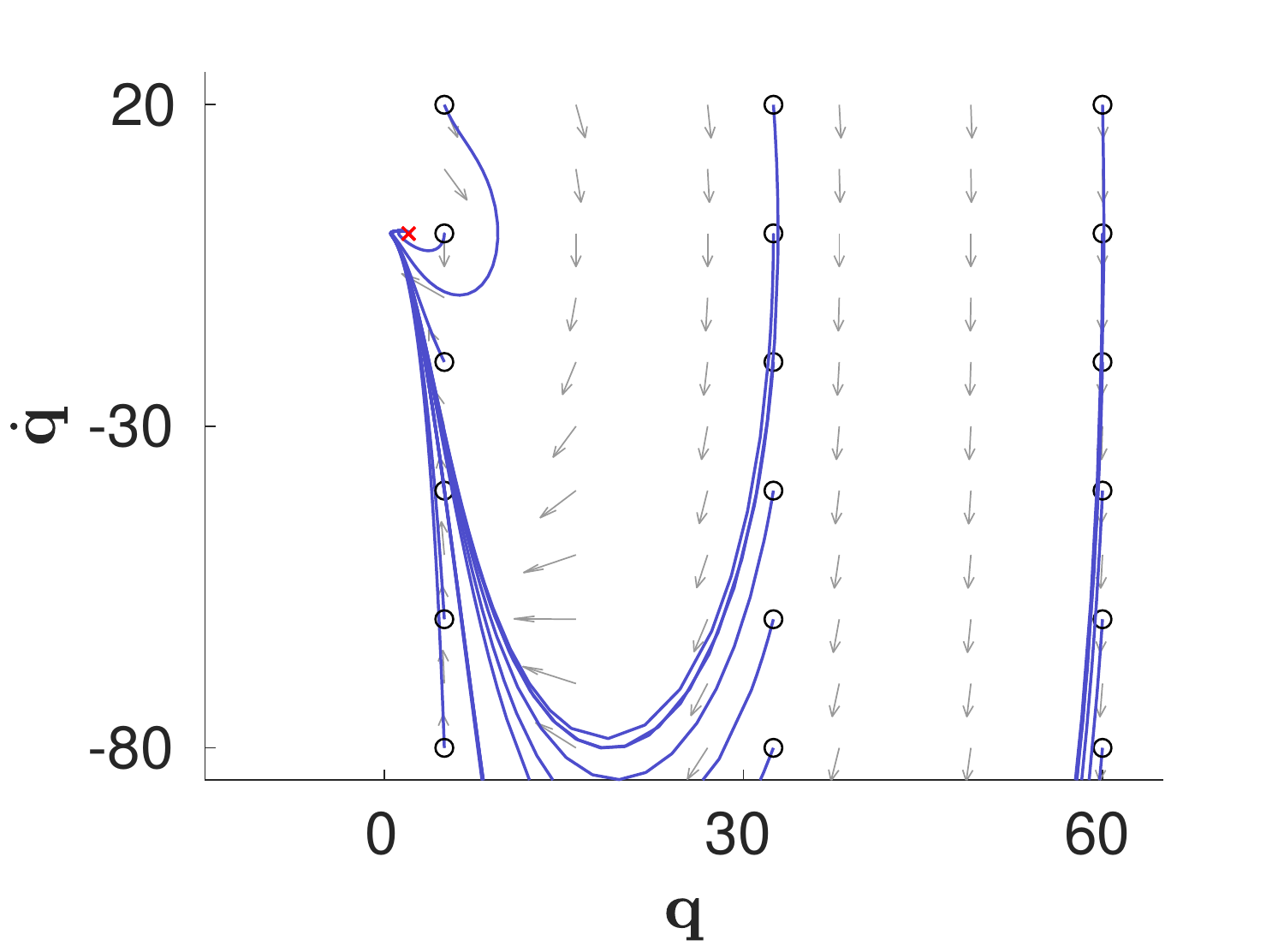}
	}

	\caption{\small Phase portraits (gray) and integral curves (blue; from black circles to red crosses) of 1D example. (a) Desired behavior. 
		(b) With curvature terms. (c) Without curvature terms. (d) Without curvature terms but with nonlinear damping.}
	\label{fig:1d}
\end{figure*}

\begin{figure*}[h]
	\centering
	\hspace{-4mm}
	\subfloat[\label{fig:2d_obs_nocorr}]{
		\includegraphics[trim={40 5 130 15},clip,height=0.37\columnwidth,keepaspectratio]{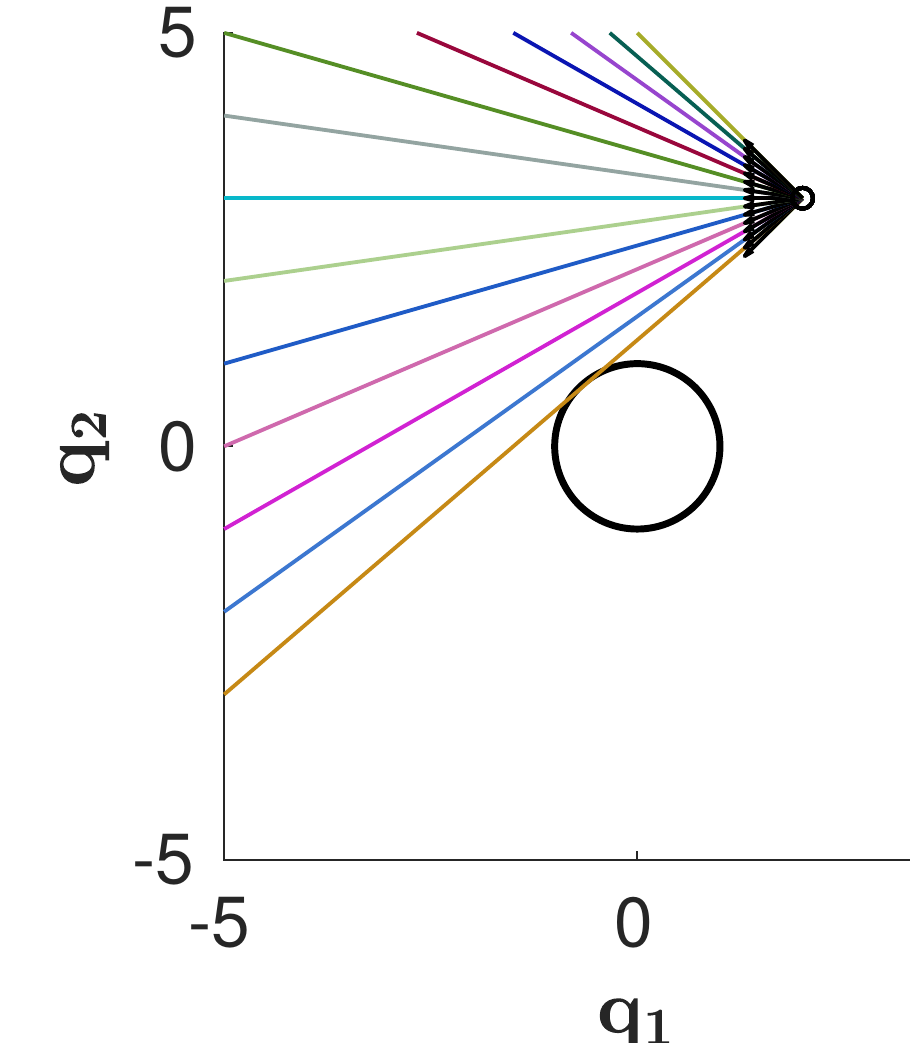}
	}\hspace{-4mm}
	\subfloat[\label{fig:2d_obs}]{
		\includegraphics[trim={40 5 130 15},clip,height=0.37\columnwidth,keepaspectratio]{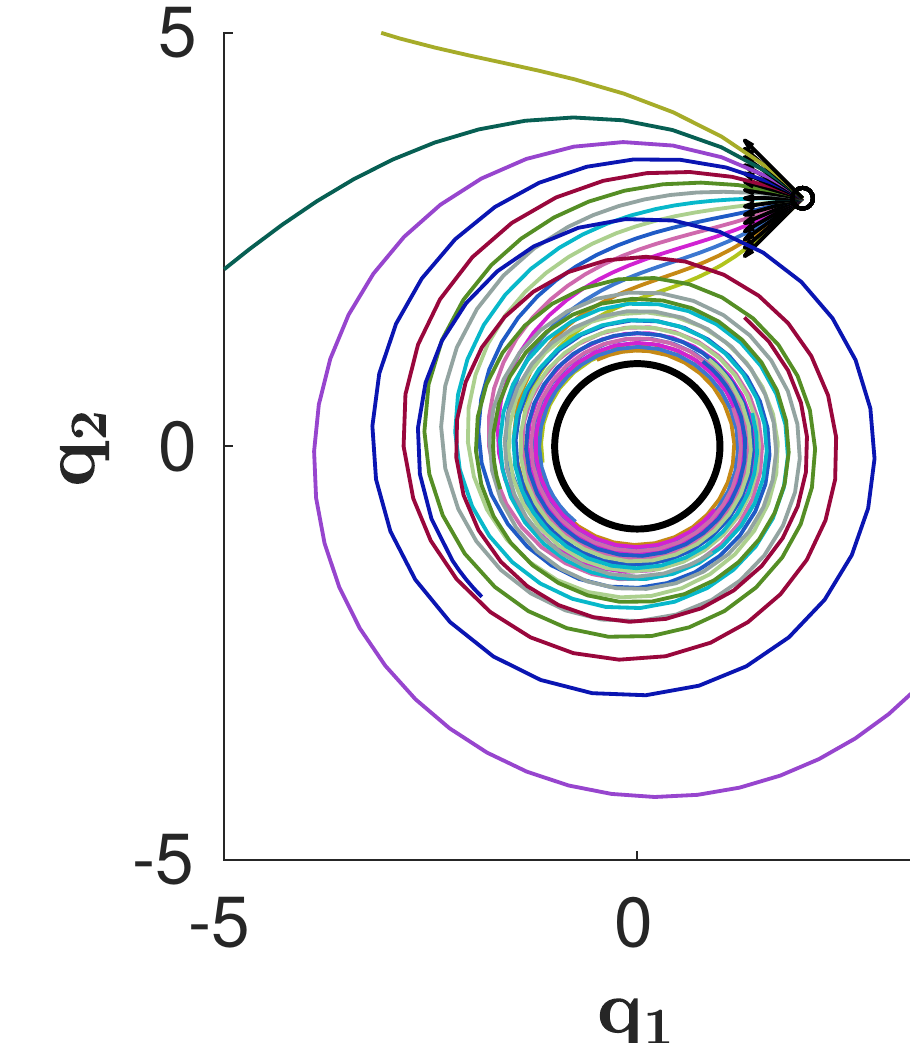}
	}\hspace{-3mm}
	\subfloat[\label{fig:2d_obs_pot_nocorr}]{
		\includegraphics[trim={40 5 130 15},clip,height=0.37\columnwidth,keepaspectratio]{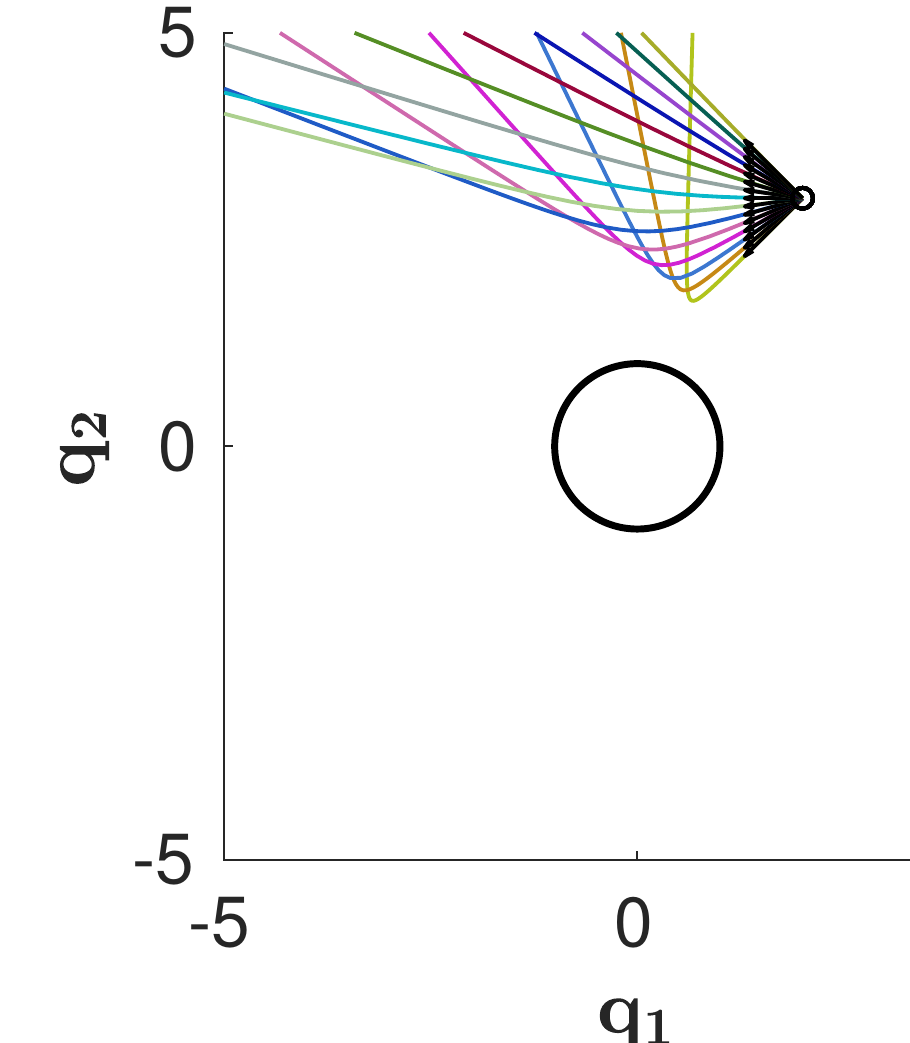}
	}\hspace{-4mm}
	\subfloat[\label{fig:2d_obs_pot}]{
		\includegraphics[trim={40 5 130 15},clip,height=0.37\columnwidth,keepaspectratio]{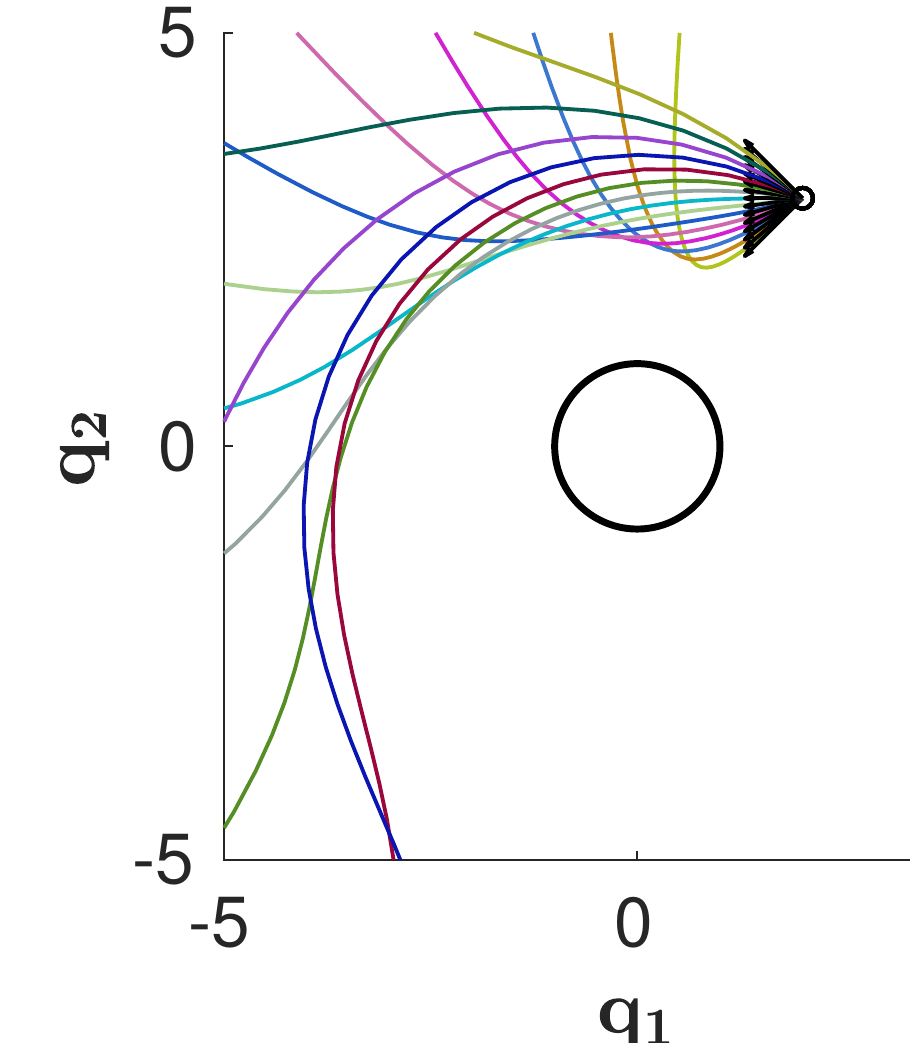}
	}\hspace{-4mm}
	\subfloat[\label{fig:2d_full}]{
		\includegraphics[trim={40 5 130 15},clip,height=0.37\columnwidth,keepaspectratio]{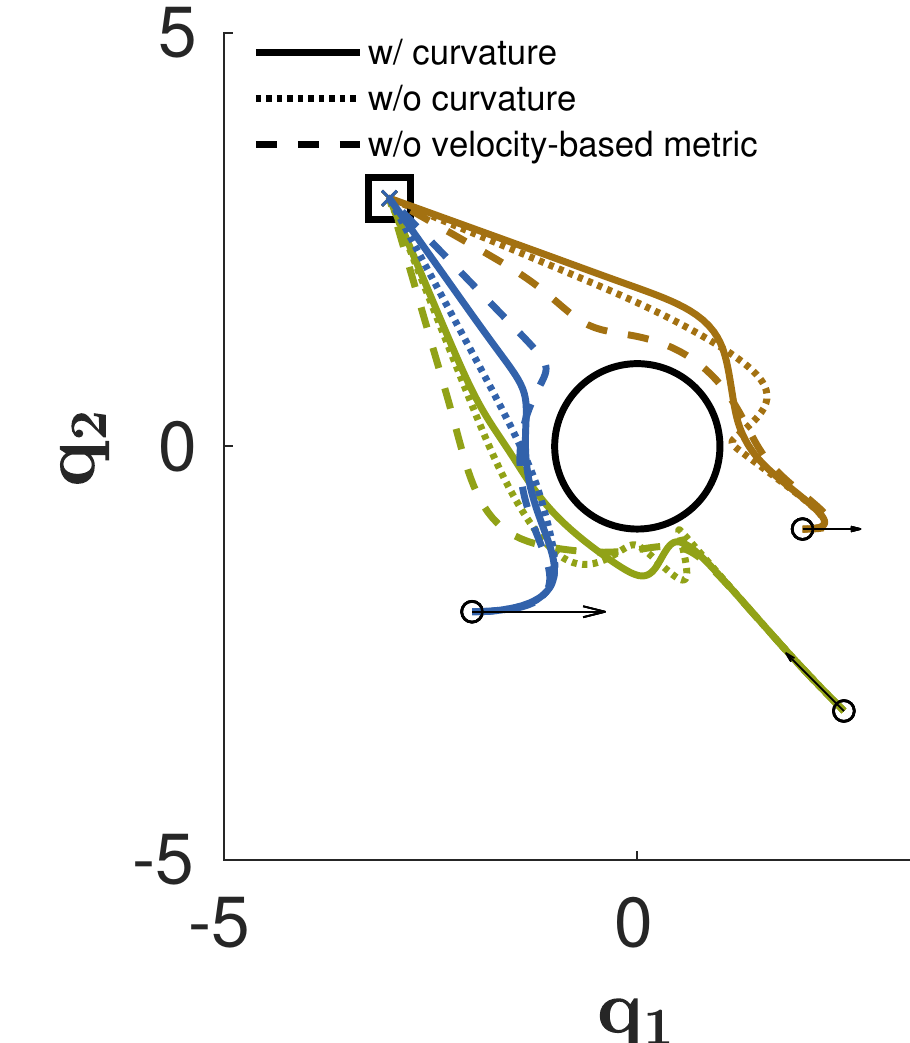}
	}\hspace{-4mm}
	\subfloat[\label{fig:2d_energy}]{
		\includegraphics[trim={0 2 130 15},clip,height=0.37\columnwidth,keepaspectratio]{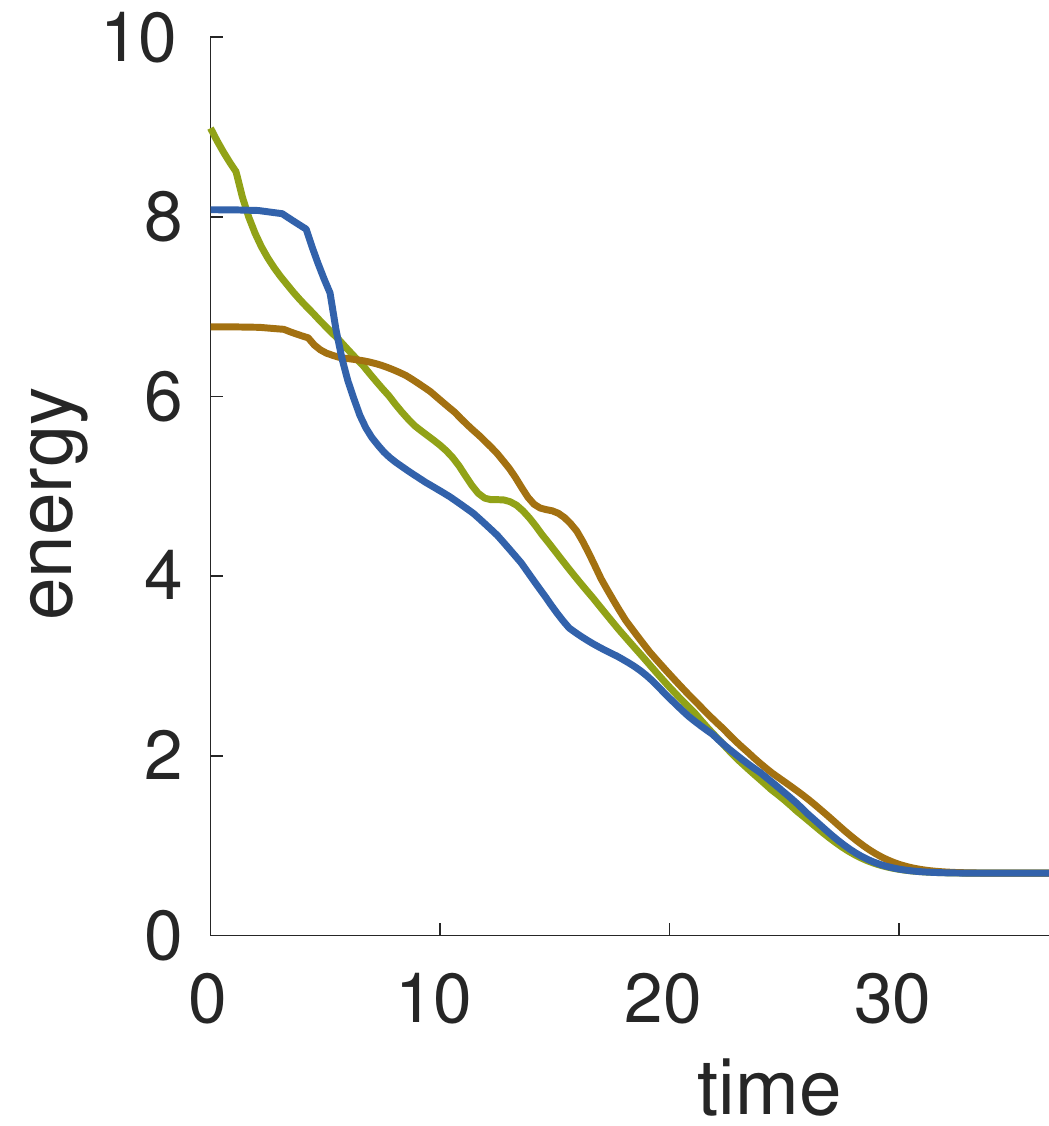}
	}
	\caption{\small
		2D example; initial positions (small circle) and velocities (arrows). (a-d) Obstacle (circle) avoidance: (a) w/o curvature terms and w/o potential. (b) w/ curvature terms and w/o potential. (c) w/o curvature terms and w/ potential. (d) w/ curvature terms and w/ potential. (e) Combined obstacle avoidance and goal (square) reaching. (f) The change of Lyapunov function in \eqref{eq:Lypunov candidate} over time along the trajectories in (e). }
	\label{fig:2DOrbits}
\end{figure*}

\subsection{Controlled Experiments}
\label{sec:1DExample}
\subsubsection{1D Example}
Let $\q \in \R$. We consider a barrier-type task map $\x = 1/\q$ and define a GDS in~\eqref{eq:GDS} with $\G = 1$, $\Phi(\x) = \frac{1}{2}(\x - \x_0)^2$, and $\B = (1 + 1/\x)$, where $\x_0 > 0$. Using the GDS, we can define an RMP $[- \nabla_\x \Phi - \Bb\xd - \bm\xi_{\G}, \M]^\R$, where
$\M$ and $\bm\xi_{\G}$ are defined according to Section~\ref{sec:GDS}.
We use this example to study the effects of $\dot{\J}\qd$ in \pullback~\eqref{eq:natural pullback}, where we define $\J = \partial_\q \x$. Fig.~\ref{fig:1d} compares the desired behavior (Fig.~\ref{fig:1d_z}) and the behaviors of correct/incorrect \pullback. If \pullback is performed correctly with $\Jd \qd$, the behavior matches the designed one (Fig.~\ref{fig:1d_x}). By contrast, if $\Jd \qd$ is ignored, the observed behavior becomes inconsistent and unstable (Fig.~\ref{fig:1d_alpha1}).
While the instability of neglecting $\dot{\J}\qd$ can be recovered with a damping $\B = (1 + \frac{\xd^2}{\x})$ nonlinear in $\xd$ (suggested in~\cite{lo2016virtual}), the behavior remains inconsistent (Fig.~\ref{fig:1d_alpha1_damp}).

\subsubsection{2D Example}
We consider a 2D goal-reaching task with collision avoidance and
study the effects of velocity dependent metrics.
First, we define an RMP (a GDS as in Section~\ref{sec:example RMPs}) in $\x = d(\q)$ (the 1D task space of the distance to the obstacle). We pick a metric $\G(\x,\xd) = w(\x) u(\xd)$, where $w(\x) = 1/\x^4$ increases if the particle is \emph{close} to the obstacle 
and $u(\xd) = \epsilon + \min(0, \xd) \xd$ (where $\epsilon \geq 0$), increases if it moves \emph{towards} the obstacle.
As this metric is non-constant, the GDS has curvature terms $\bm\Xi_{\G} = \frac{1}{2}\xd w(\x) \partial_\xd u(\xd)$ and $\bm\xi_{\G} = \frac{1}{2}\xd^2 u(\xd) \partial_\x w(\x)$.  
These curvature terms along with $\Jd \qd$ produce an acceleration that lead to
natural obstacle avoidance behavior, coaxing the system toward isocontours of the obstacle (Fig.~\ref{fig:2d_obs}).
On the other hand, when the curvature terms are ignored, 
the particle travels in straight lines with constant velocity (Fig.~\ref{fig:2d_obs_nocorr}).
To define the full collision avoidance RMP, we introduce a barrier-type potential $\Phi(\x) = \frac{1}{2}\alpha w(\x)^2$ to create extra repulsive forces, where $\alpha\geq 0$. A comparison of the curvature effects in this setting is shown in Fig.~\ref{fig:2d_obs_pot_nocorr} and~\ref{fig:2d_obs_pot} (with $\alpha = 1$).
Next, we use \flow to combine the collision avoidance RMP above (with $\alpha=0.001$) and an attractor RMP.
Let $\q_g$ be the goal. The attractor RMP is a GDS in the task space $\y = \q -
\q_g$ with a metric $w(\y) \I$, a damping $\eta w(\y) \I$, and a potential that
is zero at $\y=0$, where $\eta > 0$ (\blueold{see \cite[Appendix D]{cheng2018rmpflowarxiv}}).
Fig.~\ref{fig:2d_full} shows the trajectories of the combined RMP. The combined non-constant metrics generate a behavior that transitions smoothly towards the goal while heading away from the obstacle. When the curvature terms are ignored (for both RMPs), the trajectories oscillate 
near the obstacle. In practice, this can result in jittery behavior on manipulators. 
When the metric is not velocity-based ($\G(\x) = w(\x)$) the behavior is less efficient in breaking free from the obstacle to go toward the goal.
\red{Finally, we show the change of Lyapunov function~\eqref{eq:Lypunov candidate} over time along these trajectories in \cref{fig:2d_energy} as verification of our theory.
}

\subsection{System Experiments}

\begin{figure}[h]
	\centering
	\includegraphics[width=0.9\linewidth]{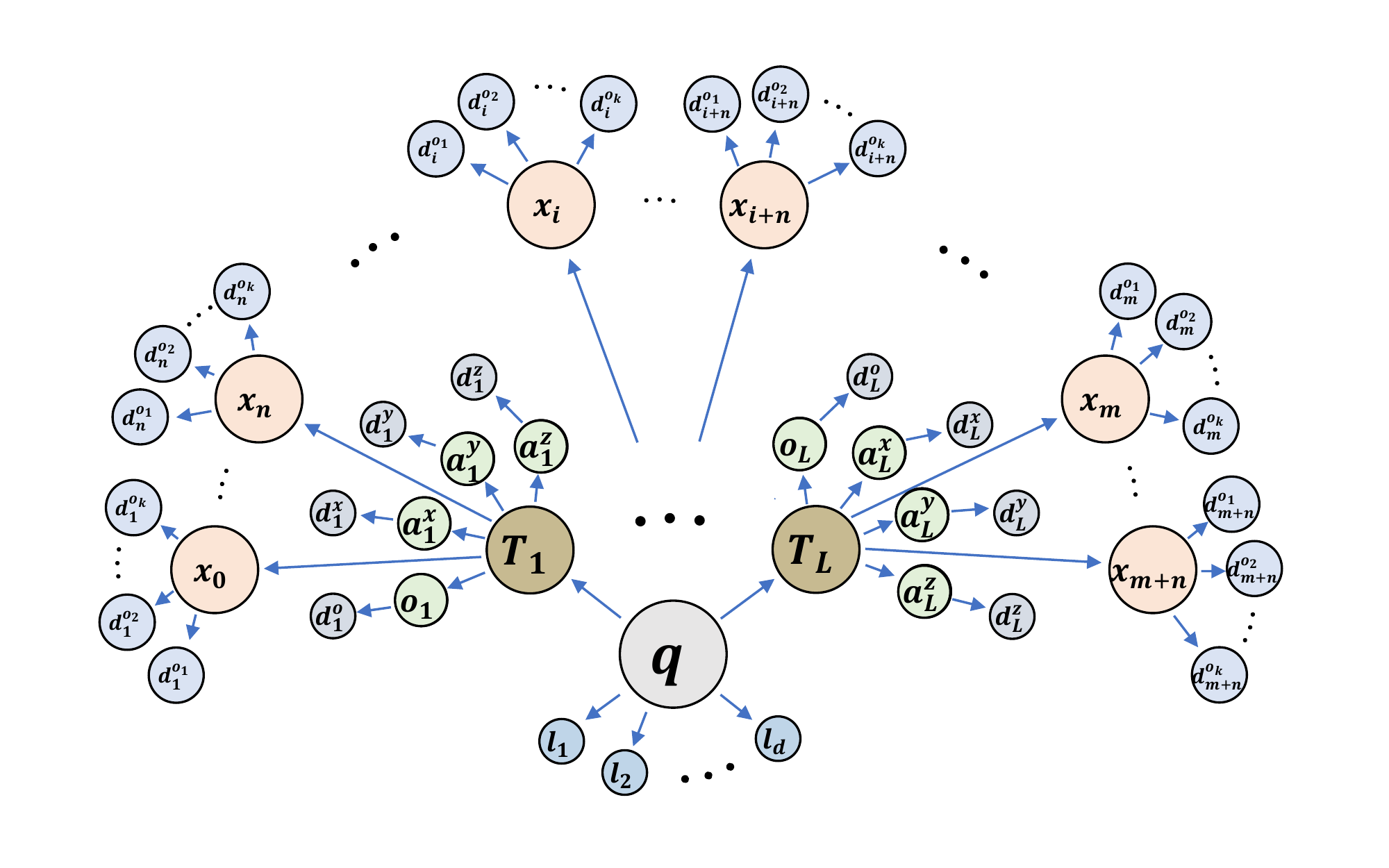}
	\caption{\small This figure depicts
		the tree of task maps used in the experiments.
		See Section~\ref{sec:TaskMapTree} for details.
	}
	\label{fig:rmpflow_taskmap_tree}
\end{figure}

\subsubsection{Task map and its Tree Structure} \label{sec:TaskMapTree}

\cref{fig:rmpflow_taskmap_tree}
depicts the tree of task maps used in the full-robot experiments.
The chosen structure emphasizes potential for parallelization over fully
exploiting the recursive nature of the kinematic chain, treating each link
frame as just one forward kinematic map step from the configuration
space.\footnote{We could possibly have saved some computation by defining the
	forward kinematic maps recursively as $(\mT_{i+1},\q_{i+1},\ldots,\q_d) =
	\psi_i(\mT_i,\q_i,\ldots,\q_d)$.} The configuration space $\q$ is linked to $L$
link frames $\mT_1,\ldots,\mT_L$ through the robot's forward kinematics (the details of tasks will be described later on for each individual experiment). Each
frame has 4 frame element spaces: the origin $o_i$ and each of the axes
$\ma_i^x,\ma_i^y,\ma_i^z$, with corresponding distance spaces to targets
$d_i^o,d_i^x,d_i^y,d_i^z$ (if they are active). Additionally, there are a
number of obstacle control points $\x_j$ distributed across each of the links,
each with $k$ associated distance spaces $d_j^{o_1},\ldots,d_j^{o_k}$, one
for each obstacle $o_1,\ldots,o_k$. Finally, for each dimension of the
configuration space there's an associated joint limit space $l_1,\ldots,l_d$.

\begin{figure*}[t]
	\centering
	\begin{tabular}{cccc}
		\includegraphics[height=0.3\columnwidth]{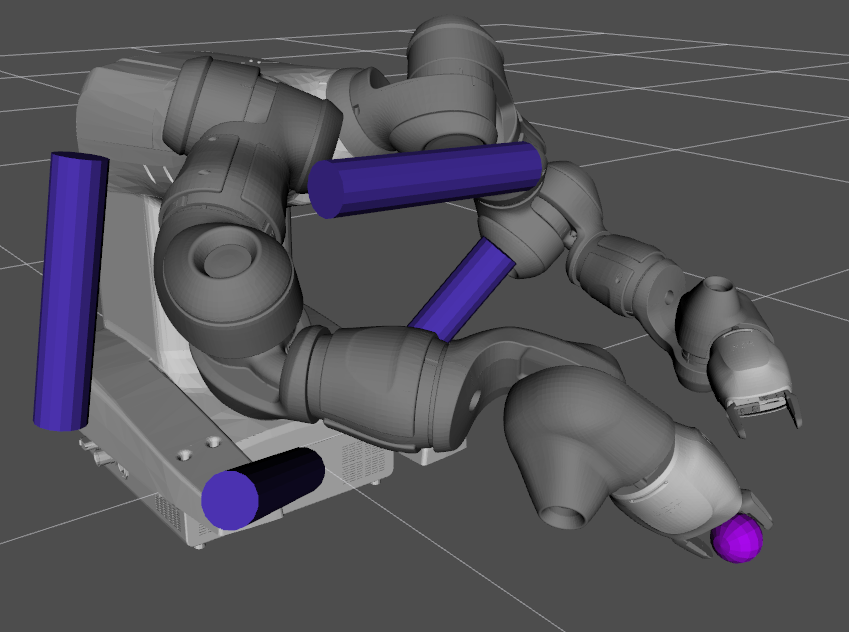} &
		\includegraphics[height=0.3\columnwidth]{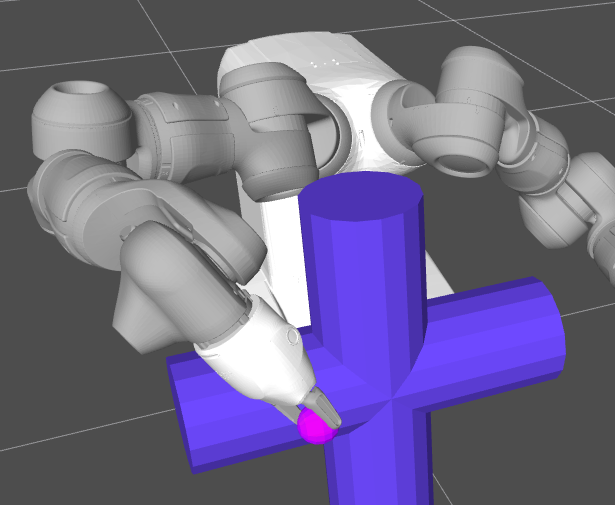} &
		\includegraphics[height=0.3\columnwidth]{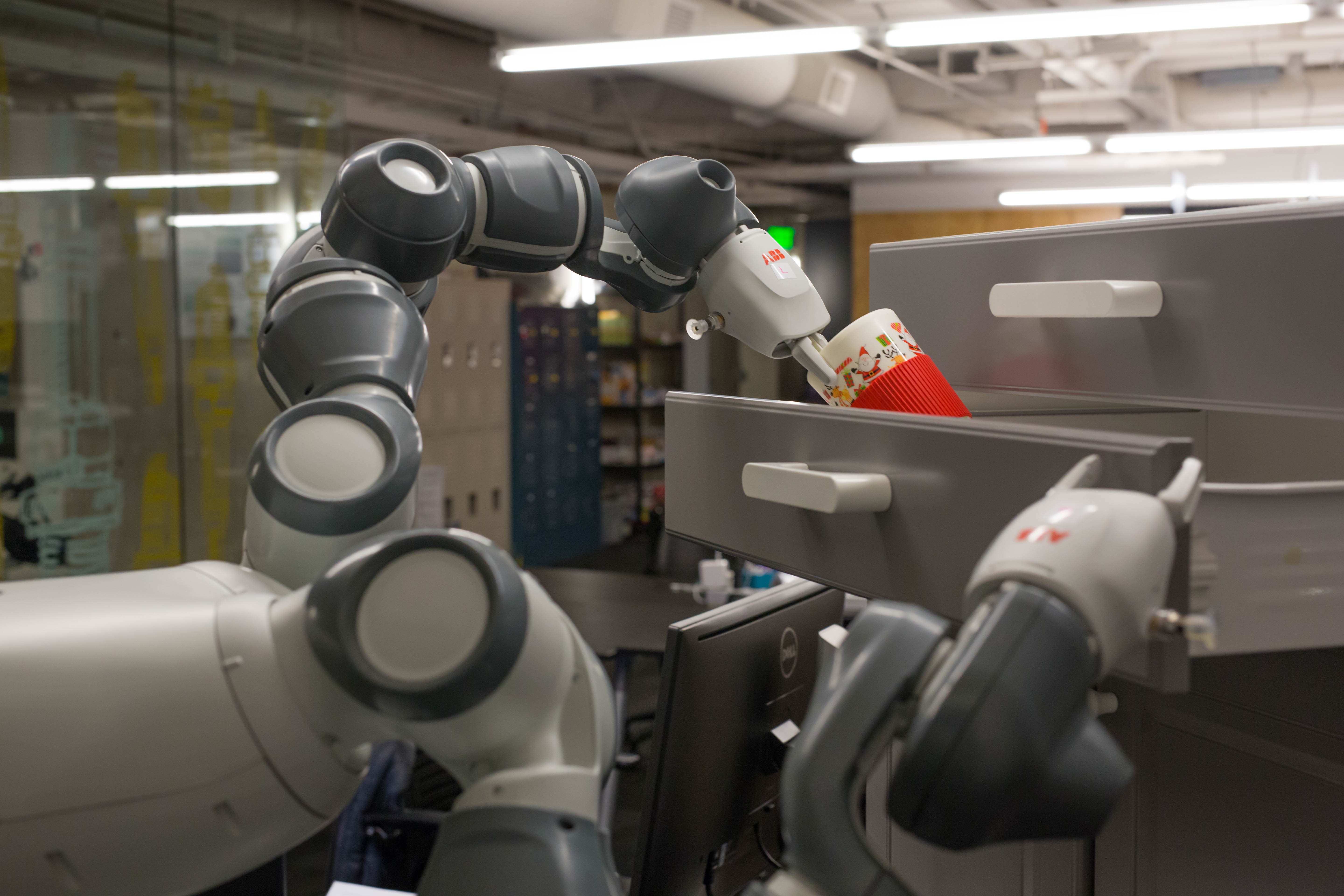} &
		\includegraphics[height=0.3\columnwidth]{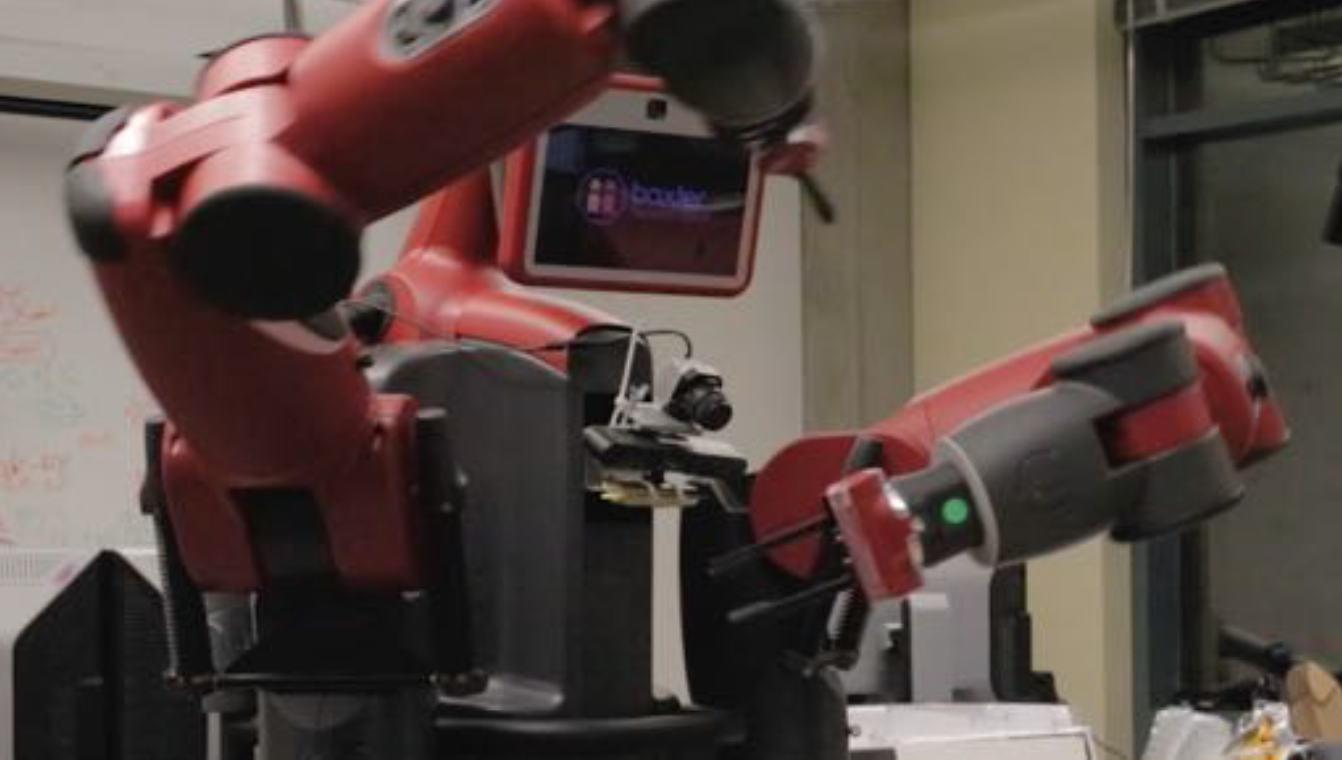} \\
		\multicolumn{2}{c}{simulated worlds} &
		\multicolumn{2}{c}{real-world experiments}
	\end{tabular}
	\caption{\small Two of the six simulated worlds in the reaching experiments (left), and
		the two physical dual-arm
		platforms in the full system experiment (right). }
	\label{fig:robots}
\end{figure*}
\begin{figure*}[h]
	\centering
	\includegraphics[trim={95 0 100 10},clip,width=0.7\textwidth,keepaspectratio]{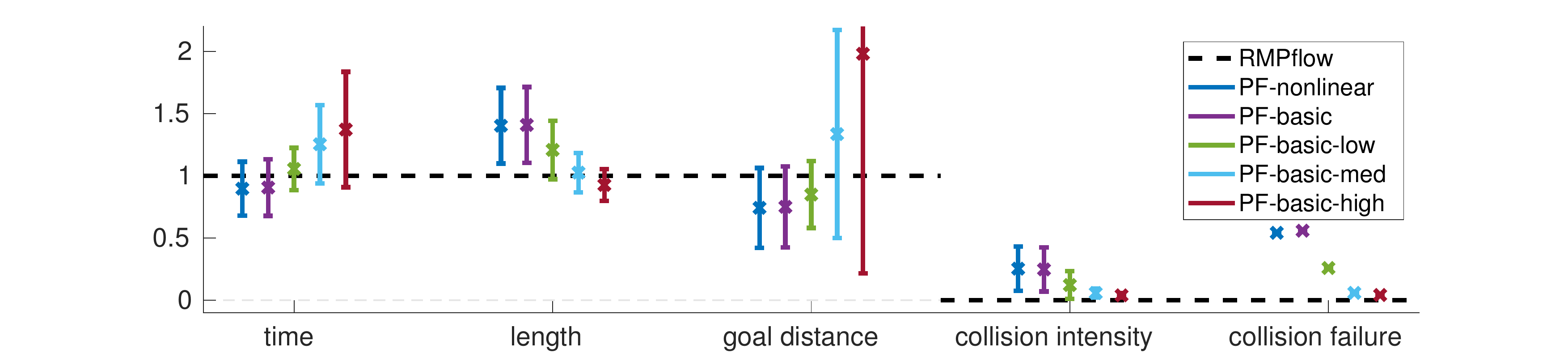}
	\caption{\small Results for reaching experiments. Though some methods achieve a shorter goal distance than \flow in successful trials, they end up in collision in most the trials. }
	\label{fig:reach}
\end{figure*}

\subsubsection{Reaching-through-clutter Experiments}

We set up a collection of clutter-filled environments with cylindrical
obstacles of varying sizes in simulation as depicted in Fig.~\ref{fig:robots}, and tested the performance of \flow and two potential
field methods on a modeled ABB YuMi robot.

\textbf{Compared methods:}
\begin{enumerate}[label=(\roman*)]
\item {\flow:} We implement \flow using the RMPs  in Section~\ref{sec:example RMPs} and
  detailed in \blueold{\cite[Appendix D]{cheng2018rmpflowarxiv}}. In particular, we place collision-avoidance
  controllers on distance spaces $s_{ij} = d_j(\x_i)$, where $j=1,\ldots,m$ indexes
  the world obstacle $o_j$ and $i=1,\ldots,n$ indexes the $n$ control point along the robot's
  body. Each collision-avoidance controller uses a weight function $w_o(\x)$ that
  ranges from $0$ when the robot is far from the obstacle to $w_o^\mathrm{max}\gg 0$ when the robot is in contact with the obstacle's surface.
  Similarly, the attractor potential uses a weight function $w_a(\x)$ that
  ranges from $w_a^{\mathrm{min}}$ far from the target to $w_a^{\mathrm{max}}$
  close to the target.
\item {PF-basic:} This variant is a basic implementation of obstacle
  avoidance potential fields with dynamics shaping. We use the
  RMP framework to implement this variant by placing collision-avoidance
  controllers on the same body control points used in \flow but with isotropic
  metrics of the form
  $\G_o^\mathrm{basic}(\x) = w_o^\mathrm{max}\I$ for each control point, with
  $w_o^\mathrm{max}$ matching
  the value \flow uses. Similarly, the attractor uses the same attractor potential
  as \flow, but with a constant isotropic metric with the form
  $\G_a^\mathrm{basic}(\x) = w_a^\mathrm{max}\I$.
\item {PF-nonlinear:} This variant matches PF-basic in construction, except
  it uses a \textit{nonlinear} isotropic metrics of the form
  $\G_o^\mathrm{nlin}(\x_i) = w_o(\x)\I$ and $\G_a^\mathrm{nlin}(\x_i) = w_a(\x)\I$
  for obstacle-avoidance and attraction, respectively, using weight functions
  matching \flow.
\end{enumerate}

\textbf{A note on curvature terms:} PF-basic uses constant metrics, so has no curvature
terms; PF-nonlinear has nontrivial curvature terms arising from the spatially
varying metrics, but we ignore them here to match common practice from
the operational space control literature.

\textbf{Parameter scaling of PF-basic:} Isotropic metrics do not express
spacial directionality toward obstacles, and that leads to
an inability of the system to effectively trade off the
competing controller requirements. That conflict results in more collisions and
increased instability. We,
therefore,
compare PF-basic under these baseline metric weights (matching \flow)
with variants that
incrementally strengthen collision avoidance controllers and C-space postural
controllers ($f_\mathcal{C}(\q, \qd) = \gamma_p(\q_0 - \q) - \gamma_d\qd$)
to improve these performance measures in the experiment.
We use the following weight scalings (first entry denotes the obstacle
metric scalar, and the second entry denotes the C-space metric scalar):
``low'' $(3, 10)$,
``med'' $(5, 50)$, and
``high'' $(10, 100)$.

\textbf{Environments:}
We run each of these variants on $6$ obstacle environments with $20$ randomly
sampled target locations each distributed on the opposite side
of the obstacle field from the robot. Three of the environments use four smaller
obstacles (depicted in panel 3 of Fig.~\ref{fig:robots}), and the remaining
three environments used two large obstacles (depicted in panel 4 of
Fig.~\ref{fig:robots}). Each environment used
the same $20$ targets to avoid implicit sampling bias in target choice.

\textbf{Performance measures:} We report results in Fig.~\ref{fig:reach} in terms of mean and one standard deviation error bars calculated across the $120$ trials for each of the following performance measures:\footnote{There is no guarantee of feasibility in planning
problems in general, so in all cases, we measure performance relative to the performance
of \flow, which is empirically stable and near optimal across these
problems.}
\begin{enumerate}[label=(\roman*)]
\item {\it Time to goal (``time''):} Length of time, in seconds, it takes for the
    robot to reach a convergence state. This convergence state
    is either the target, or its best-effort local minimum. If the system never converges, as in the case of many potential field trials for infeasible problems, the trial times
    out after 5 seconds. This metric measures time-efficiency of the movement.
\item {\it C-space path length (``length''):} This is the total path length $\int\|\qd\|dt$ of the movement through the configuration space across the trial. This metric measures how economical the movement is. In many of the potential-field variants with lower weights, we see significant fighting among the controllers resulting in highly inefficient extraneous motions.
\item {\it Minimal achievable distance to goal (``goal distance''):} Measures how close,
    in meters, the system is able to get to the goal with its end-effector.
\item {\it Percent time in collision for colliding trials (``collision intensity''):}
    Given that a trial has a collision, this metric measures the fraction of time
    the system is in collision throughout the trial. This metric indicates
    the intensity of the collision. Low values indicate short grazing collisions
    while higher values indicate long term obstacle penetration.
\item {\it Fraction of trails with collisions (``collision failure''):} Reports
    the fraction of trials with any collision event. We consider these to be
    collision-avoidance controller failures.
\end{enumerate}

\textbf{Discussion:} In Fig.~\ref{fig:reach}, we see that \flow outperforms each of these variants
significantly, with some informative trends:
\begin{enumerate}[label=(\roman*)]
\item \flow never collides, so its collision intensity and collision failure
    values are $0$.
\item The other techniques, progressing from no scaling of
    collision-avoidance and C-space controller weights to substantial scaling,
    show a profile of substantial collision in the beginning to fewer (but
    still non-zero) collision events in the end.
    But we note that improvement in collision-avoidance
    is achieved at the expense of time-efficiency and the robot's
    ability to reach the goal (it is too conservative).
\item Lower weight scaling of both PF-basic and PF-nonlinear
    actually achieve some faster times and better goal
    distances, but that is because the system pushes directly
    through obstacles, effectively
    ``cheating'' during the trial. \flow remains highly economical with its
    best effort reaching behaviors while ensuring the trials remain collision-free.
\item Lower weight scalings of PF-basic are highly uneconomical in their motion
    reflective of their relative instability. As the C-space weights on the posture
    controllers increase, the stability and economy of motion increase, but, again,
    at the expense of time-efficiency and optimality of the final reach.
\item There is little empirical difference between PF-basic and PF-nonlinear
    indicating that the defining feature separating \flow from the potential
    field techniques is its use of a highly nonlinear metric that explicitly
    stretches the space in the direction of  the obstacle as well as in the direction of the velocity
    toward the target. Those stretchings
    penalize deviations in the stretched directions during combination
    with other controllers while allowing variation along orthogonal directions.
    By being more explicit about how controllers should instantaneously trade
    off with one another, \flow is better able to mitigate the otherwise conflicting
    control signals.
\end{enumerate}

\textbf{Summary:} Isotropic metrics do not effectively convey how each collision and
attractor controller should trade off with one another, resulting in a conflict
of signals that obscure the intent of each controller making simultaneous
collision avoidance, attraction, and posture maintenance more difficult.
Increasing the weights of the controllers can improve their effectiveness, but at
the expense of decreased overall system performance. The resulting motions
are slower and less effective in reaching the goal
in spite of more stable behavior and fewer collisions. A key
feature of \flow is its ability to leverage highly nonlinear metrics that
better convey information about how controllers should trade off with one
another, while retaining provable stability guarantees. In combination, these
features result in efficient and economical obstacle avoidance behavior while
reaching toward targets amid clutter.

\subsubsection{System Integration for Real-Time Reactive Motion Generation}

We demonstrate the integrated vision and motion system on two physical
dual arm manipulation platforms: a Baxter robot from Rethink Robotics, and
a YuMi robot from ABB. Footage of our fully integrated
system (see start of Section~\ref{sec:experiments} for the link) depicting tasks such
as pick and place amid clutter, reactive manipulation of a cabinet drawers
and doors with human interaction, \emph{active} leadthrough with collision
controllers running, and pick and place into a cabinet drawer.\footnote{We have also run the RMP
portion of the system on an ABB IRB120 and a dual arm Kuka manipulation
platform with lightweight collaborative arms. Only the two platforms mentioned
here, the
YuMi and the Baxter, which use the full motion and vision integration, are
shown in the video for economy of space.}

This full integrated system, shown in the supplementary video, uses the RMPs
described in Section~\ref{sec:example RMPs} (detailed in
\blueold{\cite[Appendix D]{cheng2018rmpflowarxiv}}) with a slight modification that the curvature
terms are ignored. Instead, we maintain theoretical stability by using sufficient damping terms as described in Section~\ref{sec:1DExample} and by
operating at slower speeds. Generalization of these RMPs between embodiments was anecdotally
pretty consistent, although, as we demonstrate in our experiments, we would
expect more empirical deviation at higher speeds. For these manipulation tasks,
this early version of the system worked well as demonstrated in the video.

For visual perception, we leveraged consumer depth cameras along with two levels
of perceptual feedback:
\begin{enumerate}[label=(\roman*)]
\item {\it Ambient world:} For the Baxter system we create a voxelized
    representation of the unmodeled ambient world, and use distance fields
    to focus the collision controllers on just the closest obstacle points
    surrounding the arms. This methodology is similar in nature
    to \cite{2017_rss_system}, except we found empirically that attending
    to only the closest point to a skeleton representation resulted in
    oscillation in concaved regions where distance functions might result
    in nonsmooth kinks. We mitigate this issue by finding the closest points
    to a \emph{volume} around each control point, effectively smoothing
    over points of nondifferentiability in the distance field.
\item {\it Tracked objects:} We use the Dense Articulated Real-time
    Tracking (DART) system of \cite{Sch15DAR} to track articulated
    objects in real time through manipulations. This system is able
    to track both the robot and environmental objects, such
    as an articulated cabinet, simultaneously to give accurate
    measurements of their relative configuration effectively obviating
    the need for explicit camera-world calibration. As long as the
    system is initialized in the general region of the object locations
    (where for the cabinet and the robot, that would mean even up to
    half a foot of error in translation and a similar scale of error in rotation),
    the DART optimizer will snap to the right configuration when turned
    on. DART sends information about object locations to the motion generation,
    and receives back information about expected joint configurations (priors)
    from the
    motion system generating a robust world representation usable in
    a number of practical real-world manipulation problems.
\end{enumerate}

Each of our behaviors are decomposed as state machines that use visual feedback
to detect transitions, including transitions to reaction states as needed
to implement behavioral robustness. Each arm is represented as a separate robot
for efficiency, receiving real-time information about other arm's
current state enabling coordination. Both arms are programmed simultaneously
using a high level language that provides the programmer a unified view
of the surrounding world and command of both arms.

\section{Conclusion}

We propose an efficient policy synthesis framework, \flow, for generating policies with non-Euclidean behavior, including motion with velocity dependent metrics that are new to the literature.
In design, \flow is implemented as a computational graph, which can geometrically consistently combine subtask policies into a global policy for the robot.
In theory, we provide conditions for stability and show that \flow is intrinsically coordinate-free.
In the experiments, we demonstrate that \flow can generate smooth and natural motion for various tasks, when proper subtask RMPs are specified. Future work is to further relax the requirement on the quality of designing subtask RMPs by introducing learning components into \flow for additional flexibility.

\appendices

\section{Degenerate GDSs} \label{app:degnerate GDSs}

We discuss properties of degenerate GDSs.
Let us recall the GDS
$(\MM, \Gb, \B, \Phi)$ means the differential equation
\begin{align} \label{eq:GDS general}
	\Mb(\x,\xd) \xdd + \bm\xi_{\G}(\x,\xd)  = - \nabla_\x \Phi(\x) - \Bb(\x,\xd)\xd 
\end{align}
where  $\Mb(\x,\xd) = \Gb(\x,\xd) + \bm\Xi_{\G}(\x,\xd)$.
 For degenerate cases, $\Mb(\x,\xd)$ can be singular and~\eqref{eq:GDS general} define rather a family of differential equations. Degenerate cases are not uncommon; for example, the leaf-node dynamics could have $\G$ being only positive semidefinite. 
Having degenerate GDSs does not change the properties that we have proved, but one must be careful about whether differential equation satisfying~\eqref{eq:GDS general} exist.
For example, the existence is handled by the assumption on $\Mb$ in Theorem~\ref{th:consistency} and the assumption on $\Mb_r$ in Corollary~\ref{cr:consistency}. For \flow, we only need that $\Mb_r$ at the root node is non-singular. In other words, the natural-form RMP created by \pullback at the root node can be resolved in the canonical-form RMP for policy execution.
A sufficient and yet practical condition is provided in Theorem~\ref{th:condition on velocity metric}.

\ifCLASSOPTIONcaptionsoff
  \newpage
\fi

\bibliographystyle{IEEEtran}
\bibliography{refs}

\begin{thebibliography}{10}
\providecommand{\url}[1]{#1}
\csname url@samestyle\endcsname
\providecommand{\newblock}{\relax}
\providecommand{\bibinfo}[2]{#2}
\providecommand{\BIBentrySTDinterwordspacing}{\spaceskip=0pt\relax}
\providecommand{\BIBentryALTinterwordstretchfactor}{4}
\providecommand{\BIBentryALTinterwordspacing}{\spaceskip=\fontdimen2\font plus
\BIBentryALTinterwordstretchfactor\fontdimen3\font minus
  \fontdimen4\font\relax}
\providecommand{\BIBforeignlanguage}[2]{{%
\expandafter\ifx\csname l@#1\endcsname\relax
\typeout{** WARNING: IEEEtran.bst: No hyphenation pattern has been}%
\typeout{** loaded for the language `#1'. Using the pattern for}%
\typeout{** the default language instead.}%
\else
\language=\csname l@#1\endcsname
\fi
#2}}
\providecommand{\BIBdecl}{\relax}
\BIBdecl

\bibitem{rimon-ams-1991}
E.~Rimon and D.~Koditschek, ``The construction of analytic diffeomorphisms for
  exact robot navigation on star worlds,'' \emph{Transactions of the American
  Mathematical Society}, vol. 327, no.~1, pp. 71--116, 1991.

\bibitem{RIEMORatliff2015ICRA}
N.~Ratliff, M.~Toussaint, and S.~Schaal, ``Understanding the geometry of
  workspace obstacles in motion optimization,'' in \emph{IEEE International
  Conference on Robotics and Automation (ICRA)}, 2015.

\bibitem{VijayakumarTopologyMotionPlanning2013}
V.~Ivan, D.~Zarubin, M.~Toussaint, T.~Komura, and S.~Vijayakumar,
  ``Topology-based representations for motion planning and generalization in
  dynamic environments with interactions,'' \emph{International Journal of
  Robotics Research (IJRR)}, vol.~32, no. 9-10, pp. 1151--1163, 2013.

\bibitem{Watterson-TrajOptManifolds-RSS-18}
M.~Watterson, S.~Liu, K.~Sun, T.~Smith, and V.~Kumar, ``Trajectory optimization
  on manifolds with applications to {$SO(3)$} and {$\R^3 \times S^2$},'' in
  \emph{Robotics: Science and Systems (RSS)}, 2018.

\bibitem{ToussaintTrajOptICML2009}
M.~Toussaint, ``Robot trajectory optimization using approximate inference,'' in
  \emph{ICML}, 2009, pp. 1049--1056.

\bibitem{LavallePlanningAlgorithms06}
S.~M. LaValle, \emph{Planning Algorithms}.\hskip 1em plus 0.5em minus
  0.4em\relax Cambridge, U.K.: Cambridge University Press, 2006, available at
  http://planning.cs.uiuc.edu/.

\bibitem{KaramanRRTStar2011}
\BIBentryALTinterwordspacing
S.~Karaman and E.~Frazzoli, ``Sampling-based algorithms for optimal motion
  planning,'' \emph{International Journal of Robotics Research (IJRR)},
  vol.~30, no.~7, pp. 846--894, 2011. [Online]. Available:
  \url{http://arxiv.org/abs/1105.1186}
\BIBentrySTDinterwordspacing

\bibitem{GammellBitStar2014}
J.~D. Gammell, S.~S. Srinivasa, and T.~D. Barfoot, ``{Batch Informed Trees
  (BIT*)}: {S}ampling-based optimal planning via the heuristically guided
  search of implicit random geometric graphs,'' in \emph{IEEE International
  Conference on Robotics and Automation (ICRA)}, 2015.

\bibitem{mukadam2017continuous}
M.~Mukadam, J.~Dong, X.~Yan, F.~Dellaert, and B.~Boots, ``Continuous-time
  {G}aussian process motion planning via probabilistic inference,''
  \emph{International Journal of Robotics Research (IJRR)}, vol.~37, no.~11,
  pp. 1319---1340, 2018.

\bibitem{khatib1987unified}
O.~Khatib, ``A unified approach for motion and force control of robot
  manipulators: The operational space formulation,'' \emph{IEEE Journal on
  Robotics and Automation}, vol.~3, no.~1, pp. 43--53, 1987.

\bibitem{Peters_AR_2008}
J.~Peters, M.~Mistry, F.~Udwadia, J.~Nakanishi, and S.~Schaal, ``A unifying
  framework for robot control with redundant {DOF}s,'' \emph{Autonomous
  Robots}, vol.~24, no.~1, pp. 1--12, 2008.

\bibitem{UdwadiaGaussPrincipleControl2003}
\BIBentryALTinterwordspacing
F.~E. Udwadia, ``A new perspective on the tracking control of nonlinear
  structural and mechanical systems,'' \emph{Proceedings of the Royal Society
  of London A: Mathematical, Physical and Engineering Sciences}, vol. 459, no.
  2035, pp. 1783--1800, 2003. [Online]. Available:
  \url{http://rspa.royalsocietypublishing.org/content/459/2035/1783}
\BIBentrySTDinterwordspacing

\bibitem{udwadia1996analytical}
F.~E. Udwadia and R.~E. Kalaba, \emph{Analytical Dynamics: A New
  Approach}.\hskip 1em plus 0.5em minus 0.4em\relax Cambridge University Press,
  1996.

\bibitem{2017_rss_system}
\BIBentryALTinterwordspacing
D.~Kappler, F.~Meier, J.~Issac, J.~Mainprice, C.~Garcia~Cifuentes,
  M.~W{\"u}thrich, V.~Berenz, S.~Schaal, N.~Ratliff, and J.~Bohg, ``Real-time
  perception meets reactive motion generation,'' \emph{IEEE Robotics and
  Automation Letters}, vol.~3, no.~3, pp. 1864--1871, 2018. [Online].
  Available: \url{https://arxiv.org/abs/1703.03512}
\BIBentrySTDinterwordspacing

\bibitem{mukadam2017approximately}
M.~Mukadam, C.-A. Cheng, X.~Yan, and B.~Boots, ``Approximately optimal
  continuous-time motion planning and control via probabilistic inference,'' in
  \emph{IEEE International Conference on Robotics and Automation}, 2017, pp.
  664--671.

\bibitem{bullo2004geometric}
F.~Bullo and A.~D. Lewis, \emph{Geometric control of mechanical systems:
  modeling, analysis, and design for simple mechanical control systems}.\hskip
  1em plus 0.5em minus 0.4em\relax Springer Science \& Business Media, 2004,
  vol.~49.

\bibitem{ratliff2018riemannian}
N.~D. Ratliff, J.~Issac, D.~Kappler, S.~Birchfield, and D.~Fox, ``Riemannian
  motion policies,'' \emph{arXiv preprint arXiv:1801.02854}, 2018.

\bibitem{walker1982efficient}
M.~W. Walker and D.~E. Orin, ``Efficient dynamic computer simulation of robotic
  mechanisms,'' \emph{Journal of Dynamic Systems, Measurement, and Control},
  vol. 104, no.~3, pp. 205--211, 1982.

\bibitem{cheng2018rmpflow}
C.-A. Cheng, M.~Mukadam, J.~Issac, S.~Birchfield, D.~Fox, B.~Boots, and
  N.~Ratliff, ``{RMP}flow: A computational graph for automatic motion policy
  generation,'' in \emph{The 13th International Workshop on the Algorithmic
  Foundations of Robotics}, 2018.

\bibitem{cheng2018rmpflowarxiv}
------, ``Rmpflow: A computational graph for automatic motion policy
  generation,'' \emph{arXiv preprint arXiv:1811.07049}, 2018.

\bibitem{meng2019NeuralAutoNavigation}
X.~Meng, N.~Ratliff, Y.~Xiang, and D.~Fox, ``Neural autonomous navigation with
  riemannian motion policy,'' in \emph{IEEE International Conference on
  Robotics and Automation (ICRA)}, 2019.

\bibitem{meng2020TopologicalNavigation}
------, ``Scaling local control to large-scale topological navigation,'' in
  \emph{IEEE International Conference on Robotics and Automation (ICRA)}, 2020.

\bibitem{wingo2020Extendings}
B.~Wingo, C.~Cheng, M.~Murtaza, M.~Zafar, and S.~Hutchinson, ``Extending
  riemmanian motion policies to a class of underactuated
  wheeled-inverted-pendulum robots,'' in \emph{IEEE International Conference on
  Robotics and Automation}, 2020.

\bibitem{paxton2019RLDS}
C.~Paxton, N.~Ratliff, C.~Eppner, and D.~Fox, ``Representing robot task plans
  as robust logical-dynamical systems,'' in \emph{IEEE/RSJ International
  Conference on Intelligent Robots and Systems (IROS)}, 2019.

\bibitem{sutanto2019TactileServoing}
G.~Sutanto, N.~Ratliff, B.~Sundaralingam, Y.~Chebotar, Z.~Su, A.~Handa, and
  D.~Fox, ``Learning latent space dynamics for tactile servoing,'' in
  \emph{IEEE International Conference on Robotics and Automation (ICRA)}, 2019.

\bibitem{li2019MultiAgentRMPsArXiv}
\BIBentryALTinterwordspacing
A.~Li, M.~Mukadam, M.~Egerstedt, and B.~Boots, ``Multi-objective policy
  generation for multi-robot systems using riemannian motion policies,''
  \emph{CoRR}, vol. abs/1902.05177, 2019. [Online]. Available:
  \url{http://arxiv.org/abs/1902.05177}
\BIBentrySTDinterwordspacing

\bibitem{li2019LyapunovRMPs}
\BIBentryALTinterwordspacing
A.~Li, C.~Cheng, B.~Boots, and M.~Egerstedt, ``Stable, concurrent controller
  composition for multi-objective robotic tasks,'' \emph{CoRR}, vol.
  abs/1903.12605, 2019. [Online]. Available:
  \url{http://arxiv.org/abs/1903.12605}
\BIBentrySTDinterwordspacing

\bibitem{mukadam2019RMPFusion}
M.~Mukadam, C.-A. Cheng, D.~Fox, B.~Boots, and N.~Ratliff, ``Riemannian motion
  policy fusion through learnable lyapunov function reshaping,'' in
  \emph{Conference on Robot Learning (CoRL)}, 2019.

\bibitem{rana2019LearningRmpsFromDemonstration}
M.~A. Rana, A.~Li, H.~Ravichandar, M.~Mukadam, S.~Chernova, D.~Fox, B.~Boots,
  and N.~Ratliff, ``Learning reactive motion policies in multiple task spaces
  from human demonstrations,'' in \emph{Conference on Robot Learning (CoRL)},
  2019.

\bibitem{albu2002cartesian}
A.~Albu-Schaffer and G.~Hirzinger, ``Cartesian impedance control techniques for
  torque controlled light-weight robots,'' in \emph{IEEE International
  Conference on Robotics and Automation (ICRA)}, vol.~1, 2002, pp. 657--663.

\bibitem{sentis2006whole}
L.~Sentis and O.~Khatib, ``A whole-body control framework for humanoids
  operating in human environments,'' in \emph{IEEE International Conference on
  Robotics and Automation (ICRA)}, 2006, pp. 2641--2648.

\bibitem{lo2016virtual}
S.-Y. Lo, C.-A. Cheng, and H.-P. Huang, ``Virtual impedance control for safe
  human-robot interaction,'' \emph{Journal of Intelligent \& Robotic Systems},
  vol.~82, no.~1, pp. 3--19, 2016.

\bibitem{DRCIntegratedSystemTodorov2013}
T.~Erez, K.~Lowrey, Y.~Tassa, V.~Kumar, S.~Kolev, and E.~Todorov, ``An
  integrated system for real-time model-predictive control of humanoid
  robots,'' in \emph{IEEE/RAS International Conference on Humanoid Robots},
  2013.

\bibitem{OptimalControlTheoryTodorov06}
E.~Todorov, ``Optimal control theory,'' \emph{In Bayesian Brain: Probabilistic
  Approaches to Neural Coding}, pp. 269--298, 2006.

\bibitem{liegeois1977automatic}
A.~Liegeois, ``Automatic supervisory control of the configuration and behaviour
  of multibody mechanisms,'' \emph{IEEE Transactions on Systems, Man and
  Cybernetics}, vol.~7, no.~12, pp. 868--871, 1977.

\bibitem{RatliffCHOMP2009}
N.~Ratliff, M.~Zucker, J.~A.~D. Bagnell, and S.~Srinivasa, ``{CHOMP}: Gradient
  optimization techniques for efficient motion planning,'' in \emph{IEEE
  International Conference on Robotics and Automation (ICRA)}, 2009.

\bibitem{Mukadam-ICRA-16}
M.~Mukadam, X.~Yan, and B.~Boots, ``Gaussian process motion planning.'' in
  \emph{IEEE Conference on Robotics and Automation (ICRA)}, 2016.

\bibitem{Dong-RSS-16}
J.~Dong, M.~Mukadam, F.~Dellaert, and B.~Boots, ``Motion planning as
  probabilistic inference using {G}aussian processes and factor graphs,'' in
  \emph{Robotics: Science and Systems (RSS)}, 2016.

\bibitem{KhatibPotentialFields1985}
O.~Khatib, ``Real-time obstacle avoidance for manipulators and mobile robots,''
  in \emph{IEEE International Conference on Robotics and Automation (ICRA)},
  vol.~2, Mar 1985, pp. 500--505.

\bibitem{flacco2012icra}
F.~Flacco, T.~Kr{\"o}ger, A.~D. Luca, and O.~Khatib, ``A depth space approach
  to human-robot collision avoidance,'' in \emph{IEEE International Conference
  on Robotics and Automation (ICRA)}, May 2012, pp. 338--345.

\bibitem{kaldestad2014icra}
K.~B. Kaldestad, S.~Haddadin, R.~Belder, G.~Hovland, and D.~A. Anisi,
  ``Collision avoidance with potential fields based on parallel processing of
  3{D}-point cloud data on the {GPU},'' in \emph{IEEE International Conference
  on Robotics and Automation (ICRA)}, May 2014, pp. 3250--3257.

\bibitem{Nakanishi_IJRR_2008}
J.~Nakanishi, R.~Cory, M.~Mistry, J.~Peters, and S.~Schaal, ``Operational space
  control: A theoretical and empirical comparison,'' \emph{International
  Journal of Robotics Research (IJRR)}, vol.~6, pp. 737--757, 2008.

\bibitem{platt2011multiple}
R.~Platt, M.~E. Abdallah, and C.~W. Wampler, ``Multiple-priority impedance
  control.'' in \emph{IEEE International Conference on Robotics and Automation
  (ICRA)}, 2011, pp. 6033--6038.

\bibitem{IjspeertDMPs2013}
A.~J. Ijspeert, J.~Nakanishi, H.~Hoffmann, P.~Pastor, and S.~Schaal,
  ``Dynamical movement primitives: Learning attractor models for motor
  behaviors,'' \emph{Neural Computation}, vol.~25, no.~2, pp. 328--373, Feb
  2013.

\bibitem{MainpriceWarpingIROS2016}
J.~Mainprice, N.~Ratliff, and S.~Schaal, ``Warping the workspace geometry with
  electric potentials for motion optimization of manipulation tasks,'' in
  \emph{IEEE/RSJ International Conference on Intelligent Robots and Systems
  (IROS)}, oct 2016.

\bibitem{ClassicalMechanicsTaylor05}
J.~R. Taylor, \emph{Classical Mechanics}.\hskip 1em plus 0.5em minus
  0.4em\relax University Science Books, 2005.

\bibitem{NashEmbeddingTheorem1956}
J.~Nash, ``The imbedding problem for {R}iemannian manifolds,'' \emph{Annals of
  Mathematics}, vol.~63, no.~1, pp. 20--63, 1956.

\bibitem{khalil1996noninear}
H.~K. Khalil, ``Noninear systems,'' \emph{Prentice-Hall, New Jersey}, vol.~2,
  no.~5, pp. 5--1, 1996.

\bibitem{lee2009manifolds}
J.~M. Lee, \emph{Manifolds and differential geometry}.\hskip 1em plus 0.5em
  minus 0.4em\relax Graduate Studies in Mathematics, vol. 107, American
  Mathematical Society, 2009.

\bibitem{Featherstone08}
R.~Featherstone, \emph{Rigid Body Dynamics Algorithms}.\hskip 1em plus 0.5em
  minus 0.4em\relax Springer, 2008.

\bibitem{Sch15DAR}
T.~Schmidt, R.~Newcombe, and D.~Fox, ``{DART}: Dense articulated real-time
  tracking with consumer depth cameras,'' \emph{Autonomous Robots}, vol.~39,
  no.~3, 2015.

\end{thebibliography}

\newpage

\begin{IEEEbiography}[{\includegraphics[width=1in,height=1.25in,clip,keepaspectratio]{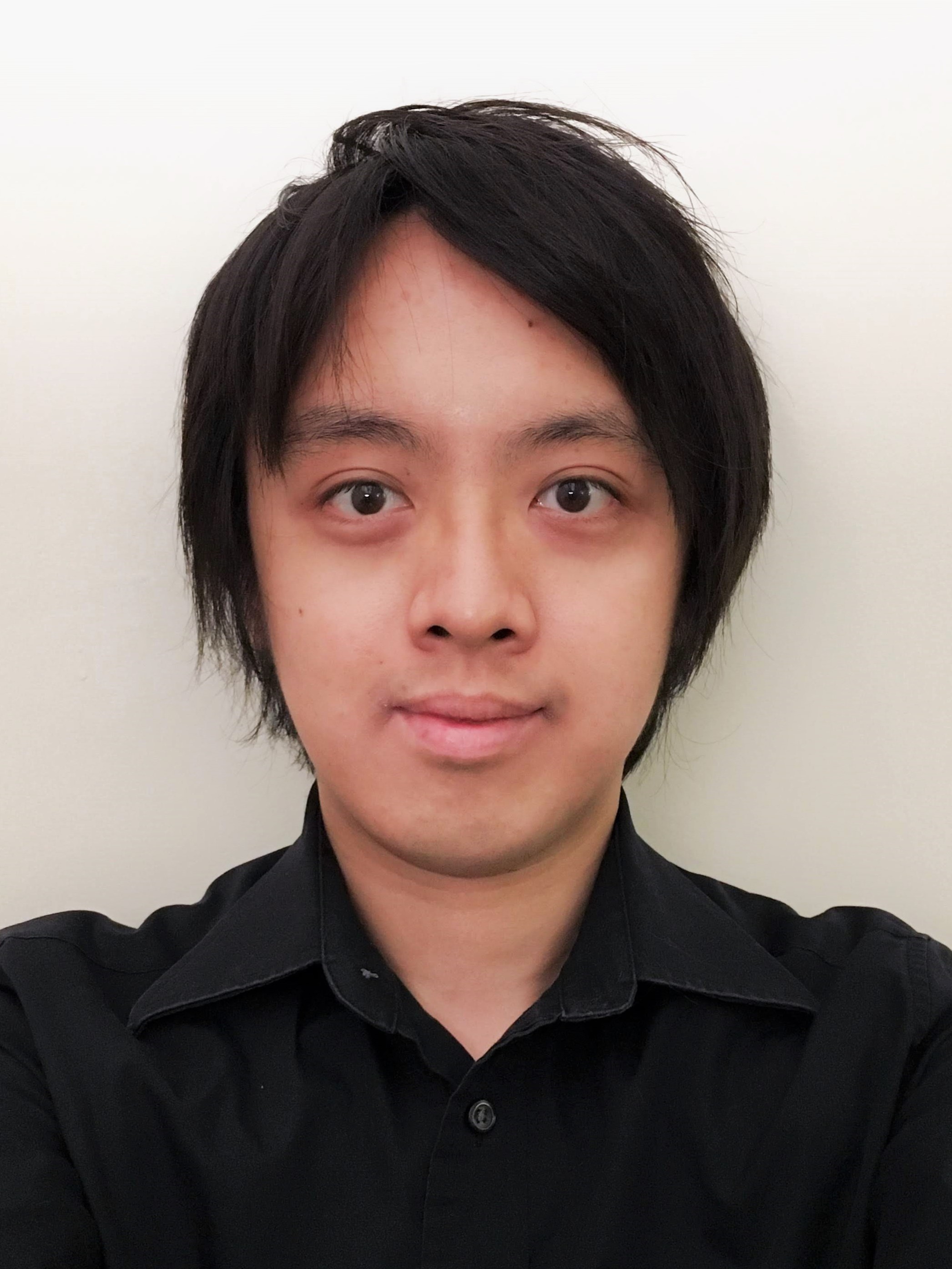}}]{Ching-An Cheng}
is a Robotics Ph.D. student at Institute for Robotics and Intelligent Machines, Georgia Tech. Previously he received M.S. in Mechanical Engineering, B.S. in Electrical Engineering, and B.S. in Mechanical Engineering from National Taiwan University. He is interested in developing theoretical foundations toward efficient and principled robot learning. His research concerns sample efficiency, structural properties, and uncertainties in reinforcement/imitation learning, online learning, and integrated motion planning and control.
\end{IEEEbiography}

\begin{IEEEbiography}[{\includegraphics[width=1in,height=1.25in,clip,keepaspectratio]{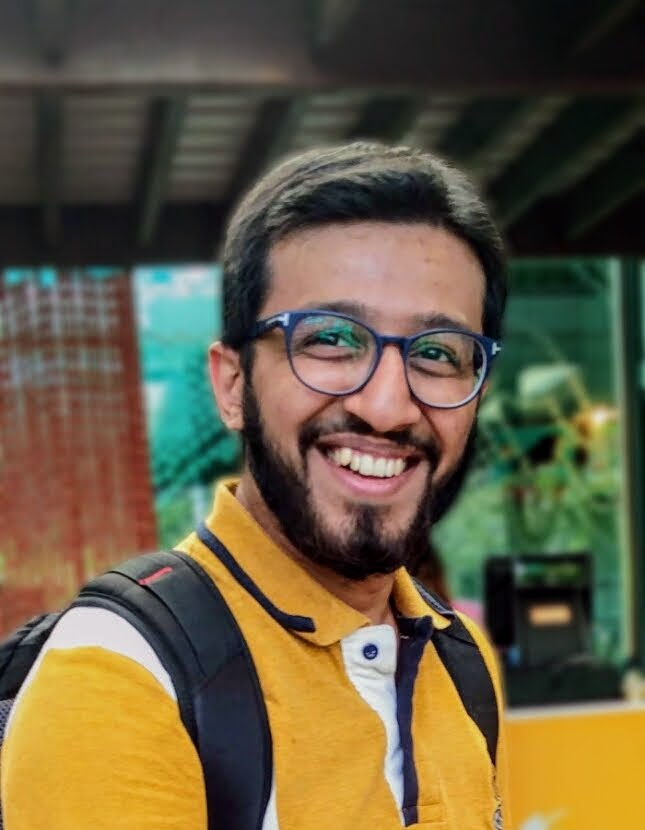}}]{Mustafa Mukadam}
is a Ph.D. student in Robotics at Georgia Institute of Technology and a member of the Institute for Robotics and Intelligent Machines. Previously, he received his M.S. degree in Aerospace Engineering from University of Illinois at Urbana-Champaign. His research focuses on motion planning, learning from demonstration, estimation, and structured techniques to bridge robotics and machine learning, with applications in autonomous navigation and mobile manipulation.
\end{IEEEbiography}

\begin{IEEEbiography}[{\includegraphics[width=1in,height=1.25in,clip,keepaspectratio]{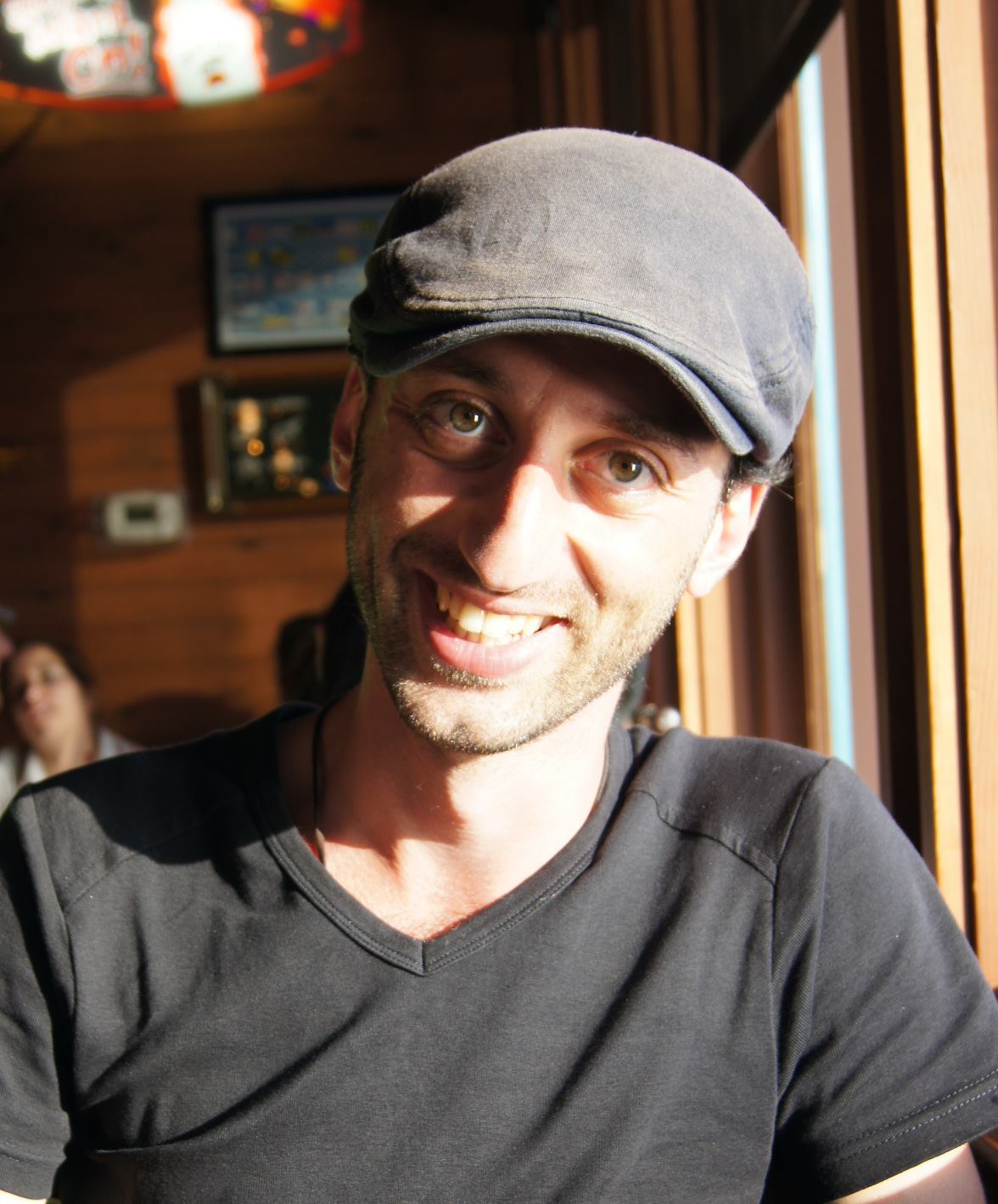}}]{Jan Issac}
is a software engineer at Nvidia with extensive experience API design, in low
level, and high level programming applied in complex systems such as
humanoid robotics.
Jan studied computer science at Kalrsruhe Institute of Technology (KIT) Germany
and Royal Institute of Technology (KTH) Stockholm, Sweden. He did his final
thesis on stochastic filtering for high dimensional on-line object tracking
at Computational Learning and Motor Control Lab (CLMC) at University of
Southern California (USC), Los Angeles, and Max Planck Institute of
Intelligent Systems - Autonomous Motion Department (AMD MPI-IS), Germany.
He co-founded Lula Robotics and helped design and build the system
which is now central to Nvidia's manipulation research.
He earned his Diploma (M.Sc.) in computer science from KIT in 2014.
\end{IEEEbiography}

\begin{IEEEbiography}[{\includegraphics[width=1in,height=1.25in,clip,keepaspectratio]{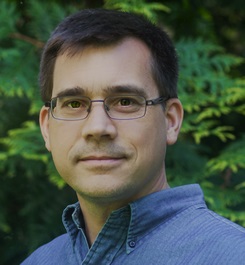}}]{Stan Birchfield}
is a principal research scientist at NVIDIA, exploring the intersection of computer vision and robotics. Previously, he was a tenured professor at Clemson University, where he led research in computer vision, visual tracking, mobile robotics, and the perception of highly deformable objects. He remains an adjunct faculty member at Clemson. He has also conducted research at Microsoft, was the principal architect of a commercial product at a startup company in the Bay Area, co-founded a startup with collaborators at Clemson, and served as a consultant for various companies. He has authored or co-authored nearly 100 publications, as well as a textbook on image processing and analysis; and his open-source software has been used by researchers around the world. He received his Ph.D. in electrical engineering, with a minor in computer science, from Stanford University.
\end{IEEEbiography}

\begin{IEEEbiography}[{\includegraphics[width=1in,height=1.25in,clip,keepaspectratio]{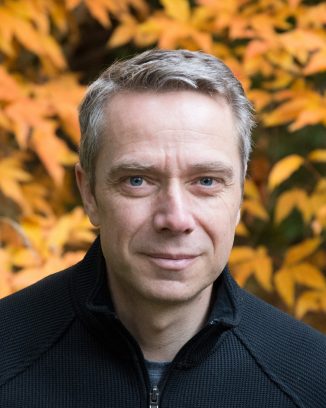}}]{Dieter Fox}
is a Professor in the Department of Computer Science \& Engineering at the University of Washington. He grew up in Bonn, Germany, and received his Ph.D. in 1998 from the Computer Science Department at the University of Bonn. He joined the UW faculty in the fall of 2000. He is currently on partial leave from UW and joined NVIDIA to start a Robotics Research Lab in Seattle. His research interests are in robotics, artificial intelligence, and state estimation. He is the head of the UW Robotics and State Estimation Lab RSE-Lab and recently served as the academic PI of the Intel Science and Technology Center for Pervasive Computing ISTC-PC. He is a Fellow of the AAAI and IEEE, and served as an editor of the IEEE Transactions on Robotics.
\end{IEEEbiography}

\begin{IEEEbiography}[{\includegraphics[width=1in,height=1.25in,clip,keepaspectratio]{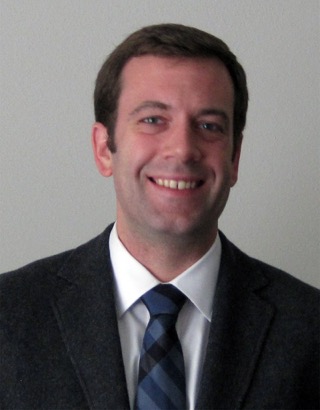}}]{Byron Boots}
is an Assistant Professor in the College of Computing at Georgia Tech. He directs the Georgia Tech Robot Learning Lab, affiliated with the Center for Machine Learning and the Institute for Robotics and Intelligent Machines. Byron’s research focuses on development of theory and systems that tightly integrate perception, learning, and control. He received his Ph.D. in Machine Learning from Carnegie Mellon University and was a postdoctoral researcher in Computer Science and Engineering at the University of Washington.
\end{IEEEbiography}

\begin{IEEEbiography}[{\includegraphics[width=1in,height=1.25in,clip,keepaspectratio]{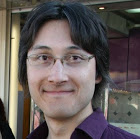}}]{Nathan Ratliff}
is a distinguished research scientist at Nvidia studying
behavior generation and robotic systems. He received his PhD in Robotics from
Carnegie Mellon University under Prof. J. Andrew Bagnell working closely
with Siddhartha Srinivasa. He has worked at Amazon, the Toyota
Technological Institute in Chicago, Intel Lab, Google, the Max Planck Institute
for Intelligent Systems, and the
University of Stuttgart, and prior to joining Nvidia he co-founded Lula
Robotics where he drove the designed and development the system which has become
the foundation of Nvidia’s manipulation platform.
\end{IEEEbiography}

\end{document}